\documentclass[a4paper,12pt]{article}
\usepackage{amsmath}
\usepackage{amsfonts}
\usepackage{amssymb}
\usepackage{amsthm}
\usepackage{graphicx}
\usepackage{ifthen}
\usepackage{verbatim}
\usepackage{amsmath}
\usepackage{setspace}
\usepackage{amsfonts}
\usepackage{amssymb}
\usepackage{amsthm}
\usepackage{lscape}
\usepackage{setspace}
\usepackage{listings}
\usepackage{amsmath}
\usepackage{amsfonts}
\usepackage{amssymb}
\usepackage{amsthm}
\usepackage{ifthen}
\usepackage{verbatim}
\usepackage{dsfont}
\usepackage{booktabs}
\usepackage{float}
\usepackage[ruled,vlined,linesnumbered]{algorithm2e}
\DontPrintSemicolon 
\SetKwInput{Input}{Input}
\SetKwInput{Output}{Output}
\usepackage[blocks]{authblk}
\usepackage[utf8]{inputenc}
\usepackage[british]{datetime2}
\DTMlangsetup[en-GB]{ord=omit}   % drop Òst/nd/rd/thÓ

\newtheorem{thm}{Theorem}[section]
\newtheorem{ex}[thm]{Example}
\newtheorem{coro}[thm]{Corollary}
\newtheorem{lem}[thm]{Lemma}
\newtheorem{defn}[thm]{Definition}
\newtheorem{rem}[thm]{Remark}
\newtheorem{prop}[thm]{Proposition}
\newtheorem{asu}{Assumption}
\newcommand{\levy}{{L\'evy }}
\newcommand{\cadlag}{{c\`adl\`ag }}
\newcommand{\Schrodinger}{{Schr\"odinger }}
\def\1{\mathds{1}}
\def\R{\mathbb{R}}
\def\P{\mathbb{P}}

\def\E{\mathbb{E}}
\def\N{\mathbb{N}}

\def\T{\mathbb{T}}

\def\dd{\mathrm{d}}

\def\law{\overset{\textnormal{law}}{=}}
\def\tp{\top}
\def\F{\mathcal{F}}

\def\Lambdab{\boldsymbol{\Phi}}
\def\xib{\boldsymbol{\xi}}

\def\xb{\mathbf{x}}
\def\yb{\mathbf{y}}

\DeclareMathOperator*{\argmin}{argmin}
\begin{document}
\title{Random-Bridges as Stochastic Transports \\
for Generative Models}
%\author{Stefano Goria\thanks{ 
%Artificial Intelligence and Mathematics Research Lab, levent@aimresearchlab.com},
%\hspace{0.025in} Levent Ali Meng\"ut\"urk\thanks{ 
%University College London, Department of Mathematics}\,\,\footnotemark[1], \\
%\hspace{0.025in} Murat Cahit Meng\"ut\"urk\footnotemark[1],
%\hspace{0.025in} M. Berkan Sesen\footnotemark[1]}
\author[1]{Stefano Goria}
\author[1,2]{Levent Ali Meng\"ut\"urk}
\author[1]{Murat Cahit Meng\"ut\"urk}
\author[1]{M. Berkan Sesen}

\affil[1]{Artificial Intelligence and Mathematics Research Lab\par
James Carter Road, Mildenhall, Bury St. Edmunds IP28 7DE, UK}

\affil[2]{Department of Mathematics, University College London\par
25 Gordon Street, London, 
WC1H 0AY, UK}

\date{}
\maketitle
\begin{abstract}
This paper motivates the use of random-bridges -- stochastic processes conditioned to take target distributions at fixed timepoints -- in the realm of generative modelling. Herein, random-bridges act as stochastic transports between two probability distributions when appropriately initialized, displaying either Markovian or non-Markovian, and continuous, discontinuous, or hybrid patterns depending on the choice of the driving process. We show how one can start from general probabilistic statements and branch out into specific representations for learning and simulation algorithms. Essentially, our framework formulates stochastic transports as noisy information processes governing the dynamic translation and propagation of information between arbitrary probability measures. Our theoretical and algorithmic presentation details training and simulation frameworks for both Gaussian and \levy (specifically, Gamma, Poisson and Stable-$1/2$) random-bridge transports. The empirical validation focuses on the Gaussian case for demonstration and comparison purposes, where our construct delivers high-quality samples, achieving competitive Fr\'echet Inception Distance scores compared to widely used denoising diffusions and flow matching frameworks.
\end{abstract}

\section{Introduction}
\label{Introduction}
Over the last decade, diffusion processes have transformed generative modeling across unstructured data domains, such as image and audio generation \cite{12, 8a, 13}. In particular, Denoising Diffusion Probabilistic Models (DDPMs) have emerged as a prominent paradigm, achieving sample generation quality comparable to conventional Generative Adversarial Networks (GANs) without the instability of adversarial training \cite{12, 13a, 13}. To address the computational bottlenecks of these early models, subsequent advancements across academia and industry have focused on optimizing log-likelihood estimation \cite{13a}, accelerating sampling speed via non-Markovian formulations \cite{13b}, and designing distillation techniques that reduce the number of inference steps while maintaining high-quality, large-scale outputs \cite{20, 21, 19, 22}.

The DDPM approach relies on stochastic diffusion models that map Gaussian noise onto a target data distribution via a noising-denoising protocol governed by a time-reversal construct. Consequently, these applications are restricted by the assumption that the prior distribution is Gaussian, which requires model-dependent denoising machinery built on backward-time transitions rooted in time-reversal theory \cite{1,2,3}. Algorithmically, DDPM methods progressively inject random noise into target data and subsequently approximate step-by-step noise removal to generate new target samples \cite{12}. While empirically powerful and conceptually straightforward, this framework exhibits several core limitations in theory and practice. First, it inherently excludes scenarios where the prior distribution is non-Gaussian, which renders image-to-image translation unachievable, thereby hampering its flexibility in arbitrary reference-to-target transports. Second, the forward noising process toward an asymptotic Gaussian limit introduces statistical bias over finite time horizons, since the true Gaussian reference distribution is only coarsely approximated \cite{12, 13a}. Third, the framework relies on iterative denoising approximations that remain computationally demanding \cite{12}. Fourth, DDPMs are strictly driven by Brownian motion, which can generate intermediate distributions that carry no structural relevance to a potential relationship between the reference and target data configurations (e.g., order-preservation).

Our framework demonstrates that the reliance of DDPMs on iterative denoising steps, as part of a bi-directional learning architecture, is neither theoretically necessary, nor adequately flexible for data-to-data translation tasks. By requiring a single-directional map between arbitrary reference and target distributions, we employ randomized bridge processes, defined through a tractable probabilistic formulation, to structurally eliminate the denoising procedure. Consequently, our mathematical approach introduces substantial generality regarding reference and target distributions, and significant flexibility in choosing the driving stochastic process (beyond Brownian motion) that conducts the transport between them. In essence, this paper establishes a deep connection between generative artificial intelligence and information processing, wherein randomized bridges are interpreted as noisy information processes that govern the dynamic translation and propagation of information between arbitrary probability measures. Accordingly, our framework places stochastic filtering and entropy evolution at a central position, and generalizes fundamental principles previously developed across mathematical physics, mathematical finance, and signal processing \cite{4,5,6,7,7gamma,8,10,9,11,15} to develop an alternative foundation for generative AI. The core program is to bypass any dependence on bi-directional reversal schemes, and ensure probabilistic convergence entirely (and exactly, without any statistical bias) through single-directional paths to achieve robust generative results across a broader class of data-to-data transports. In doing so, we introduce an overarching framework that does not assume a baseline Brownian motion-driven diffusion, but instead accommodates any driving process whose sample paths are right-continuous with left-limits (c\`adl\`ag). One of the merits of this generality can be illustrated through a concrete example: consider a reference-to-target transport governed by a monotonic translation, where every random element sampled from the reference distribution is structurally smaller than every element sampled from the target distribution (e.g., generating audio from a baseline sample onto a higher pitch). In such scenarios, stochastic trajectories driven by standard Brownian motion will inadvertently include unnecessary spatial states during intermediate transport phases, given the natural order-preserving relationship between the reference and target data. Our formulation can readily address such scenarios to enhance computational efficiency via better-suited, monotonically non-decreasing stochastic drivers.

We highlight the recent stream of literature \(\cite{25, 24, 14, 26, 23, 16b, 17}\) leveraging Brownian-based bridge processes to enhance generative modeling capabilities, offering transports between arbitrary reference and target distributions without relying on denoising protocols. While our framework shares this core philosophy, it fundamentally differs in the formulation, scope and generality of its construct. First, rather than initializing the framework with a baseline Brownian motion-driven diffusion that is subsequently transformed into a bridge via Doob $h$-transforms, we establish our formulation on a general probabilistic footing almost entirely free of specific model constraints. This abstraction allows our framework to naturally encompass any \cadlag process, whether Markovian or non-Markovian, continuous or discontinuous. It is only after establishing this overarching mathematical architecture that we derive explicit representations for learning and simulation to achieve generative propagation. Consequently, standard Brownian motion-driven diffusion bridges are recovered as a specific instance of our framework, which hosts a vast corpus of intermediate informational transitions (e.g., subordinator bridge patterns). Furthermore, our formulation presents an alternative to the classical \Schrodinger bridge paradigm, which invokes the machinery of Optimal Transport to construct generative maps. In contrast, our approach generates valid stochastic transitions without requiring any optimal transport constraints. This shift in perspective inherently raises the question of what classical optimal transport machinery introduces to the setup. A random-bridge coupling constitutes a structurally exact transport by design: the endpoint marginals are matched by construction, requiring no underlying optimality constraints to ensure distributional alignment. We highlight that entropic optimal coupling serves to optimize sampling and inference efficiency, rather than to establish transport admissibility, which our construct already satisfies by definition. Our position in this paper is to trade the immediate optimality of the coupling for absolute analytic tractability and the inclusion of non-Brownian driving processes. Nonetheless, our framework reveals a fundamental connection to Flow Matching paradigms \cite{27, 28}, as does the \Schrodinger bridge approach. Crucially, we observe that the conditional expectation of our general \levy random-bridge transport aligns directly with the velocity fields governing Flow Matching frameworks, positioning our randomized construct as a stochastic generalization of deterministic generative flows. In terms of learning conditional transport velocities through squared-error regression objectives, our framework has connections to diffusion bridge mixtures, such as the non-denoising forward-time diffusion paradigm \cite{14, 26}, and stochastic interpolants \cite{interpolants1}. On the other hand, while stochastic interpolants rely on explicit path-by-path algebraic constructions pairing individual samples, and non-denoising forward diffusions mix SDE paths built on Doob $h$-transforms, our general formulation operates as a construction in \emph{law}. More specifically, by configuring the transport structurally through the conditional law of a general \cadlag process under global endpoint constraints, we circumvent the need for predefined algebraic pairings or the mixing of continuous path trajectories. This allows our framework to naturally host non-Brownian, discontinuous processes where explicit path mixtures or individual SDE transforms are analytically intractable. Accordingly, our approach shifts the pivot towards conditional information flows, demonstrating that regression-based generative estimation extends naturally to non-Brownian drivers, offering an expanded avenue for structured data geometries that continuous path-wise mixtures cannot accommodate. While the formal implications and mathematical details merit a separate, dedicated study, the aforementioned connections highlight that our framework serves as an augmented, structurally generalized direction in generative modeling, no longer strictly bound to the classical Brownian diffusion paradigm.

Since the main goal of this paper is to introduce and study an alternative framework for generative modelling, we aim to reserve a balanced presentation across theory, algorithms and empirical experiments. Theoretically, we provide a mathematical setup equipped with a general definition of random-bridges with almost no model-specific assumptions. In essence, our definition offers an abstract but concise mathematical recipe for constructing and simulating different bridge transport patterns. After discussing some of the relevant broad properties of random-bridges, we continue with analytically tractable Gaussian random-bridge transports. In this section, we prove various properties of these structures including their anticipative, non-anticipative and conditional expectation representations that lend themselves to our training and simulation algorithms. Thereafter, for demonstration and comparison purposes, we generate empirical noise-to-image translations on MNIST and CIFAR-10 datasets via Brownian random-bridge transports, and compare our outputs with improved denoising diffusions and flow matching frameworks. We report the Fr\'echet inception distance (FID) scores of all models, and show that our framework achieves highly competitive performance, which, albeit not surprising theoretically, is reassuring numerically for empirical implementations. Our paper continues with \levy random-bridge transports including their mathematical properties, stochastic representations and their connection to Doob $h$-transform. Furthermore, we prove that convex transformations of the conditional mean flow velocity field of \levy random-bridge transports are bounded above, which ensures that convex regularizing potentials or error metrics along the mean path remain well-behaved and constrained, providing corresponding structural stability during generative transports. In order to underscore the numerical implementation capability and flexibility of \levy random-bridge transports, we also provide their general training algorithm together with specific simulation algorithms for Gamma, Poisson and Stable-$1/2$ random-bridge transports. Although our presented training and simulation algorithms for the aforementioned \levy cases are ready to deploy by the reader, we leave their empirical implementations for future, in the scope of specific use-cases that would naturally benefit from these alternative \levy transports. At this point, we should also highlight the following two aspects: (i) although our theoretical framework directly encapsulates complex image-to-image translations, this paper excludes their numerical implementation, and leaves it for future in an empirical scope that includes \levy transports, (ii) although the empirical comparisons presented in this work do not include important recent advancements such as distillation-based approaches and \Schrodinger bridges, our initial findings encourage us to continue with this line of research and reserve a significantly more encompassing empirical experiment for the future. 

As mentioned above, the specific random bridge constructs developed and studied in \cite{4,5,6,7,7gamma,8,10,9,11,15} are foundational for this paper. Accordingly, we shall highlight our contributions relative to this stream of literature in order to provide a conceptual and methodological comparison detailing the original components of this study.

\begin{itemize}
    \item \textbf{Random Initialization and Coupling Framework:} Classical random-bridge formulations are fundamentally pinned to a deterministic initial state. A core extension of this work is the structural generalization of the random-bridge framework to host \textit{both} random initial and target laws simultaneously via coupling measures.
    \item \textbf{Arbitrary Data-to-Data Transports:} By expanding the mathematical initialization boundaries, our formulation is the first to establish random bridges as structurally exact stochastic transports capable of realizing arbitrary data-to-data translations, transforming the stochastic filtering paradigm into a generative engine.
    \item \textbf{Multivariate Stochastic Representations:} While univariate stochastic representations expressed in anticipative and non-anticipative forms have already been established in the random bridge literature, our framework extends these results to multivariate generalizations to encapsulate high-dimensional generative architectures. In doing so, achieving the corresponding generalised results (e.g. involving matrices) require additional or alternative mathematical proof patterns as presented in this paper. These generalizations naturally recover known representations under one-dimensional choices, as expected.		
    \item \textbf{Novel Generative Algorithms:} Since random bridges had never been realized and deployed as stochastic transport models prior to this study, every single training and simulation algorithm proposed in this paper is entirely new and original to this study. In order to render simulations operational in practice, new stochastic representations for Gaussian and \levy bridge transports, hitherto undiscovered, are established.	
 \item \textbf{Numerical Experiments:} Since this paper augments random bridges as generative machines, every numerical experiment, including benchmark comparisons, is completely original to this work.	We note that numerical implementations form a core applied section of this study for empirical demonstration and comparison purposes. 
    \item \textbf{Doob $h$-Transform Adaptations:} We establish formal, multi-dimensional connections linking our structural in-law construct to Doob $h$-transforms, clarifying how classical diffusion-bridge dynamics are hosted within our broader transport framework.
    \item \textbf{Convex Stability and Numerical Consistency:} This paper additionally presents (i) a novel result on convex transformations of the conditional mean flow velocity field of \levy random-bridge transports, and (ii) a numerical stability result for generative executions, both having no precedent in the random bridge literature.
\end{itemize}

The paper is structured as follows. Section 2 builds the mathematical framework and establishes relevant theoretical results for the rest of the paper. Section 3 is dedicated to Gaussian random-bridge transports including theory, training and simulation algorithms, in addition to empirical generations. Section 4 details \levy random-bridge transports including theory, training and specific simulation algorithms based on choices of the \levy driver. Section 5 concludes.

\section{Mathematical Framework}
We work on a probability space $(\Omega,\F,\P)$ equipped with a right-continuous and complete filtration $\{\F_{t}\}_{t \in \T}$, where $\T = [0,T]$ for some $T <\infty$. For some $n\in\N_+$, let
\begin{align}
\boldsymbol{X} = [X^{(1)}, \ldots,X^{(j)}, \ldots, X^{(n)}]^{\tp} \sim \boldsymbol{\Phi}
\end{align}
be a $\boldsymbol{\Phi}$-distributed square-integrable random variable (i.e., $\boldsymbol{X} \in \mathcal{L}^2(\mathbb{P}; \mathbb{X})$) with state-space $(\mathbb{X},\mathcal{B}(\mathbb{X}))$, where $\mathbb{X} \subseteq \R^n$ and $\mathcal{B}(\cdot)$ the Borel $\sigma$-field, and 
\begin{align}
\boldsymbol{Y}= [Y^{(1)}, \ldots,Y^{(j)}, \ldots, Y^{(n)}]^{\tp} \sim \boldsymbol{\Psi}
\end{align}
be a $\boldsymbol{\Psi}$-distributed square-integrable random variable (i.e., $\boldsymbol{Y} \in \mathcal{L}^2(\mathbb{P}; \mathbb{Y})$) with state-space $(\mathbb{Y},\mathcal{B}(\mathbb{Y}))$, where $\mathbb{Y} \subseteq \R^n$. Using measure-theoretic language, we highlight that
\begin{align}
\boldsymbol{\Phi} = \boldsymbol{X}_\sharp \mathbb{P} \,\,\,\, \text{and} \,\,\,\, \boldsymbol{\Psi}=\boldsymbol{Y}_\sharp \mathbb{P}, \notag 
\end{align}
are the push-forward measures on $(\mathbb{X},\mathcal{B}(\mathbb{X}))$ and $(\mathbb{Y},\mathcal{B}(\mathbb{Y}))$, representing the laws of the random variables $\boldsymbol{X}$ and $\boldsymbol{Y}$, respectively. We denote $\boldsymbol{\Gamma}(\boldsymbol{\Phi},\boldsymbol{\Psi})$ as a joint probability measure on $(\mathbb{X}\times\mathbb{Y}, \mathcal{B}(\mathbb{X}) \otimes \mathcal{B}(\mathbb{Y}))$ with marginals $\boldsymbol{\Phi}$ and $\boldsymbol{\Psi}$; i.e. a coupling.
\begin{rem}
\label{productcoupling}
If $\boldsymbol{X}$ and $\boldsymbol{Y}$ are mutually independent, $\boldsymbol{\Gamma}(\boldsymbol{\Phi},\boldsymbol{\Psi})$ is the product measure $\boldsymbol{\Gamma} = \boldsymbol{\Phi} \otimes \boldsymbol{\Psi}$.
\end{rem}

In this paper, $\{\boldsymbol{Z}_t\}_{t\in\T}$ denotes an $\R^n$-valued stochastic process we call the \emph{driving} process, which is right-continuous with left-limits (c\`adl\`ag). To begin, we do not need to specify $\{\boldsymbol{Z}_t\}_{t\in\T}$ any further and denote its natural filtration by $\{\mathcal{F}_t^{\boldsymbol{Z}}\}_{t\in\T}$ such that $\mathcal{F}_t^{\boldsymbol{Z}} \subset \mathcal{F}_t$ is a sub-$\sigma$-algebra for every $t\in\T$. Let the push-forward measures of the joint, initial and terminal laws of the driving process $\{\boldsymbol{Z}_t\}_{t\in\T}$ be denoted as 
\begin{align}
\mu_{\boldsymbol{Z}_0,\boldsymbol{Z}_T} = (\boldsymbol{Z}_0,\boldsymbol{Z}_T)_\sharp \mathbb{P} \,\,\,\, \text{and} \,\,\,\, \mu_{\boldsymbol{Z}_0} = (\boldsymbol{Z}_0)_\sharp \mathbb{P} \,\,\,\, \text{and} \,\,\,\, \mu_{\boldsymbol{Z}_T} = (\boldsymbol{Z}_T)_\sharp \mathbb{P}, \notag 
\end{align}
respectively. Crucially, our framework imposes no requirement for the existence of Lebesgue densities for the endpoint distributions; this allows the initial and terminal laws to be continuous, discrete, or singular (e.g. Dirac measure). Under this generality, we require only that the Radon-Nikodym derivative $\dd \boldsymbol{\Gamma} / \dd\mu_{\boldsymbol{Z}_0,\boldsymbol{Z}_T}$ exists. Hence, we only ask for the absolute continuity condition 
\begin{align}
\boldsymbol{\Gamma} \ll \mu_{\boldsymbol{Z}_0,\boldsymbol{Z}_T}  \,\,\,\, \implies  \,\,\,\, \boldsymbol{\Phi} \ll \mu_{\boldsymbol{Z}_0} \,\,\,\, \text{and} \,\,\,\, \boldsymbol{\Psi} \ll \mu_{\boldsymbol{Z}_T}. \notag
\end{align}
Since we work on Polish spaces, the existence of regular conditional probability distributions for conditioned processes to follow are guaranteed. We can now define the main mathematical object of this work in the spirit of \cite{8,11}, while including random-initialization.
\begin{defn}
\label{definitionRandomBridge}
If $\{\xib_t\}_{t\in\T}$ is an $\R^n$-valued $\{\F_{t}\}$-adapted stochastic process such that 
\begin{enumerate}
\item $(\xib_0,\xib_T) \sim \boldsymbol{\Gamma}(\boldsymbol{\Phi},\boldsymbol{\Psi})$, 
\item For all $m\in\mathbb{N}_{+}$, every $0 < t_{1}<\ldots<t_{m}<T$ and $(\boldsymbol{z}_{1},\ldots,\boldsymbol{z}_{m})\in\R^{n\times m}$,
\begin{align}
&\mathbb{P}\left(\xib_{t_{1}}\leq \boldsymbol{z}_{1},\ldots, \xib_{t_{m}}\leq \boldsymbol{z}_{m} \, \left| \, \xib_0=\boldsymbol{x}, \xib_T=\boldsymbol{y}\right.\right)= 
    \mathbb{P}\left(\boldsymbol{Z}_{t_{1}}\leq \boldsymbol{z}_{1},\ldots, \boldsymbol{Z}_{t_{m}}\leq \boldsymbol{z}_{m} \, \left| \, \boldsymbol{Z}_0=\boldsymbol{x}, \boldsymbol{Z}_T=\boldsymbol{y}\right.\right) \nonumber
\end{align}
for $\boldsymbol{\Gamma}$-almost-every $(\boldsymbol{x},\boldsymbol{y}) \in \mathbb{X}\times \mathbb{Y}$,
\end{enumerate}
then $\{\xib_t\}_{t\in\T}$ is a $(\boldsymbol{\Phi}, \boldsymbol{\Psi})$-bridge under action of $\{\boldsymbol{Z}_t\}_{t\in\T}$.
\end{defn}
Definition \ref{definitionRandomBridge} is a statement with almost no model-specific assumptions. Although this may look almost \emph{too} abstract to be useful at first sight, such a definition, in essence, provides a concise mathematical recipe for constructing and simulating different bridge-patterns. We highlight that all vector inequalities should be understood component-wise, and expressions involving point-wise conditioning  can be rigorously interpreted via the existence of regular conditional probability distributions on Polish spaces; herein, $\R^n$ and its subsets.

\begin{rem}
\label{rem:time_reversal_equity}
Under \cadlag regularization of the time-reversed driving process, Definition \ref{definitionRandomBridge} is symmetric; reversing a $(\boldsymbol{\Phi}, \boldsymbol{\Psi})$-bridge yields a $(\boldsymbol{\Psi}, \boldsymbol{\Phi})$-bridge with transposed endpoint coupling.
\end{rem}

\begin{rem}
\label{remstochastictransport}
We can set $\xib_0 = \boldsymbol{X}$ and $\xib_T = \boldsymbol{Y} $ from Definition \ref{definitionRandomBridge}, since $\boldsymbol{\Gamma}(\boldsymbol{\Phi},\boldsymbol{\Psi})$ is a coupling. Accordingly, $(\boldsymbol{\Phi}, \boldsymbol{\Psi})$-bridges act as stochastic transports between $\boldsymbol{X}$ and $\boldsymbol{Y}$. 
\end{rem}

Although a given c\`adl\`ag $\{\boldsymbol{Z}_t\}_{t\in\T}$ can generate a stochastic map $\{\xib_t\}_{t\in\T}$ between the two measures, $\{\xib_{t}\}_{t\in\T}$ does not necessarily share the same law as $\{\boldsymbol{Z}_{t}\}_{t\in\T}$ due to Property 1. in Definition \ref{definitionRandomBridge}, as formalized below.
\begin{prop}
The equality in law $\{\xib_{t}\}_{t\in\T} \law\{\boldsymbol{Z}_{t}\}_{t\in\T}$ holds if and only if $(\boldsymbol{Z}_0,\boldsymbol{Z}_T) \sim \boldsymbol{\Gamma}(\boldsymbol{\Phi},\boldsymbol{\Psi})$.
\end{prop}
\begin{proof}
By the absolute continuity condition $\boldsymbol{\Gamma} \ll \mu_{\boldsymbol{Z}_0,\boldsymbol{Z}_T}$, any measurable set $A \in \mathcal{B}(\mathbb{X}) \otimes \mathcal{B}(\mathbb{Y})$ satisfies $\boldsymbol{\Gamma}(A) > 0 \implies \P((\boldsymbol{Z}_0,\boldsymbol{Z}_T) \in A) > 0$ (and analogously for the marginals $\boldsymbol{\Phi} \ll \mu_{\boldsymbol{Z}_0}$ and $\boldsymbol{\Psi} \ll \mu_{\boldsymbol{Z}_T}$). Consequently, the regular conditional probabilities given the joint boundaries are well-defined $\boldsymbol{\Gamma}$-almost everywhere. If $(\boldsymbol{Z}_0,\boldsymbol{Z}_T) \sim \boldsymbol{\Gamma}(\boldsymbol{\Phi},\boldsymbol{\Psi})$, then $\P(\boldsymbol{Z}_0 \in \dd \boldsymbol{x})=\boldsymbol{\Phi}(\dd \boldsymbol{x})$ and $\P(\boldsymbol{Z}_T \in \dd \boldsymbol{y})=\boldsymbol{\Psi}(\dd \boldsymbol{y})$, and the unconditional joint boundary law of the driving process satisfies:
\begin{align}
\P(\boldsymbol{Z}_0 \in \dd \boldsymbol{x}, \boldsymbol{Z}_T \in \dd \boldsymbol{y}) = \boldsymbol{\Gamma}(\dd \boldsymbol{x}, \dd \boldsymbol{y}). \notag 
\end{align}
By disintegrating the finite-dimensional distributions and applying Property 2 of Definition \ref{definitionRandomBridge} for $\boldsymbol{\Gamma}$-almost every $(\boldsymbol{x},\boldsymbol{y}) \in \mathbb{X}\times \mathbb{Y}$, we have:
\begin{align}
&\P\left(\boldsymbol{\xi}_{0}\in  \dd\boldsymbol{x}, \boldsymbol{\xi}_{t_{1}}\leq \boldsymbol{z}_{1},\ldots, \boldsymbol{\xi}_{t_{m}}\leq \boldsymbol{z}_{m} \, , \,  \boldsymbol{\xi}_T \in \dd \boldsymbol{y}\right) \nonumber \\
&= \P\left(\boldsymbol{Z}_{0}\in \dd\boldsymbol{x},\boldsymbol{Z}_{t_{1}}\leq \boldsymbol{z}_{1},\ldots, \boldsymbol{Z}_{t_{m}}\leq \boldsymbol{z}_{m} \, , \,  \boldsymbol{Z}_T \in \dd \boldsymbol{y}\right), \notag
\end{align}
for every $0 < t_{1}<\ldots<t_{m}<T$ and $(\boldsymbol{z}_{1},\ldots,\boldsymbol{z}_{m})\in\mathbb{R}^{n\times m}$. The forward statement follows from the Kolmogorov extension theorem since $\{\boldsymbol{\xi}_{t}\}_{t\in\mathbb{T}}$ and $\{\boldsymbol{Z}_{t}\}_{t\in\mathbb{T}}$ are c\`adl\`ag on the Skorokhod path space. 
Conversely, if $\{\boldsymbol{\xi}_{t}\}_{t\in\mathbb{T}} \overset{\text{law}}{=} \{\boldsymbol{Z}_{t}\}_{t\in\mathbb{T}}$, their laws must coincide at the initial and terminal timepoints $t=0$ and $t=T$; hence, $(\boldsymbol{Z}_0, \boldsymbol{Z}_T) \overset{\text{law}}{=} (\boldsymbol{\xi}_0, \boldsymbol{\xi}_T)$ must hold. By Property 1 of Definition \ref{definitionRandomBridge}, we have $(\boldsymbol{\xi}_0, \boldsymbol{\xi}_T) \sim \boldsymbol{\Gamma}(\boldsymbol{\Phi},\boldsymbol{\Psi})$, which immediately yields $(\boldsymbol{Z}_0,\boldsymbol{Z}_T) \sim \boldsymbol{\Gamma}(\boldsymbol{\Phi},\boldsymbol{\Psi})$.
\end{proof}

Note that $\{\xib_t\}_{t\in\T}$ can host continuous, discontinuous or hybrid dynamics depending entirely on the structural specification of the driving process $\{\boldsymbol{Z}_t\}_{t\in\T}$. In order to construct empirical channels for sample generation, we define the regular conditional process $\{\xib^{\boldsymbol{x}}_t\}_{t\in\T}$ initialized at a specific source point as follows:
\begin{align}
\xib^{\boldsymbol{x}}_t \triangleq  \xib_t \, |_{\boldsymbol{X} = \boldsymbol{x} \, , \,\boldsymbol{x} \leftarrow \boldsymbol{\Phi}} \hspace{0.1in} \forall t\in\T, \label{sampledinit}
\end{align}
where $\boldsymbol{x} \leftarrow \boldsymbol{\Phi}$ denotes a fixed $\boldsymbol{x}\in\mathbb{X}$ sampled from $\boldsymbol{\Phi}$. We call (\ref{sampledinit}) a \emph{$\boldsymbol{\Phi}$-initialized random-bridge to $\boldsymbol{\Psi}$}. For the rest of this paper, we let $\boldsymbol{\Gamma} = \boldsymbol{\Phi} \otimes \boldsymbol{\Psi}$ in line with Remark \ref{productcoupling}, and set $\xib^{\boldsymbol{x}}_T = \boldsymbol{Y}\sim \boldsymbol{\Psi}$ in line with Remark \ref{remstochastictransport}. Our assumption of mutual independence between $\boldsymbol{X}$ and $\boldsymbol{Y}$ in what follows serves solely to keep the mathematical expressions parsimonious and can be readily relaxed.

\begin{rem}
\label{reminitializedrb}
For any fixed $\boldsymbol{x}\in\mathbb{X}$, $\{\xib^{\boldsymbol{x}}_t\}_{t\in\T}$ defines a $(\boldsymbol{\delta}_{\boldsymbol{x}}, \boldsymbol{\Psi})$-bridge under action of $\{\boldsymbol{Z}_t\}_{t\in\T}$ given that $\boldsymbol{Z}_0 =\boldsymbol{x}$, where $\boldsymbol{\delta}_{\boldsymbol{x}}$ is the Dirac measure centered at $\boldsymbol{x}\in\mathbb{X}$.
\end{rem}
Regarding Remark \ref{productcoupling} and Remark \ref{reminitializedrb}, with $\boldsymbol{\Gamma}(\boldsymbol{\delta}_{\boldsymbol{x}},\boldsymbol{\Psi})$ a coupling on $(\mathbb{X}\times\mathbb{Y}, \mathcal{B}(\mathbb{X}) \otimes \mathcal{B}(\mathbb{Y}))$, we have the following:
\begin{align}
(\xib^{\boldsymbol{x}}_0,\xib^{\boldsymbol{x}}_T) \sim \boldsymbol{\Gamma}(\boldsymbol{\delta}_{\boldsymbol{x}},\boldsymbol{\Psi}) = \boldsymbol{\delta}_{\boldsymbol{x}}\otimes\boldsymbol{\Psi}
\end{align}
for  any fixed $\boldsymbol{x}\in\mathbb{X}$. This essentially means that if $\{\xib^{\boldsymbol{x}}_t\}_{t\in\T}$ is a $(\boldsymbol{\delta}_{\boldsymbol{x}}, \boldsymbol{\Psi})$-bridge under action of $\{\boldsymbol{Z}_t\}_{t\in\T}$, then $\xib^{\boldsymbol{x}}_0 = \boldsymbol{Z}_0 = \boldsymbol{x}$ must hold almost-surely. It further means that Property 2 in Definition \ref{definitionRandomBridge} can be reduced to 
\begin{align}
\mathbb{P}&\left(\xib^{\boldsymbol{x}}_{t_{1}}\leq \boldsymbol{z}_{1},\ldots, \xib^{\boldsymbol{x}}_{t_{m}}\leq \boldsymbol{z}_{m} \, \left| \, \xib^{\boldsymbol{x}}_T=\boldsymbol{y}\right.\right)= 
    \mathbb{P}\left(\boldsymbol{Z}_{t_{1}}\leq \boldsymbol{z}_{1},\ldots, \boldsymbol{Z}_{t_{m}}\leq \boldsymbol{z}_{m} \, \left| \, \boldsymbol{Z}_T=\boldsymbol{y}\right.\right) \nonumber
\end{align}
for all $m\in\mathbb{N}_{+}$, every $0 < t_{1}<\ldots<t_{m}<T$ and $(\boldsymbol{z}_{1},\ldots,\boldsymbol{z}_{m})\in\R^{n\times m}$, for $\boldsymbol{\Psi}$-almost-everywhere $\boldsymbol{y}\in\mathbb{Y}$. The following result shows $\{\boldsymbol{\xi}^{\boldsymbol{x}}_t\}_{t\in\T}$ is a Markov vs. non-Markov process depending on the driving process $\{\boldsymbol{Z}_t\}_{t\in\T}$ -- the result below can be found in \cite{11,15}.
\begin{prop}
\label{markovprop}
Let $\{\F_{t}^{\xib^{\boldsymbol{x}}}\}_{t\in\T}$ be the filtration of $\{\boldsymbol{\xi}^{\boldsymbol{x}}_t\}_{t\in\T}$. Any $\{\xib^{\boldsymbol{x}}_t\}_{t\in\T}$ is Markov with respect to $\{\F_t^{\xib^{\boldsymbol{x}}}\}$ if its driving process $\{\boldsymbol{Z}_t\}_{t\in\T}$ is Markov with respect to $\{\F_t^{\boldsymbol{Z}}\}$.
\end{prop}
Proposition \ref{markovprop} will be foundational for our practical empirical applications later in the paper due to the computational efficiency it brings forward. Let $\mathcal{P}(\mathbb{Y})$ be the space of probability measures over $(\mathbb{Y},\mathcal{B}(\mathbb{Y}))$ and $\{\pi_{t}\}_{t\in\T}$ be a $\mathcal{P}(\mathbb{Y})$-valued $\{\mathcal{F}^{\xib^{\boldsymbol{x}}}_{t}\}$-adapted stochastic process representing the posterior distribution of the target given the filtration:
\begin{align} 
\pi_{t}(\dd \boldsymbol{y}) \triangleq \mathbb{P}(\boldsymbol{Y} \in \dd \boldsymbol{y} \,| \,\mathcal{F}^{\xib^{\boldsymbol{x}}}_{t}) \hspace{0.1in} \text{$\forall t\in\T$}. \label{measurevaluedprocess}
\end{align}
Since $\boldsymbol{\xi}^{\boldsymbol{x}}_{T} = \boldsymbol{Y}$ almost surely, we have $\mathbb{P}(\xib^{\boldsymbol{x}}_{T} \in \dd \boldsymbol{y}) = \mathbb{P}(\boldsymbol{Y} \in \dd \boldsymbol{y})$. If $\{\boldsymbol{Z}_t\}_{t\in\T}$ is Markov with respect to $\{\F_t^{\boldsymbol{Z}}\}$, then using Proposition \ref{markovprop}, we have the following:
\begin{align}
\pi_{t}(\dd \boldsymbol{y}) &= \mathbb{P}(\boldsymbol{Y} \in \dd \boldsymbol{y} \,| \,\xib^{\boldsymbol{x}}_{t}= \boldsymbol{z}) = \frac{\mathbb{P}(\xib^{\boldsymbol{x}}_{t} \in \dd \boldsymbol{z} \,|\, \boldsymbol{Y} =  \boldsymbol{y})\mathbb{P}(\boldsymbol{Y} \in \dd \boldsymbol{y})}{\int_{\mathbb{Y}}\mathbb{P}(\xib^{\boldsymbol{x}}_{t} \in \dd \boldsymbol{z} \,|\, \boldsymbol{Y} =  \boldsymbol{y}')\mathbb{P}(\boldsymbol{Y} \in \dd \boldsymbol{y}')} \notag \\
&= \frac{\mathbb{P}(\xib^{\boldsymbol{x}}_{t} \in \dd \boldsymbol{z} \,|\, \xib^{\boldsymbol{x}}_T = \boldsymbol{y})\boldsymbol{\Psi}(\dd\boldsymbol{y})}{\int_{\mathbb{Y}}\mathbb{P}(\xib^{\boldsymbol{x}}_{t} \in \dd \boldsymbol{z} \,|\, \xib^{\boldsymbol{x}}_{T} = \boldsymbol{y}')\boldsymbol{\Psi}(\dd\boldsymbol{y}')} = \mathbb{P}(\xib^{\boldsymbol{x}}_{T} \in \dd \boldsymbol{y} | \xib^{\boldsymbol{x}}_{t}= \boldsymbol{z}), \notag
\end{align}
using Bayes' formula. This structural simplification allows us to compute conditional expectations via the spatial integral given by
\begin{align}
\E[ \boldsymbol{Y} \,| \,\mathcal{F}^{\xib^{\boldsymbol{x}}}_{t}] = \int_{\mathbb{Y}}\boldsymbol{y}\pi_{t}(\dd \boldsymbol{y}) = \E\left[ \boldsymbol{Y} \,| \,\xib^{\boldsymbol{x}}_{t} \right], \label{markovconditionalexpectation}
\end{align}
which will serve as our primary target for neural network approximations during training and simulation. If $\mathbb{Y}$ is a discrete state-space, the measure $\pi_{t}(\dd \boldsymbol{y})$ collapses naturally to the atomic counting measure, yielding the standard discrete probability mass function $\pi_{t}(\boldsymbol{y}) \triangleq \mathbb{P}(\boldsymbol{Y}  = \boldsymbol{y} \mid \mathcal{F}^{\boldsymbol{\xi}^{\boldsymbol{x}}}_{t})$ for every  $t\in\T$.  Consequently, if $\{\boldsymbol{Z}_t\}_{t\in\T}$ is Markovian with respect to $\{\F_t^{\boldsymbol{Z}}\}$, the conditional expectation simplifies to the discrete summation:
\begin{align}
\E[ \boldsymbol{Y} \,| \,\mathcal{F}^{\xib^{\boldsymbol{x}}}_{t}] = \sum_{\boldsymbol{y} \in\mathbb{Y}}\boldsymbol{y}\pi_{t}(\boldsymbol{y}) = \E\left[ \boldsymbol{Y} \,| \,\xib^{\boldsymbol{x}}_{t} \right]. \label{markovconditionalexpectationdiscrete}
\end{align}

\section{Gaussian Random-Bridge Transport}
For our practical objective, we shall focus on an analytically-tractable family of random-bridges, namely Gaussian random-bridges. Accordingly, for the remainder of this section, let $\mathbf{Z}_0 = \boldsymbol{x}$ where $\{\mathbf{Z}_t\}_{t\in\T}$ is an $\R^n$-valued \emph{Gaussian} process, mutually independent of $\boldsymbol{X}$ and $\boldsymbol{Y}$, with the following:
\begin{align}
\boldsymbol{\mu}_t = \E[\boldsymbol{Z}_t] \in\R^n \hspace{0.1in}\text{and} \hspace{0.1in} \boldsymbol{\Sigma}_{s,t} = \text{Cov}[\boldsymbol{Z}_s, \boldsymbol{Z}_t] \in \R^{n\times n} \notag
\end{align}
as the mean and the covariance kernel, respectively, for all $\{s,t\}\in\T$, where $\boldsymbol{\Sigma}^{(i,j)}_{s,t} = \text{Cov}[Z^{(i)}_s, Z^{(j)}_t]$ for every $i,j \in \{1,\ldots,n\}$. Hence, if $\mathcal{N}(.)$ denotes the Gaussian distribution, we have 
\begin{align}
\mathbf{Z}_t \sim \mathcal{N}( \boldsymbol{\mu}_t \, , \, \boldsymbol{\Sigma}_{t,t}) \hspace{0.1in} \forall t\in\T. \notag
\end{align}
In order to guarantee that the multivariate stochastic representations of the resulting bridges do not degenerate, we introduce the following structural requirement on the covariance operators.
\begin{asu}
\label{assumpcovariance}
The covariance kernel $\boldsymbol{\Sigma}_{s,t}$ is strictly positive-definite for all $s,t \in \T\setminus\{0\}$. 
\end{asu}
Assumption \ref{assumpcovariance} means that for any finite collection of timepoints $t_1, \ldots, t_m \in \T\setminus\{0\}$, the joint covariance matrix is strictly positive-definite. Consequently, the marginal spatial covariance matrix $\boldsymbol{\Sigma}_{t,t}$ is invertible for all $t \in \T\setminus\{0\}$. Hence, given a sample $\boldsymbol{x} \leftarrow \boldsymbol{\Phi}$, under Assumption \ref{assumpcovariance}, the process $\{\xib^{\boldsymbol{x}}_t\}_{t\in\T}$ forms a well-defined $\boldsymbol{\Phi}$-initialized \emph{Gaussian} random-bridge to $\boldsymbol{\Psi}$ following Definition \ref{definitionRandomBridge} and \eqref{sampledinit}. This yields the following multivariate generalisation of \cite{9}.

\begin{prop}
\label{initialgaussiandecomp}
Let $\boldsymbol{\Sigma}^{*}_{t,T}=\boldsymbol{\Sigma}_{t,T}\boldsymbol{\Sigma}^{-1}_{T,T}$, where $\boldsymbol{\Sigma}^{-1}_{T,T}$ is the precision-matrix. Then $\{\xib^{\boldsymbol{x}}_t\}_{t\in\T}$ admits the following anticipative representation:
\begin{align}
\xib^{\boldsymbol{x}}_t \law \boldsymbol{\Sigma}^{*}_{t,T}\boldsymbol{Y} + \left(\boldsymbol{Z}_t - \boldsymbol{\Sigma}^{*}_{t,T}\boldsymbol{Z}_T\right). \label{canonicalGRB}
\end{align}
\end{prop}
\begin{proof}
Since $\boldsymbol{\Sigma}_{T,T}$ is strictly positive-definite under Assumption \ref{assumpcovariance}, the precision matrix $\boldsymbol{\Sigma}^{-1}_{T,T}$ exists.  Since $\mathbf{Z}_0 = \boldsymbol{x}$, we have $\boldsymbol{\Sigma}_{0,T}=0=\boldsymbol{\Sigma}^{*}_{0,T}$, and hence, $\xib^{\boldsymbol{x}}_0 \sim \boldsymbol{Z}_0=\boldsymbol{x}$. At $t=T$, we have 
\begin{align}
\boldsymbol{\Sigma}^{*}_{T,T} = \boldsymbol{I} \, \Rightarrow \, \xib^{\boldsymbol{x}}_T \law \boldsymbol{Y}, \notag
\end{align}
where $\boldsymbol{I}$ is the identity-matrix, which implies $(\xib^{\boldsymbol{x}}_0,\xib^{\boldsymbol{x}}_T) \sim \boldsymbol{\delta}_{\boldsymbol{x}}\otimes\boldsymbol{\Psi}$. Hence, Property 1 of Definition \ref{definitionRandomBridge} is satisfied. In order to verify Property 2 of Definition \ref{definitionRandomBridge}, define 
\begin{align}
\mathbf{N}_t \triangleq \boldsymbol{Z}_t - \boldsymbol{\Sigma}^{*}_{t,T}\boldsymbol{Z}_T, \label{noiseprocessn}
\end{align}
for all $t \in \T$. Since $\{\mathbf{Z}_t\}_{t\in\T}$ is a Gaussian process, $\{\mathbf{N}_t\}_{t\in\T}$ is also a Gaussian process with $\mathbb{E}[\mathbf{N}_t] = \boldsymbol{\mu}_t - \boldsymbol{\Sigma}^{*}_{t,T}\boldsymbol{\mu}_T$. In addition, we have
\begin{align}
\text{Cov}(\mathbf{N}_t, \mathbf{Z}_T) &= \text{Cov}(\boldsymbol{Z}_t - \boldsymbol{\Sigma}_{t,T}\boldsymbol{\Sigma}^{-1}_{T,T}\boldsymbol{Z}_T, \mathbf{Z}_T) = \text{Cov}(\boldsymbol{Z}_t, \mathbf{Z}_T) - \boldsymbol{\Sigma}_{t,T}\boldsymbol{\Sigma}^{-1}_{T,T}\text{Cov}(\mathbf{Z}_T, \mathbf{Z}_T) \notag \\
&= \boldsymbol{\Sigma}_{t,T} - \boldsymbol{\Sigma}_{t,T}\boldsymbol{\Sigma}^{-1}_{T,T}\boldsymbol{\Sigma}_{T,T} = \boldsymbol{\Sigma}_{t,T} - \boldsymbol{\Sigma}_{t,T} = \mathbf{0}. \notag
\end{align}
Since $\boldsymbol{N}_t$ and $\boldsymbol{Z}_T$ are jointly Gaussian, their orthogonality $\text{Cov}(\boldsymbol{N}_t, \boldsymbol{Z}_T) = \mathbf{0}$ implies strict probabilistic independence. Furthermore, since $\boldsymbol{Y}$ is independent of the driving process by construction, $\mathbf{N}_t$ is independent of the joint pair $(\boldsymbol{Z}_T, \boldsymbol{Y})$. Thus, conditioning the representation (\ref{canonicalGRB}) on $\boldsymbol{\xi}^{\boldsymbol{x}}_T = \boldsymbol{Y} = \boldsymbol{y}$ does not alter the distribution of $\mathbf{N}_t$ for any $t\in\T$. Using standard multivariate Gaussian conditioning rules, the regular conditional distribution of the driving process $\{\boldsymbol{Z}_t\}_{t\in\T}$ given $\boldsymbol{Z}_T = \boldsymbol{y}$ is $\boldsymbol{\Sigma}^*_{t,T}\boldsymbol{y} + \mathbf{N}_t$, which has the same mean and covariance structure of the conditioned driving process. Since the conditional finite-dimensional distributions match identically for $\boldsymbol{\Psi}$-almost-everywhere $\boldsymbol{y} \in \mathbb{Y}$, Property 2 of Definition \ref{definitionRandomBridge} is satisfied.
\end{proof}
Note that the random variables $\boldsymbol{Y}$, $\boldsymbol{Z}_t$ and $\boldsymbol{Z}_T$ are \emph{not} $\mathcal{F}^{\xib^{\boldsymbol{x}}}_{t}$-measurable for $t < T$ (except that $\boldsymbol{Z}_0$ is $\mathcal{F}^{\xib^{\boldsymbol{x}}}_{0}$-measurable), which is the reason we call (\ref{canonicalGRB}) an \emph{anticipative} representation. 
\begin{rem}
\label{reminforem}
The additive decomposition in (\ref{canonicalGRB}) can be interpreted as a \emph{noisy information} process in stochastic filtering theory, where $\boldsymbol{Y}$ is the hidden target signal and $\{\boldsymbol{N}_t\}_{t\in\T}$ in (\ref{noiseprocessn}) is a noise process obscuring the signal $\boldsymbol{Y}$ for any $ t < T$.
\end{rem}
Remark \ref{reminforem} follows in the spirit of the fruitful information-based literature -- see \cite{4,5,6,7,8} as seminal papers in this stream -- for solving problems in mathematical physics towards quantum measurement theory and in mathematical finance towards pricing derivatives (also, see \cite{10a,9aa,10,10b,15} for various mathematical generalisations). The next statement will be highly useful for \emph{training} purposes in generative modelling, which we shall detail in the next section.
\begin{coro}
\label{corodistribution}
Let $\{\xib^{\boldsymbol{x}}_t\}_{t\in\T}$ be a $\boldsymbol{\Phi}$-initialized Gaussian random-bridge to $\boldsymbol{\Psi}$. Then
\begin{align}
\xib^{\boldsymbol{x}}_t \, |_{\boldsymbol{Y} = \boldsymbol{y} \, , \, \boldsymbol{y} \leftarrow \boldsymbol{\Psi}} \sim \mathcal{N}\left(\boldsymbol{E}_t \, , \, \boldsymbol{V}_t \right), \notag
\end{align}
where the mean and variance functions are given by
\begin{align}
\boldsymbol{E}_t  = \boldsymbol{\mu}_t + \boldsymbol{\Sigma}^{*}_{t,T}\left(\boldsymbol{y} - \boldsymbol{\mu}_T\right)  \hspace{0.1in}\text{and} \hspace{0.1in}
\boldsymbol{V}_t  = \boldsymbol{\Sigma}_{t,t} - \boldsymbol{\Sigma}^{*}_{t,T}\boldsymbol{\Sigma}_{T,t} \notag
\end{align}
for every $t\in(0,T)$.
\end{coro}
\begin{proof}
Under Assumption \ref{assumpcovariance}, we have  $\boldsymbol{\Sigma}_{t,t} \succ \mathbf{0}$ for all $t \in \T\setminus\{0\}$. Next, define the following:
\begin{align}
\boldsymbol{\Sigma}_{T|t} \triangleq \boldsymbol{\Sigma}_{T,T} - \boldsymbol{\Sigma}_{t,T}^\tp \boldsymbol{\Sigma}_{t,t}^{-1} \boldsymbol{\Sigma}_{t,T} \notag
\end{align}
which, under Assumption \ref{assumpcovariance}, must also be strictly non-singular for any $t \in (0,T)$, satisfying $\boldsymbol{\Sigma}_{T|t} \succ \mathbf{0}$. Since $\boldsymbol{\Sigma}_{T|t}$ and $\boldsymbol{V}_t = \boldsymbol{\Sigma}_{t,t} - \boldsymbol{\Sigma}^{*}_{t,T}\boldsymbol{\Sigma}_{T,t}$ represent the dual Schur complements of the same symmetric joint covariance matrix over the temporal pair $\{t,T\}$, we have $\boldsymbol{V}_t \succ \mathbf{0}$ for all $t \in (0,T)$. The statement follows from Proposition \ref{initialgaussiandecomp} given $\boldsymbol{Y}=\boldsymbol{y}$ by taking the expectation and variance of the representation in (\ref{canonicalGRB}).
\end{proof}

\begin{rem}
\label{remindependencedependence}
Since $\{\mathbf{Z}_t\}_{t\in\T}$ is mutually independent from the random variables $\boldsymbol{X}$ and $\boldsymbol{Y}$, its underlying covariance structure does not alter the intrinsic structures that exist within the boundary distributions $\boldsymbol{\Phi}$ and $\boldsymbol{\Psi}$. 
\end{rem}
Remark \ref{remindependencedependence} highlights a core feature of this framework: it allows the driving process $\{\boldsymbol{Z}_t\}_{t\in\T}$ to be selected completely a priori without any knowledge of the data distributions $\boldsymbol{\Phi}$ and $\boldsymbol{\Psi}$, leaving the bridge $\{\boldsymbol{\xi}_t\}_{t\in\T}$ to capture and resolve the necessary coupling pathways.
\begin{coro}
\label{exbrowniananticipative}
Let $\boldsymbol{\sigma}\in\R^{n\times n}$ be a diagonal-matrix with diagonals $\sigma^{(i)} > 0$ for $i = 1,\ldots,n$. Let $\{\mathbf{Z}_t = \boldsymbol{\sigma}\mathbf{W}_t\}_{t\in\T}$, where $\{\mathbf{W}_t\}_{t\in\T}$ is a standard Brownian motion with mutually-independent coordinates with $W^{(i)}_0 = x^{(i)}/\sigma^{(i)}\in\R$ for $i = 1,\ldots,n$. Then each coordinate of $\{\xib^{\boldsymbol{x}}_t\}_{t\in\T}$ satisfies the anticipative representation:
\begin{align}
\xi^{(i)}_t \law \frac{t}{T}Y^{(i)} + \sigma^{(i)}\left(W^{(i)}_t - \frac{t}{T}W^{(i)}_T\right), \notag
\end{align}
for every $t\in\T$ and $i=1,\ldots,n$. In addition, 
\begin{align}
\boldsymbol{E}_t = \boldsymbol{x} + \frac{t}{T}(\boldsymbol{y} - \boldsymbol{x}) \hspace{0.1in} \text{and} \hspace{0.1in}
\boldsymbol{V}_t = \frac{t(T - t)}{T}\boldsymbol{\sigma}^2, \hspace{0.1in} \text{$\forall t\in(0,T)$}. \label{conditionaldistroneed}
\end{align}
\end{coro}
\begin{proof}
Given the initial state $W^{(i)}_0 = x^{(i)}/\sigma^{(i)}$, we have $\mathbf{Z}_0 = \boldsymbol{x}$, yielding $\boldsymbol{\mu}_t = \boldsymbol{x}$, $\boldsymbol{\Sigma}_{s,t} = \min(s,t)\boldsymbol{\sigma}^2$ and $\boldsymbol{\Sigma}^*_{t,T} = (t/T)\boldsymbol{I}$ for every $s,t\in\T$. Substituting these into Corollary \ref{corodistribution} gives the statement.
\end{proof}
Having $\{\mathbf{Z}_t\}_{t\in\mathbb{T}}$ a Markov process presents valuable analytical tractability and computational efficiency. Although we worked with the anticipative representation given in (\ref{canonicalGRB}) thus far, we can take the conditional expectation of both sides with respect to the $\sigma$-algebra $\mathcal{F}^{\xib^{\boldsymbol{x}}}_{t}$ to obtain a non-anticipative form. More specifically, using Proposition \ref{markovprop}, we can write the following:
\begin{align}
\xib^{\boldsymbol{x}}_t = \boldsymbol{\Sigma}^{*}_{t,T}\left( \mathbb{E}\left[ \boldsymbol{Y} -  \boldsymbol{Z}_T \, | \, \xib^{\boldsymbol{x}}_t \right]\right) +\mathbb{E}\left[ \boldsymbol{Z}_t \, | \, \xib^{\boldsymbol{x}}_t \right], \label{alternativenonanticipative}
\end{align}
which holds since $\mathbb{E}[ \xib^{\boldsymbol{x}}_t \, | \, \mathcal{F}^{\xib^{\boldsymbol{x}}}_{t} ] = \mathbb{E}\left[ \xib^{\boldsymbol{x}}_t \, | \, \xib^{\boldsymbol{x}}_t \right] = \xib^{\boldsymbol{x}}_t$ for all $t\in\mathbb{T}$. The identity in (\ref{alternativenonanticipative}) highlights a dynamic balance between the expected difference of the target signal and driving noise. If, in addition to being Markov, $\{\mathbf{Z}_t\}_{t\in\T}$ is also a martingale process, then we can reach a wide family of non-anticipative stochastic differential equations (SDEs), as per below.
\begin{prop}
\label{stochfilt}
Let each coordinate of $\{\mathbf{Z}_t\}_{t\in\T}$ be mutually independent. If $\{\mathbf{Z}_{t}\}_{t\in\T}$ is a time-continuous Gaussian martingale with strictly increasing quadratic-variation $\{\left\langle \mathbf{Z}\right\rangle_{t}\}_{t\in\T}$, where $\left\langle \mathbf{Z}\right\rangle_{0}=\boldsymbol{0}$, then $\{\xib^{\boldsymbol{x}}_t\}_{t\in\T}$ admits the following representation:
\begin{align}
%\xib^{\boldsymbol{x}}_t \law \int_0^t \boldsymbol{\Lambda}_{s,T}^{-1} \,\dd \left\langle \mathbf{Z}\right\rangle_s \left( \E\left[ \boldsymbol{Y} \, | \, \xib^{\boldsymbol{x}}_s \right]  - \xib^{\boldsymbol{x}}_s \right) + \mathbf{Z}_t, \notag
\xib^{\boldsymbol{x}}_t \law \int_0^t \boldsymbol{\Lambda}_{s,T}^{-1} \left( \mathbb{E}\left[ \boldsymbol{Y} \, | \, \xib^{\boldsymbol{x}}_s \right]  - \xib^{\boldsymbol{x}}_s \right) \odot \dd\left\langle \mathbf{\bar{Z}}\right\rangle_s + \mathbf{Z}_t, \notag
\end{align}
for every $t\in[0,T)$, where $\boldsymbol{\Lambda}_{t,T} = \left\langle \mathbf{Z}\right\rangle_{T} - \left\langle \mathbf{Z}\right\rangle_{t}$ is an $n\times n$ diagonal matrix, $\langle \mathbf{\bar{Z}}\rangle_t = \operatorname{diag}(\left\langle \mathbf{Z}\right\rangle_t)$ is an $n\times 1$ vector, and $\odot$ is the element-wise Hadamard product.
\end{prop}

\begin{proof}
Let $\pi_{t}(\dd \boldsymbol{y}) = \mathbb{P}(\boldsymbol{Y} \in \dd \boldsymbol{y} \,| \,\mathcal{F}^{\xib^{\boldsymbol{x}}}_{t})$ be the $\mathcal{P}(\mathbb{Y})$-valued stochastic process given in (\ref{measurevaluedprocess}). Having each Gaussian martingale $\{\mathbf{Z}_t\}_{t\in\T}$ mutually independent of $\boldsymbol{Y}$, and each coordinate of $\{\mathbf{Z}_t\}_{t\in\T}$ also mutually independent, the anticipative representation in (\ref{canonicalGRB}) reduces to 
\begin{align}
\xi^{(i)}_t \law \boldsymbol{\Sigma}^{*}_{t,T}[i,i]Y^{(i)} + \left(Z^{(i)}_t - \boldsymbol{\Sigma}^{*}_{t,T}[i,i]Z^{(i)}_T\right), \label{canonicalGRBcoordinate}
\end{align}
coordinate-wise for $i=1,\ldots,n$, where $\boldsymbol{\Sigma}^{*}_{t,T}[i,i]$ denotes the $i$th diagonal value of the $n\times n$ matrix $\boldsymbol{\Sigma}^{*}_{t,T}$ given by the following form:
\begin{align}
\boldsymbol{\Sigma}^{*}_{t,T}[i,i] = \frac{\left\langle Z^{(i)}\right\rangle_{t}}{\left\langle Z^{(i)}\right\rangle_{T}}. \notag
\end{align}
Continuous Gaussian martingales with strictly increasing quadratic-variation have independent increments, and hence, they are Markov processes. Thus, we can apply Proposition \ref{markovprop} and achieve $\pi_{t}(\dd \boldsymbol{y}) = \mathbb{P}(\boldsymbol{Y} \in \dd \boldsymbol{y} \,| \,\xib^{\boldsymbol{x}}_{t})$ for any $t\in\T$. Utilizing nonlinear filtering theory and applying It\^o's lemma on the posterior measure $\pi_{t}(\dd \boldsymbol{y})$ for any fixed $\boldsymbol{y}\in\mathbb{Y}$, we get the following non-anticipative SDE representation:
\begin{align}
\dd \pi_t(\dd\boldsymbol{y}) = \pi_t(\dd\boldsymbol{y}) \left( \boldsymbol{y} - \E\left[ \boldsymbol{Y} \, | \, \xib^{\boldsymbol{x}}_t \right] \right)^\top \boldsymbol{\Lambda}_{t,T}^{-1} \, \dd\mathbf{\hat{Z}}_t, \label{densitySDE}
\end{align}
for $t \in [0,T)$, where the innovation process $\{\mathbf{\hat{Z}}_t\}_{t\in\T}$ is an $\mathcal{F}^{\xib^{\boldsymbol{x}}}_{t}$-adapted continuous local martingale defined pathwise by
\begin{align}
\mathbf{\hat{Z}}_t = \xib^{\boldsymbol{x}}_t - \int_0^t \boldsymbol{\Lambda}_{s,T}^{-1} \,\dd \langle \mathbf{Z}\rangle_s \left( \E\left[ \boldsymbol{Y} \, | \, \xib^{\boldsymbol{x}}_s \right]  - \xib^{\boldsymbol{x}}_s \right), \label{initialsderep}
\end{align}
subject to $\xib^{\boldsymbol{x}}_T = \boldsymbol{Y}\sim \boldsymbol{\Psi}$, where $\boldsymbol{\Lambda}_{t,T}$ and $\langle \mathbf{Z}\rangle_t$ are strictly diagonal operators for $t\in[0,T)$ due to coordinate independence. Using (\ref{canonicalGRBcoordinate}) and (\ref{initialsderep}), the innovation process satisfies $\langle \mathbf{\hat{Z}} \rangle_t = \langle \xib^{\boldsymbol{x}} \rangle_t = \langle \mathbf{Z} \rangle_t$ with cross-coordinates satisfying $\left\langle Z^{(i)},Z^{(j)}\right\rangle_{t} = 0$ for any $i\neq j$ and every $t\in[0,T)$. Given that $\{\mathbf{Z}_t\}_{t\in\T}$ is a continuous martingale with a strictly increasing $\{\left\langle \mathbf{Z}\right\rangle_{t}\}_{t\in\T}$ with $\left\langle \mathbf{Z}\right\rangle_{0}=\boldsymbol{0}$, the Dambis--Dubins--Schwarz theorem establishes that any continuous martingale sharing this bracket is uniquely determined in distribution. Since $\langle \mathbf{\hat{Z}} \rangle_t = \langle \mathbf{Z} \rangle_t$ for all $t \in [0,T)$, it follows that 
\begin{align}
\{\mathbf{\hat{Z}}_t\}_{0 \le t < T} \overset{\text{law}}{=} \{\mathbf{Z}_t\}_{0 \le t < T}. \notag 
\end{align}
Since the SDE (\ref{initialsderep}) satisfies the regularity conditions for weak uniqueness on $[0,T)$, substituting the innovation process $\{\mathbf{\hat{Z}}_t\}_{t\in[0,T)}$ with the driving process $\{\mathbf{Z}_t\}_{t\in[0,T)}$ uniquely preserves the law of $\{\boldsymbol{\xi}^{\boldsymbol{x}}_t\}_{t\in[0,T)}$ and the statement follows.
\end{proof}
Note that using (\ref{densitySDE}), we get a closed-form analytical solution given by the stochastic Dol\'eans-Dade exponential martingale:
\begin{align}\label{eq:doleans_exponential}
\pi_t(\dd\boldsymbol{y}) \overset{\text{law}}{=} \pi_0(\dd\boldsymbol{y})\exp \left( \int_0^t \boldsymbol{\beta}_s^\top \boldsymbol{\Lambda}_{s,T}^{-1} \, \dd\mathbf{Z}_s - \frac{1}{2} \int_0^t \left( \boldsymbol{\beta}_s^\top \boldsymbol{\Lambda}_{s,T}^{-2} \odot \boldsymbol{\beta}_s^\top \right)\dd\left\langle \mathbf{\bar{Z}}\right\rangle_s \right),
\end{align}
for any fixed $\boldsymbol{y}\in\mathbb{Y}$, where $\pi_0(\dd\boldsymbol{y}) = \mathbb{P}(\boldsymbol{Y} \in \dd\boldsymbol{y} \mid \boldsymbol{\xi}^{\boldsymbol{x}}_0 = \boldsymbol{x})$, $\langle \mathbf{\bar{Z}}\rangle_t = \operatorname{diag}(\left\langle \mathbf{Z}\right\rangle_t)$ and $\boldsymbol{\beta}_t = \left( \boldsymbol{y} - \E\left[ \boldsymbol{Y} \, | \, \xib^{\boldsymbol{x}}_t \right] \right)$ for $t\in[0,T)$. Since the ratio $\{R_t(\dd\boldsymbol{y})\}_{t\in[0,T)} = \{\pi_t(\dd\boldsymbol{y})/\pi_0(\dd\boldsymbol{y})\}_{t\in[0,T)}$ is a strictly positive martingale with initial value $R_0(\dd\boldsymbol{y})= 1$, it is a Radon-Nikodym derivative. Hence, dynamic measure-changes, in the form of evolving Radon-Nikodym derivatives, are fundamentally connected to stochastic transports in generative tasks.

\begin{rem}
\label{rem:boundary_convergence}
Although Proposition \ref{stochfilt} establishes the non-anticipative SDE representation strictly on the interval $[0,T)$, the terminal boundary condition $\xib^{\boldsymbol{x}}_T = \boldsymbol{Y}$ is attained in law as $t \to T$. Specifically, taking the limit of (\ref{canonicalGRBcoordinate}), we see that $\boldsymbol{\Sigma}^{*}_{t,T}[i,i] \to 1$ and $\left(Z^{(i)}_t - \boldsymbol{\Sigma}^{*}_{t,T}[i,i]Z^{(i)}_T\right) \to 0$ almost surely as $t \to T$.
\end{rem}

Since $\boldsymbol{Y}$ is square-integrable, the conditional expectation $\E\left[ \boldsymbol{Y} \, | \, \xib^{\boldsymbol{x}}_t \right]$ in (\ref{nonanticipativerepGMRB}) is the $\mathcal{L}^2$ best-estimate of $\boldsymbol{Y}$ given information $\xib^{\boldsymbol{x}}_t$ for $t\in\T$. Geometrically, this corresponds to an orthogonal projection in the Hilbert space $\mathcal{L}^2(\Omega,\F,\P)$ -- this property lends itself naturally to our training and simulation algorithms, which we detail in the subsequent section. We can explicitly describe the internal temporal dynamics of this $\mathcal{L}^2$ optimal estimator as given below.

\begin{prop}
\label{expectationprocessdynamics}
Keep the setup in Proposition \ref{stochfilt} and define $\{\boldsymbol{Y}^{\boldsymbol{x}}_t\}_{t\in\T}$ as
\begin{align}
\boldsymbol{Y}^{\boldsymbol{x}}_t = \E[\boldsymbol{Y}\,|\, \F^{\xib^{\boldsymbol{x}}}_t] = \E\left[ \boldsymbol{Y} \, | \, \xib^{\boldsymbol{x}}_t \right]. \label{conditionalexpectatonfilterdef}
\end{align}
Then, $\{\boldsymbol{Y}^{\boldsymbol{x}}_t\}_{t\in\T}$ admits the following representation:
\begin{align}
%\label{nonanticipativeconditionalexpectation}
\boldsymbol{Y}^{\boldsymbol{x}}_t \law \E\left[ \boldsymbol{Y} \, | \, \xib^{\boldsymbol{x}}_0 \right] +  \int_0^t \emph{Var}\left[ \boldsymbol{Y} \, | \, \xib^{\boldsymbol{x}}_s \right]\boldsymbol{\Lambda}_{s,T}^{-1} \dd \mathbf{Z}_s, \notag
\end{align}
for every $t\in[0,T)$, where $\text{Var}[\cdot\mid\cdot]$ denotes the conditional covariance matrix.
\end{prop}
\begin{proof}
Using the Markovian property (\ref{markovconditionalexpectation}) and the posterior measure SDE (\ref{densitySDE}), we have
\begin{align}
\dd\boldsymbol{Y}^{\boldsymbol{x}}_t &= \int_{\mathbb{Y}} \boldsymbol{y} \, \dd\pi_t(\dd\boldsymbol{y}) = \int_{\mathbb{Y}} \boldsymbol{y} \left[ \left( \boldsymbol{y} - \boldsymbol{Y}^{\boldsymbol{x}}_t \right)^\top \boldsymbol{\Lambda}_{t,T}^{-1} \, \dd\mathbf{\hat{Z}}_t \right] \pi_t(\dd\boldsymbol{y}) \notag \\
&= \int_{\mathbb{Y}} \left[ \boldsymbol{y} \left( \boldsymbol{y} - \boldsymbol{Y}^{\boldsymbol{x}}_t \right)^\top \right] \boldsymbol{\Lambda}_{t,T}^{-1} \, \dd\mathbf{\hat{Z}}_t \, \pi_t(\dd\boldsymbol{y}) = \left[ \int_{\mathbb{Y}} \boldsymbol{y} \left( \boldsymbol{y} - \boldsymbol{Y}^{\boldsymbol{x}}_t \right)^\top \pi_t(\dd\boldsymbol{y}) \right] \boldsymbol{\Lambda}_{t,T}^{-1} \dd\mathbf{\hat{Z}}_t. \label{varextraction1}
\end{align}
Note that
\begin{align}
\int_{\mathbb{Y}} \boldsymbol{y} \left( \boldsymbol{y} - \boldsymbol{Y}^{\boldsymbol{x}}_t \right)^\top \pi_t(\dd\boldsymbol{y}) = \int_{\mathbb{Y}} \boldsymbol{y} \, \boldsymbol{y}^\top \pi_t(\dd\boldsymbol{y}) - \boldsymbol{Y}^{\boldsymbol{x}}_t \left(\boldsymbol{Y}^{\boldsymbol{x}}_t\right)^\top = \text{Var}\left[ \boldsymbol{Y} \;\middle\vert{}\; \boldsymbol{\xi}^{\boldsymbol{x}}_t \right], \label{varextraction2}
\end{align}
where $\text{Var}[\boldsymbol{Y} \mid \boldsymbol{\xi}^{\boldsymbol{x}}_t]$ is a symmetric $n \times n$ matrix. Substituting (\ref{varextraction2}) into (\ref{varextraction1}) provides
\begin{align}
\dd\boldsymbol{Y}^{\boldsymbol{x}}_t = \text{Var}\left[ \boldsymbol{Y} \;\middle\vert{}\; \boldsymbol{\xi}^{\boldsymbol{x}}_t \right]\boldsymbol{\Lambda}_{t,T}^{-1} \, \dd\mathbf{\hat{Z}}_t. \notag
\end{align}
From Proposition \ref{stochfilt}, the innovation process satisfies $\{\mathbf{\hat{Z}}_t\}_{0 \le t < T} \overset{\text{law}}{=} \{\mathbf{Z}_t\}_{0 \le t < T}$. Since the coupled SDE system $\{(\boldsymbol{\xi}^{\boldsymbol{x}},\boldsymbol{\boldsymbol{Y}}^{\boldsymbol{x}})\}_{t\in[0,T)}$ satisfies the regularity conditions for weak uniqueness on $[0,T)$, replacing $\mathbf{\hat{Z}}_t$ with  $\mathbf{Z}_t$ preserves the joint law, and the statement follows by integration.
\end{proof}

The martingale dynamics of $\{\boldsymbol{Y}^{\boldsymbol{x}}_t\}_{t\in\T}$ in Proposition \ref{expectationprocessdynamics} provide a geometric insight on generative trajectories: the volatility of $\{\boldsymbol{Y}^{\boldsymbol{x}}_t\}_{t\in\T}$ is governed by the conditional uncertainty process $\{\text{Var}[\boldsymbol{Y} \mid \boldsymbol{\xi}^{\boldsymbol{x}}_t]\}_{t\in\T}$, which we shall show is a supermartingale that trends to zero as $t\rightarrow T$.

\begin{coro}
\label{stochfiltcoro}
Keep the setup in Corollary \ref{exbrowniananticipative}. Then $\{\xib^{\boldsymbol{x}}_t\}_{t\in\T}$ admits the representation
\begin{align}
\xib^{\boldsymbol{x}}_t \law \int_0^t \frac{\E\left[ \boldsymbol{Y} \, | \, \xib^{\boldsymbol{x}}_s \right]  - \xib^{\boldsymbol{x}}_s }{T - s} \dd s + \boldsymbol{\sigma}\mathbf{W}_t, \,\,\,\, \forall t\in[0,T). \label{nonanticipativerepGMRB}
\end{align}
Accordingly,
\begin{align}
%\label{nonanticipativeconditionalexpectationbrownian}
\boldsymbol{Y}^{\boldsymbol{x}}_t \law \E\left[ \boldsymbol{Y} \, | \, \xib^{\boldsymbol{x}}_0 \right] +  \int_0^t \frac{\emph{Var}\left[ \boldsymbol{Y} \, | \, \xib^{\boldsymbol{x}}_s \right]}{T-s} \boldsymbol{\sigma}^{-1}\dd \mathbf{W}_s, \,\,\,\, \forall t\in[0,T). \notag
\end{align}
\end{coro}
\begin{proof}
The scaled Brownian motion $\{\mathbf{Z}_t = \boldsymbol{\sigma}\mathbf{W}_t\}_{t\in\T}$ is a time-continuous Gaussian martingale with a strictly increasing quadratic-variation satisfying $\langle \boldsymbol{Z} \rangle_t = t \boldsymbol{\sigma}^2$ and $\langle \boldsymbol{Z} \rangle_0 = \mathbf{0}$. Hence $\boldsymbol{\Lambda}_{s,T}^{-1} = \frac{1}{T-s}\boldsymbol{\sigma}^{-2}$. The statement follows using Proposition \ref{stochfilt} and Proposition \ref{expectationprocessdynamics}.
\end{proof}

\begin{prop}
\label{propconditionalvarianceprocess}
Let $\{\boldsymbol{V}^{\boldsymbol{x}}_t\}_{t\in[0,T)}$ be defined by 
\begin{align}
\boldsymbol{V}^{\boldsymbol{x}}_t = \emph{Var}\left[ \boldsymbol{Y} \mid \mathcal{F}^{\xib^{\boldsymbol{x}}}_t \right] = \emph{Var}\left[ \boldsymbol{Y} \, | \, \xib^{\boldsymbol{x}}_t \right]. \notag
\end{align}
Then $\{\boldsymbol{V}^{\boldsymbol{x}}_t\}_{t\in[0,T)}$ is a $(\P,\{\mathcal{F}_{t}^{\xib^{\boldsymbol{x}}}\}_{t\in[0,T)})$-supermartingale under the Loewner partial order $\preceq$. 
\end{prop}
\begin{proof}
Fix some $s, t \in [0,T)$ such that $s \le t$, which yields $\mathcal{F}^{\xib^{\boldsymbol{x}}}_s \subset \mathcal{F}^{\xib^{\boldsymbol{x}}}_t$. By applying the law of total conditional variance
\begin{align}
\text{Var}\left[ \boldsymbol{Y} \mid \mathcal{F}^{\xib^{\boldsymbol{x}}}_s \right] = \E\left[ \text{Var}\left[ \boldsymbol{Y} \mid \mathcal{F}^{\xib^{\boldsymbol{x}}}_t \right] \;\middle\vert{}\; \mathcal{F}^{\xib^{\boldsymbol{x}}}_s \right] + \text{Var}\left[ \E\left[ \boldsymbol{Y} \mid \mathcal{F}^{\xib^{\boldsymbol{x}}}_t \right] \;\middle\vert{}\; \mathcal{F}^{\xib^{\boldsymbol{x}}}_s \right], \notag
\end{align}
we can write the conditional variance process $\boldsymbol{V}^{\boldsymbol{x}}_t$ as follows:
\begin{align}
\boldsymbol{V}^{\boldsymbol{x}} _s = \E\left[ \boldsymbol{V}^{\boldsymbol{x}}_t \mid \mathcal{F}^{\xib^{\boldsymbol{x}}}_s \right] + \text{Var}\left[ \boldsymbol{Y}^{\boldsymbol{x}}_t \mid \mathcal{F}^{\xib^{\boldsymbol{x}}}_s \right], \notag
\end{align}
where $\boldsymbol{Y}^{\boldsymbol{x}}_t = \E[ \boldsymbol{Y} \mid \mathcal{F}^{\xib^{\boldsymbol{x}}}_t ]$. Rearranging this yields
\begin{align}
\E\left[ \boldsymbol{V}^{\boldsymbol{x}}_t \mid \mathcal{F}^{\xib^{\boldsymbol{x}}}_s \right] = \boldsymbol{V}^{\boldsymbol{x}}_s -\text{Var}\left[ \boldsymbol{Y}^{\boldsymbol{x}}_t \mid \mathcal{F}^{\xib^{\boldsymbol{x}}}_s \right]. \label{supermartingaleincrement}
\end{align}
Since any covariance matrix is positive semi-definite, we have $ \boldsymbol{0} \preceq \text{Var}[ \boldsymbol{Y}^{\boldsymbol{x}}_t \mid \mathcal{F}^{\xib^{\boldsymbol{x}}}_s ]$ in the sense of the Loewner partial order $\preceq$. It follows directly from (\ref{supermartingaleincrement}) that
\begin{align}
\E\left[ \boldsymbol{V}^{\boldsymbol{x}}_t \mid \mathcal{F}^{\xib^{\boldsymbol{x}}}_s \right] \preceq \boldsymbol{V}^{\boldsymbol{x}}_s, \notag
\end{align}
for all $0 \leq s \leq t < T$. Since $\boldsymbol{Y}$ is a square-integrable random variable and $\boldsymbol{V}^{\boldsymbol{x}}_t$ is integrable on $[0,T)$, the statement follows.
\end{proof}
\begin{rem}
\label{remarkconditionalvarianceprocess}
Note that $\lim_{t \uparrow T} \boldsymbol{V}^{\boldsymbol{x}}_t = \boldsymbol{0}$ pathwise. 
\end{rem}
Proposition \ref{remarkconditionalvarianceprocess} shows that the conditional variance process is on average a non-increasing process, and Remark \ref{remarkconditionalvarianceprocess} shows that it trends to zero at $t=T$. In the spirit of Proposition \ref{propconditionalvarianceprocess} and Remark \ref{remarkconditionalvarianceprocess}, we shall bring forward Shannon entropy as a means to quantify uncertainty in the system, and introduce $\{S_{t}\}_{0\leq t < T}$ by
\begin{align}
S_{t} = -\sum_{\mathbb{Y}} \pi_{t}(\boldsymbol{y})\log\pi_{t}(\boldsymbol{y}) \hspace{0.1in} \text{for $0\leq t < T$,} \label{sumentropy}
\end{align}
given that $\mathbb{Y}$ is a discrete state-space -- one can also consider differential entropy when $\mathbb{Y}$ is a continuous state-space, whereby the summation in (\ref{sumentropy}) would be replaced by an integral. Herein, we shall focus only on the discrete case to present the fundamental characteristic of how uncertainty behaves in the scope of generative tasks. First, we prove the following lemma.
\begin{lem}
\label{lemmaforvar}
Keep the setup in Proposition \ref{stochfilt}. Define the sum
\begin{align}
\lambda_t = \sum_{\mathbb{Y}} \pi_t(\boldsymbol{y}) \left( \boldsymbol{y} - \boldsymbol{Y}^{\boldsymbol{x}}_t \right)^\top \boldsymbol{\Lambda}_{t,T}^{-1} \, \dd\langle \mathbf{\hat{Z}} \rangle_t \, \boldsymbol{\Lambda}_{t,T}^{-1} \left( \boldsymbol{y} - \boldsymbol{Y}^{\boldsymbol{x}}_t \right), \label{scalarforlemma}
\end{align}
for any fixed $t\in[0,T)$. Then, $\lambda_t$ can be written as follows:
\begin{align}
\lambda_t = \emph{Tr}\left( \emph{Var}\left[ \boldsymbol{Y} \, | \, \xib^{\boldsymbol{x}}_t \right] \boldsymbol{\Lambda}_{t,T}^{-1} \, \dd\langle \mathbf{\hat{Z}} \rangle_t \, \boldsymbol{\Lambda}_{t,T}^{-1} \right). \notag
\end{align}
\end{lem}
\begin{proof}
The statement follows using the cyclic property of the trace operator. Specifically, let 
\begin{align}
\boldsymbol{\epsilon}_t(\boldsymbol{y}) = \boldsymbol{y} - \boldsymbol{Y}^{\boldsymbol{x}}_t \,\,\,\, \text{and} \,\,\,\, \boldsymbol{A}_t = \boldsymbol{\Lambda}_{t,T}^{-1} \, \dd\langle \mathbf{\hat{Z}} \rangle_t \, \boldsymbol{\Lambda}_{t,T}^{-1}, \notag
\end{align}
for any fixed $t\in[0,T)$. The sum $\lambda_t =\boldsymbol{\epsilon}_t(\boldsymbol{y})^\top \boldsymbol{A}_t \boldsymbol{\epsilon}_t(\boldsymbol{y}) $ in (\ref{scalarforlemma}) is a scalar, and since a scalar equals its own trace, we have
\begin{align}
\boldsymbol{\epsilon}_t(\boldsymbol{y})^\top \boldsymbol{A}_t \boldsymbol{\epsilon}_t(\boldsymbol{y}) &= \text{Tr}\left( \boldsymbol{\epsilon}_t(\boldsymbol{y})^\top \boldsymbol{A}_t \boldsymbol{\epsilon}_t(\boldsymbol{y}) \right). \notag
\end{align}
Applying the cyclic property $\text{Tr}(\boldsymbol{B}\boldsymbol{C}) = \text{Tr}(\boldsymbol{C}\boldsymbol{B})$, we can write
\begin{align}
\text{Tr}\left( \boldsymbol{\epsilon}_t(\boldsymbol{y})^\top \left( \boldsymbol{A}_t \boldsymbol{\epsilon}_t(\boldsymbol{y}) \right) \right) &= \text{Tr}\left( \boldsymbol{A}_t \boldsymbol{\epsilon}_t(\boldsymbol{y}) \boldsymbol{\epsilon}_t(\boldsymbol{y})^\top \right). \notag
\end{align}
Substituting this into (\ref{scalarforlemma}) and using the linearity of the summation and trace operators, we get
\begin{align}
\sum_{\mathbb{Y}} \pi_t(\boldsymbol{y}) \text{Tr}\left( \boldsymbol{A}_t \boldsymbol{\epsilon}_t(\boldsymbol{y}) \boldsymbol{\epsilon}_t(\boldsymbol{y})^\top \right) &= \text{Tr}\left( \boldsymbol{A}_t \left[ \sum_{\mathbb{Y}} \pi_t(\boldsymbol{y}) \boldsymbol{\epsilon}_t(\boldsymbol{y}) \boldsymbol{\epsilon}_t(\boldsymbol{y})^\top \right] \right) \notag \\
&= \text{Tr}\left( \boldsymbol{A}_t \left[ \sum_{\mathbb{Y}} \pi_t(\boldsymbol{y}) \left( \boldsymbol{y} - \boldsymbol{Y}^{\boldsymbol{x}}_t \right)\left( \boldsymbol{y} - \boldsymbol{Y}^{\boldsymbol{x}}_t \right)^\top \right] \right). \label{midsteplem}
\end{align}
Next, note that we have the following:
\begin{align}
\text{Var}\left[ \boldsymbol{Y} \, | \, \xib^{\boldsymbol{x}}_t \right] = \sum_{\mathbb{Y}} \pi_t(\boldsymbol{y}) \left( \boldsymbol{y} - \boldsymbol{Y}^{\boldsymbol{x}}_t \right)\left( \boldsymbol{y} - \boldsymbol{Y}^{\boldsymbol{x}}_t \right)^\top. \notag 
\end{align}
Thus, using (\ref{midsteplem}), the following holds:
\begin{align}
\sum_{\mathbb{Y}} \pi_t(\boldsymbol{y}) \text{Tr}\left( \boldsymbol{A}_t \boldsymbol{\epsilon}_t(\boldsymbol{y}) \boldsymbol{\epsilon}_t(\boldsymbol{y})^\top \right) &= \text{Tr}\left( \boldsymbol{A}_t \, \text{Var}\left[ \boldsymbol{Y} \, | \, \xib^{\boldsymbol{x}}_t \right] \right). \notag
\end{align}
Finally, applying the cyclic property once more to commute $\boldsymbol{A}_t$ we achieve
\begin{align}
\text{Tr}\left( \boldsymbol{\Lambda}_{t,T}^{-1} \, \dd\langle \mathbf{\hat{Z}} \rangle_t \, \boldsymbol{\Lambda}_{t,T}^{-1} \, \text{Var}\left[ \boldsymbol{Y} \, | \, \xib^{\boldsymbol{x}}_t \right] \right) &= \text{Tr}\left( \text{Var}\left[ \boldsymbol{Y} \, | \, \xib^{\boldsymbol{x}}_t \right] \boldsymbol{\Lambda}_{t,T}^{-1} \, \dd\langle \mathbf{\hat{Z}} \rangle_t \, \boldsymbol{\Lambda}_{t,T}^{-1} \right), \notag
\end{align}
and the statement follows.
\end{proof}

\begin{prop}
\label{shannonentropy}
Keep the setup in Proposition \ref{stochfilt}. Then $\{S_{t}\}_{0\leq t < T}$ is a $(\P,\{\mathcal{F}_{t}^{\xib^{\boldsymbol{x}}}\}_{t\in[0,T)})$-supermartingale, and hence, the following holds for every $0 \leq u < t < T$:
\begin{align}
\E[S_{t} \,|\, \F^{\xib^{\boldsymbol{x}}}_u] \leq S_{u}. \notag
\end{align}
%for every $0 \leq u < t < T$.
\end{prop}
\begin{proof}
Applying the It\^o product rule on (\ref{sumentropy}) yields the following:
\begin{align} \label{eq:difentropysde}
\dd S_{t} = -\sum_{\mathbb{Y}} \log\pi_{t}(\boldsymbol{y}) \, \dd \pi_{t}(\boldsymbol{y}) - \sum_{\mathbb{Y}} \pi_{t}(\boldsymbol{y}) \, \dd \log\pi_{t}(\boldsymbol{y}) - \sum_{\mathbb{Y}} \dd \langle \pi(\boldsymbol{y}), \log\pi(\boldsymbol{y}) \rangle_t.
\end{align}
First, substituting (\ref{densitySDE}) into the first term of (\ref{eq:difentropysde}) provides
\begin{align}
-\sum_{\mathbb{Y}} \log\pi_{t}(\boldsymbol{y}) \, \dd \pi_{t}(\boldsymbol{y}) = -\sum_{\mathbb{Y}} \log\pi_{t}(\boldsymbol{y}) \pi_t(\boldsymbol{y}) \left( \boldsymbol{y} - \boldsymbol{Y}^{\boldsymbol{x}}_t \right)^\top \boldsymbol{\Lambda}_{t,T}^{-1} \, \dd\mathbf{\hat{Z}}_t. \label{firstterm}
\end{align}
In order to evaluate the remaining terms in (\ref{eq:difentropysde}), define $g_{t}(\boldsymbol{y}) = \log\pi_{t}(\boldsymbol{y})$ for fixed $\boldsymbol{y} \in \mathbb{Y}$. By applying the It\^o formula to $g_t(\boldsymbol{y})$, we have
\begin{align}
\dd \log\pi_{t}(\boldsymbol{y}) &= \frac{1}{\pi_{t}(\boldsymbol{y})} \dd \pi_{t}(\boldsymbol{y}) - \frac{1}{2\pi_{t}(\boldsymbol{y})^{2}} \dd \langle \pi(\boldsymbol{y}) \rangle_t. \label{eq:dlogpi}
\end{align}
From (\ref{densitySDE}), the continuous martingale part of $\dd \pi_t(\boldsymbol{y})$ drives the quadratic variation. Specifically, setting $\boldsymbol{\psi}_t(\boldsymbol{y})^\top = \pi_t(\boldsymbol{y}) \left( \boldsymbol{y} - \boldsymbol{Y}^{\boldsymbol{x}}_t \right)^\top \boldsymbol{\Lambda}_{t,T}^{-1}$, we have $\dd\pi_t(\boldsymbol{y}) = \boldsymbol{\psi}_t(\boldsymbol{y})^\top \dd\mathbf{\hat{Z}}_t$. Thus,
\begin{align}
\dd \langle \pi(\boldsymbol{y}) \rangle_t = \pi_t^2(\boldsymbol{y}) \left( \boldsymbol{y} - \boldsymbol{Y}^{\boldsymbol{x}}_t \right)^\top \boldsymbol{\Lambda}_{t,T}^{-1} \, \dd\langle \mathbf{\hat{Z}} \rangle_t \, \boldsymbol{\Lambda}_{t,T}^{-1} \left( \boldsymbol{y} - \boldsymbol{Y}^{\boldsymbol{x}}_t \right). \label{eq:quadvarpi}
\end{align}
Next, substituting (\ref{eq:dlogpi}) into the second term of (\ref{eq:difentropysde}) yields the following:
\begin{align}
-\sum_{\mathbb{Y}}\pi_{t}(\boldsymbol{y}) \, \dd \log\pi_{t}(\boldsymbol{y}) &= -\sum_{\mathbb{Y}} \dd \pi_t(\boldsymbol{y}) + \frac{1}{2}\sum_{\mathbb{Y}} \frac{1}{\pi_t(\boldsymbol{y})} \dd \langle \pi(\boldsymbol{y}) \rangle_t \notag \\
&= -\dd\left(\sum_{\mathbb{Y}} \pi_t(\boldsymbol{y})\right) + \frac{1}{2}\sum_{\mathbb{Y}} \pi_t(\boldsymbol{y}) \left( \boldsymbol{y} - \boldsymbol{Y}^{\boldsymbol{x}}_t \right)^\top \boldsymbol{\Lambda}_{t,T}^{-1} \, \dd\langle \mathbf{\hat{Z}} \rangle_t \, \boldsymbol{\Lambda}_{t,T}^{-1} \left( \boldsymbol{y} - \boldsymbol{Y}^{\boldsymbol{x}}_t \right). \notag
\end{align}
Since $\sum_{\mathbb{Y}} \pi_t(\boldsymbol{y}) = 1$ for every $t\in\T$, we must have $\sum_{\mathbb{Y}} \dd \pi_t(\boldsymbol{y}) = 0$. Hence,
\begin{align}
-\sum_{\mathbb{Y}}\pi_{t}(\boldsymbol{y}) \, \dd \log\pi_{t}(\boldsymbol{y}) = \frac{1}{2}\sum_{\mathbb{Y}} \pi_t(\boldsymbol{y}) \left( \boldsymbol{y} - \boldsymbol{Y}^{\boldsymbol{x}}_t \right)^\top \boldsymbol{\Lambda}_{t,T}^{-1} \, \dd\langle \mathbf{\hat{Z}} \rangle_t \, \boldsymbol{\Lambda}_{t,T}^{-1} \left( \boldsymbol{y} - \boldsymbol{Y}^{\boldsymbol{x}}_t \right). \label{secondterm}
\end{align}
Finally, for the third term in (\ref{eq:difentropysde}), since $\dd \pi_{t}(\boldsymbol{y})\dd \langle \pi(\boldsymbol{y}) \rangle_t = 0$ for any $t\in[0,T)$, we have the following expression:
\begin{align}
-\sum_{\mathbb{Y}} \dd \langle \pi(\boldsymbol{y}), \log\pi(\boldsymbol{y}) \rangle_t &= -\sum_{\mathbb{Y}} \frac{1}{\pi_t(\boldsymbol{y})} \dd \langle \pi(\boldsymbol{y}) \rangle_t \notag \\
&= -\sum_{\mathbb{Y}} \pi_t(\boldsymbol{y}) \left( \boldsymbol{y} - \boldsymbol{Y}^{\boldsymbol{x}}_t \right)^\top \boldsymbol{\Lambda}_{t,T}^{-1} \, \dd\langle \mathbf{\hat{Z}} \rangle_t \, \boldsymbol{\Lambda}_{t,T}^{-1} \left( \boldsymbol{y} - \boldsymbol{Y}^{\boldsymbol{x}}_t \right). \label{thirdterm}
\end{align}
Note that summing (\ref{secondterm}) and (\ref{thirdterm}) provides $\frac{1}{2}\lambda_t - \lambda_t = -\frac{1}{2}\lambda_t$ for $t\in[0,T)$. Therefore, combining (\ref{firstterm}), (\ref{secondterm}) and (\ref{thirdterm}) to reconstruct \eqref{eq:difentropysde}, we reach
\begin{align} \label{sumshannon}
\dd S_{t} = -\frac{1}{2}\sum_{\mathbb{Y}} \pi_t(\boldsymbol{y}) \left( \boldsymbol{y} - \boldsymbol{Y}^{\boldsymbol{x}}_t \right)^\top \boldsymbol{\Lambda}_{t,T}^{-1} \, \dd\langle \mathbf{\hat{Z}} \rangle_t \, \boldsymbol{\Lambda}_{t,T}^{-1} \left( \boldsymbol{y} - \boldsymbol{Y}^{\boldsymbol{x}}_t \right) + \dd M_t,
\end{align}
where the stochastic process $\{M_t\}_{t\in[0,T)}$ in (\ref{sumshannon}) is an $\mathcal{F}^{\boldsymbol{\xi}^{\boldsymbol{x}}}_t$-adapted continuous martingale given by the following representation:
\begin{align}
\dd M_t = -\sum_{\mathbb{Y}} \log\pi_{t}(\boldsymbol{y}) \pi_t(\boldsymbol{y}) \left( \boldsymbol{y} - \boldsymbol{Y}^{\boldsymbol{x}}_t \right)^\top \boldsymbol{\Lambda}_{t,T}^{-1} \, \dd\mathbf{\hat{Z}}_t. \notag  
\end{align}
Using Lemma \ref{lemmaforvar}, the expression in (\ref{sumshannon}) can be written as
\begin{align}
\dd S_t = -\frac{1}{2} \text{Tr}\left( \text{Var}\left[ \boldsymbol{Y} \, | \, \xib^{\boldsymbol{x}}_t \right] \boldsymbol{\Lambda}_{t,T}^{-1} \, \dd\langle \mathbf{\hat{Z}} \rangle_t \, \boldsymbol{\Lambda}_{t,T}^{-1} \right) + \dd M_t. \label{condexpfin}
\end{align}
As established in Proposition \ref{stochfilt}, the quadratic variation process $\{\langle \mathbf{\hat{Z}}\rangle_{t}\}_{t\in[0,T)}$ equals $\{\left\langle \mathbf{Z}\right\rangle_{t}\}_{t\in[0,T)}$, which is strictly increasing, and hence, $\dd\langle \mathbf{\hat{Z}} \rangle_t \succeq \boldsymbol{0}$. In addition, 
\begin{align}
\boldsymbol{\Lambda}_{t,T}^{-1}\succeq \boldsymbol{0} \,\,\,\, \text{and} \,\,\,\,  \text{Var}\left[ \boldsymbol{Y} \, | \, \xib^{\boldsymbol{x}}_t \right] \succeq \boldsymbol{0}. \notag 
\end{align}
Since the product of symmetric positive semidefinite matrices yields a non-negative trace, the drift in (\ref{condexpfin}) is strictly non-positive. Finally, since $\{M_t\}_{t\in[0,T)}$ is an $\mathcal{F}^{\boldsymbol{\xi}^{\boldsymbol{x}}}_t$-adapted martingale, taking conditional expectation of (\ref{condexpfin}) gives $\E[\dd S_{t} \,|\, \F^{\xib^{\boldsymbol{x}}}_t] \leq 0$, and the statement follows.
\end{proof}

Proposition \ref{shannonentropy} formalizes a critical structural characteristic of our framework: the total informational uncertainty $\{S_{t}\}_{0\leq t < T}$ is a \emph{non-increasing} process on average. This property provides foundational intuition for generative tasks. Since $\pi_T(\boldsymbol{y}) = \boldsymbol{\delta}_{\boldsymbol{y}}$, taking the limit pathwise yields $\lim_{t \uparrow T} S_{T} = 0$ using the standard identity $\lim_{\pi\rightarrow 0^+} \pi\log \pi = 0$. Consequently, Shannon entropy vanishes at the terminal boundary, steering the system toward a zero-uncertainty state as $t\rightarrow T$.

\begin{rem}
A defining characteristic of our framework is the structural centrality of stochastic filtering and entropy evolution. Under this view, we interpret stochastic transports as noisy information processes that govern the dynamic translation and propagation of information between arbitrary probability measures.
\end{rem}

\subsection{Gaussian Random-Bridge Transport Numerical Framework}
Section \ref{sec:gausssetup} presents the training and simulation algorithms and
the experimental setup. Section \ref{sec:gaussresults} reports the results on
MNIST and CIFAR-10, including generation quality, coverage, and computational
cost.

\subsubsection{Training, Simulation and Experimental Setup}
\label{sec:gausssetup}
For the remainder of this section we operate within the specific Brownian setup established in Corollary \ref{exbrowniananticipative} with $\{\mathbf{Z}_t = \boldsymbol{\sigma}\mathbf{W}_t\}_{t\in\T}$. We see that the conditional expectation $\E[ \boldsymbol{Y} \, | \, \F^{\xib^{\boldsymbol{x}}}_t ]$ for $t\in[0,T)$ has a fundamental stance in non-anticipative representations, where we have
\begin{align}
\xib^{\boldsymbol{x}}_t \law \frac{t}{T}\boldsymbol{Y}^{\boldsymbol{x}}_t + \boldsymbol{\sigma}\E\left[ \boldsymbol{W}_t - \frac{t}{T}\boldsymbol{W}_T \, \middle|\, \xib^{\boldsymbol{x}}_t \right], \label{alternativenonanticipativebrownian}
\end{align}
using (\ref{alternativenonanticipative}) and the SDE representation using (\ref{nonanticipativerepGMRB}) given by
\begin{align}
\dd\xib^{\boldsymbol{x}}_t \law \frac{\boldsymbol{Y}^{\boldsymbol{x}}_t  - \xib^{\boldsymbol{x}}_t }{T-t}\dd t + \boldsymbol{\sigma}\dd\mathbf{W}_t, \label{rewrittendifform}
\end{align}
for $t\in[0,T)$, where $\{\boldsymbol{Y}^{\boldsymbol{x}}_t\}_{t\in\T}$ is defined as in (\ref{conditionalexpectatonfilterdef}).
\begin{rem}
The SDE system \eqref{rewrittendifform} can be interpreted as a generalized Ornstein-Uhlenbeck process characterized by a time-varying mean-reversion rate given by $\{(T-t)^{-1}\}_{0\leq t < T}$, and a state-dependent reversion level tracked by $\{\boldsymbol{Y}^{\boldsymbol{x}}_t\}_{0\leq t < T}$. 
\end{rem}
Accordingly, the stochastic transport process $\{\boldsymbol{\xi}^{\boldsymbol{x}}_t\}_{t\in\T}$ continuously pulls its sample paths toward the $\mathcal{L}^2$ optimal best-estimate $\mathbb{E}[ \boldsymbol{Y} \mid \mathcal{F}^{\boldsymbol{\xi}^{\boldsymbol{x}}}_t ]$ for $t\in[0,T)$. The conditional expectation is fundamental in stochastic filtering and information processing for the state estimation of a dynamical system. In light of this structural property, we parameterize a neural function approximator to target the $\mathcal{L}^2$ best-estimate:
\begin{align}
f^*\left( \xib^{\boldsymbol{x}}_t, t ; \boldsymbol{\Theta}^* \right) \approx \boldsymbol{Y}^{\boldsymbol{x}}_t = \E\left[ \boldsymbol{Y} \, | \, \xib^{\boldsymbol{x}}_t \right], \label{crucialapprox}
\end{align}
where $\boldsymbol{\Theta}$ represents the learnable parameter weights of the network. First, define a function $H$ as follows:
\begin{align}
H(\boldsymbol{y} , \xib^{\boldsymbol{x}}_t, t ; \boldsymbol{\Theta} ) = \left\| \boldsymbol{y} -  f\left( \boldsymbol{\xi}^{\boldsymbol{x}}_t, t ; \boldsymbol{\Theta} \right) \right\|_2^2,    
\end{align}
as the point-wise error map. The structural optimization objective required to train the network given in (\ref{crucialapprox}) resolves to the following expected loss:
\begin{align}
\boldsymbol{L}^{f}\left(\boldsymbol{x}, \boldsymbol{y}, \xib^{\boldsymbol{x}}_t, t ; \boldsymbol{\Theta}\right) = \mathbb{E}_{\substack{(\boldsymbol{x},\boldsymbol{y}) \sim \boldsymbol{\Gamma}(\boldsymbol{\Phi},\boldsymbol{\Psi}) \\ t \sim \mathcal{U}[0,T) \\ \boldsymbol{\xi}^{\boldsymbol{x}}_t \sim \mathcal{N}(\boldsymbol{E}_t, \boldsymbol{V}_t)}}\left[ H\left(\boldsymbol{y} , \boldsymbol{\xi}^{\boldsymbol{x}}_t, t ; \boldsymbol{\Theta}\right) \right], \label{optimizationlossfunction}
\end{align}
where $\mathcal{U}[0,T)$ denotes the uniform probability distribution over the time horizon $[0,T)$, and $\mathcal{N}(\boldsymbol{E}_t, \boldsymbol{V}_t)$ is the conditional distribution established in \eqref{conditionaldistroneed}. 
\\
\\
\begin{algorithm}[H]
\caption{Gaussian Random-Bridge Transport Training Pipeline}
\label{alg:training}
\Input{Data distributions $\boldsymbol{\Phi}, \boldsymbol{\Psi}$, target time horizon $T$, and initial network parameters $\boldsymbol{\Theta}$}
\Output{Optimised network parameters $\boldsymbol{\Theta}^*$}
\BlankLine
\While{not converged}{
    Sample a boundary data pair $(\boldsymbol{x},\boldsymbol{y}) \leftarrow \boldsymbol{\Gamma}(\boldsymbol{\Phi},\boldsymbol{\Psi})$\;
    Sample a random time point $t \leftarrow \mathcal{U}[0,T)$\;
    \BlankLine
    \If{$t = 0$}{
        Set the bridge state $\boldsymbol{\xi}^{\boldsymbol{x}}_0 = \boldsymbol{x}$\;
    }
    \Else{
        Sample the path state $\boldsymbol{\xi}^{\boldsymbol{x}}_t \leftarrow \mathcal{N}\left(\boldsymbol{E}_t \, , \, \boldsymbol{V}_t \right)$ via \eqref{conditionaldistroneed}\;
    }
    \BlankLine
    Compute the empirical loss $\boldsymbol{L}^{f}\left(\boldsymbol{x}, \boldsymbol{y}, \boldsymbol{\xi}^{\boldsymbol{x}}_t, t ; \boldsymbol{\Theta}\right)$ using \eqref{optimizationlossfunction}\;
    Update network weights via gradient descent step to minimise the loss: \;
    $\boldsymbol{\Theta} \leftarrow \boldsymbol{\Theta} - \eta \nabla_{\boldsymbol{\Theta}} \boldsymbol{L}^{f}\left(\boldsymbol{x}, \boldsymbol{y}, \boldsymbol{\xi}^{\boldsymbol{x}}_t, t ; \boldsymbol{\Theta}\right)$\;
}
\Return{$\boldsymbol{\Theta}^* = \boldsymbol{\Theta}$}\;
\end{algorithm}

\bigskip
\begin{algorithm}[H]
\caption{Gaussian Random-Bridge Transport Simulation Pipeline}
\label{alg:simulation}
\Input{Source distribution $\boldsymbol{\Phi}$, optimised model parameters $\boldsymbol{\Theta}^*$, grid size $m$, path parameter $\boldsymbol{\sigma}$, final time $T$, and boundary offset $\epsilon > 0$}
\Output{Generated path trajectory $\{\boldsymbol{\xi}^{\boldsymbol{x}}_t\}$ and terminal samples approximating $\boldsymbol{\Psi}$}
\BlankLine
Initialise a temporal discretization grid: $0 = t_0 < t_1 < \ldots < t_m = T -\epsilon$\;
Compute step size: $\delta = (T - \epsilon)/m$\;
Sample initial state $\boldsymbol{x} \leftarrow \boldsymbol{\Phi}$ and fix starting anchor: $\boldsymbol{\xi}^{\boldsymbol{x}}_0 = \boldsymbol{x}$\;
Sample a reference driving process trajectory $\{\boldsymbol{Z}_t\}_{t \in \hat{\mathcal{T}}}$ where $\hat{\mathcal{T}} = \{t_0, \ldots, t_m\}$\;
\BlankLine
\For{$r = 0$ \KwTo $m-1$}{
    Set current time step: $t = t_r$\;
    Evaluate neural network estimation: $f^*\left( \boldsymbol{\xi}^{\boldsymbol{x}}_t, t ; \boldsymbol{\Theta}^* \right)$\;
    Compute the driving process increment: $\dd \boldsymbol{Z}_t = \boldsymbol{Z}_{t_{r+1}} - \boldsymbol{Z}_{t}$\;
    \BlankLine
    Generate next path step via non-anticipative SDE mapping:\;
    $\boldsymbol{\xi}^{\boldsymbol{x}}_{t_{r+1}} = \boldsymbol{\xi}^{\boldsymbol{x}}_t + \frac{ f^*\left( \boldsymbol{\xi}^{\boldsymbol{x}}_t, t ; \boldsymbol{\Theta}^* \right)  - \boldsymbol{\xi}^{\boldsymbol{x}}_t }{T-t}\delta + \dd \boldsymbol{Z}_t$\;
}
\Return{Generated trajectories starting at $\boldsymbol{\xi}^{\boldsymbol{x}}_0$ through terminal profile $\boldsymbol{\xi}^{\boldsymbol{x}}_{t_m}$}\;
\end{algorithm}

\bigskip
We highlight that in the training algorithm, the optimal parameter configuration is recovered via the following optimization:
\begin{align}
f^*\left( \xib^{\boldsymbol{x}}_t, t ; \boldsymbol{\Theta}^* \right) = \argmin_{f(.;\boldsymbol{\Theta})} \, \boldsymbol{L}^{f}\left(\boldsymbol{x}, \boldsymbol{y}, \xib^{\boldsymbol{x}}_t, t ; \boldsymbol{\Theta}\right). \notag
\end{align}
which, for practical empirical execution, is equivalent to computing the optimal learnable parameter weights of the network:
\begin{align}
\boldsymbol{\Theta}^* = \arg\min_{\boldsymbol{\Theta}} \, \boldsymbol{L}^{f}\left(\boldsymbol{x}, \boldsymbol{y}, \xib^{\boldsymbol{x}}_t, t ; \boldsymbol{\Theta}\right). \notag
\end{align}
The uniform temporal sampling above is to maximize entropy across training intervals, ensuring unbiased gradient updates. A structurally related optimization framework is observed in \cite{14}. 
For empirical simulation, one can employ an appropriate numerical scheme, such as the Euler-Maruyama discretization. Since the non-anticipative dynamical representations exhibit singularity at $t=T$ due to the exploding mean-reversion rate $\lim_{t\rightarrow T}(T-t))^{-1}$, numerical integrations are halted at $T-\epsilon$ for a sufficiently small $\epsilon > 0$. Even if $\xib^{\boldsymbol{x}}_{T-\epsilon} \approx \xib^{\boldsymbol{x}}_{T}$ is a good approximation, the numerical implementation warrants careful consideration during implementation.
\begin{rem}
One can also directly sample $(\mathbf{Z}_{t_{r+1}} - \mathbf{Z}_{t})$ within the loop (instead of sampling a priori), since $\{\mathbf{Z}_{t}\}_{t\in\T} = \{\boldsymbol{\sigma}\mathbf{W}_{t}\}_{t\in\T}$ has independent increments.
\end{rem}
The simulation algorithm must strictly respect the exact driving process configuration $\{\boldsymbol{Z}_{t}\}_{t\in\T}$ selected during the training phase. For obvious reasons, one cannot alter structural parameters such as the diffusion coefficient $\boldsymbol{\sigma}$ between the optimization and inference stages. Note that the scale $\boldsymbol{\sigma}$ operates as a hyperparameter that can be optimized to balance noise injection and sample quality for a given task, a direction we reserve for future exploration. As an additional structural characteristic, we can exploit the martingale property of the driving process to isolate the conditional expectation path, which yields a deterministic flow mapping given by
\begin{align}
\E\left[\xib^{\boldsymbol{x}}_{t_{r+1}} \, | \, \xib^{\boldsymbol{x}}_t \right]  \approx \xib^{\boldsymbol{x}}_t + \frac{ f^*\left( \xib^{\boldsymbol{x}}_t, t ; \boldsymbol{\Theta}^* \right)  - \xib^{\boldsymbol{x}}_t }{T-t}\delta = \frac{T-t_{r+1}}{T-t}\boldsymbol{\xi}^{\boldsymbol{x}}_t + \frac{\delta}{T-t}f^*\left( \xib^{\boldsymbol{x}}_t, t ; \boldsymbol{\Theta}^* \right), \label{condflowmatch}
\end{align}
for $0 \le t < t_{r+1} \le T$, where $\delta = t_{r+1} - t$, given $\boldsymbol{\xi}^{\boldsymbol{x}}_T = \boldsymbol{Y}$. We observe that the conditional expectation map in (\ref{condflowmatch}) aligns with the velocity fields of Flow Matching frameworks, positioning our random-bridge framework as a stochastic generalization of deterministic generative flows. We believe the formal details of this observation and its implications merit a separate dedicated study.

 % =====================================================================
%  Gaussian Random-Bridge Transport Experimental Setup and Results
%  Input after the training and simulation explanation in main_Final.tex.
%  Continues Training, Simulation and Experimental Setup, then starts Empirical Results.
%  Requires: booktabs, graphicx and float (for tables placed after opening prose)
%  Requires the bibitems listed in bibitems_to_add.tex
% =====================================================================

We validate the training and simulation pipelines of Algorithms \ref{alg:training}
and \ref{alg:simulation} on noise-to-image generation over MNIST \cite{18} and
CIFAR-10 \cite{cifar2009}, and compare against simple denoising diffusion
\cite{12}, improved denoising diffusion \cite{13a}, their deterministic DDIM
samplers \cite{13b}, and conditional Flow Matching \cite{27}. Within each dataset every
model shares one network architecture, one exponential-moving-average decay, one
batch size and one training budget, and every sample is scored through one
Fr\'echet Inception Distance (FID) implementation \cite{fid2017}. Optimisation is
AdamW throughout at one learning rate and one schedule, with a single exception:
the random-bridge runs use zero weight decay while the three baselines use
AdamW's default $10^{-2}$, so the baselines are regularised slightly more
strongly than the bridge. Otherwise the differences between rows of the tables
below are the generative process and its simulation scheme.

\paragraph{Scope of the comparison.} We compare generative processes simulated
directly at a common training budget. Distillation-based approaches
\cite{19,21,22} require an already-trained teacher and a second training stage
and are excluded from this comparison. They reach lower FID in the same
few-evaluation regime. The same compression would apply to the random-bridge
construct.

\paragraph{Configuration.} The network $f(\cdot\,;\boldsymbol{\Theta})$ is a UNet
with $64$ base channels on MNIST and $128$ on CIFAR-10, and self-attention at
the $16\times16$ and $8\times8$
resolutions, carrying $7.28$M parameters on MNIST and $52.54$M on CIFAR-10
(improved denoising diffusion carries $52.55$M, its learned reverse variance
doubling the output channels). Optimisation uses AdamW \cite{adamw2019} at
learning rate $2\times10^{-4}$ with a $500$-step linear warmup followed by a
constant rate, gradient-norm clipping at $1.0$, batch size $128$, $40\,000$
steps, an exponential moving average of the weights at decay $0.9999$, and
\texttt{bfloat16} autocast. Data is placed on a $32\times32$ grid in $[-1,1]$.
The terminal checkpoint is evaluated; no checkpoint is selected by FID. All runs
are executed on one NVIDIA GB10 accelerator with $119$GB of unified memory, on
which CIFAR-10 training proceeds at $2.84$ iterations per second, so that each
model costs $3.9$ hours to train.

\paragraph{Evaluation protocol.} Each table entry generates $50\,000$ samples and
compares their Inception-v3 pool3 statistics against those of the full training
split; FID against the test split is recorded alongside and moves in step with
it. Generated samples are mapped to 8-bit exactly once, by a clamped affine map,
and no per-image rescaling is applied to either real or generated images. The
matrix square root in the Fr\'echet distance is computed by eigendecomposition of
the symmetrised product; the implementation returns exactly $0$ on identical
inputs and $2048.0000$ on a unit analytic shift.

Because FID is biased upward at finite sample count, the achievable floor is
measured rather than assumed, using the comparator a table entry uses: $n$
held-out real images scored against the statistics of the full reference split.
On CIFAR-10 this gives $14.60$ at $n=2000$, $5.85$ at $n=5000$ and $3.15$ at
$n=10\,000$. The test split caps the held-out count at $10\,000$, so the floor at
$n=50\,000$ is extrapolated rather than measured: those three points follow
$20470\,n^{-0.955}$ to within a maximum log residual of $0.028$, placing the
$n=50\,000$ floor of Table \ref{tab:cifar} near $0.67$. Repeated simulation seeds give a standard deviation of $0.01$ to $0.07$
on the MNIST entries. On CIFAR-10 the seed spread is measured over three seeds at
$n=10\,000$, giving $0.12$ to $0.36$ at ten network evaluations and $0.05$ to
$0.09$ at fifty; since the sampling variance of FID decreases with $n$, a
standard deviation measured at $n=10\,000$ bounds the one at $n=50\,000$ from
above, and using it to decide which entries are resolved errs towards reporting
genuine differences as ties rather than the converse.

\paragraph{The regression target is the velocity field.}
Algorithm \ref{alg:training} regresses the $\mathcal{L}^2$ best-estimate
$\boldsymbol{Y}^{\boldsymbol{x}}_t$ under a loss that draws
$t\sim\mathcal{U}[0,T)$. Both Algorithm \ref{alg:simulation} and the
deterministic map \eqref{condflowmatch} update the state using the velocity field
\begin{align}
\boldsymbol{u}(\xib^{\boldsymbol{x}}_t,t)
= \frac{\boldsymbol{Y}^{\boldsymbol{x}}_t - \xib^{\boldsymbol{x}}_t}{T-t},
\qquad t\in[0,T),
\label{eq:velocityrepeat}
\end{align}
also identified in \eqref{continuousvelocity}. An estimation error
\begin{align}
\boldsymbol{e}_t = f(\xib^{\boldsymbol{x}}_t,t;\boldsymbol{\Theta}) -
\boldsymbol{Y}^{\boldsymbol{x}}_t \notag 
\end{align}
therefore enters the simulated trajectory at
magnitude $\|\boldsymbol{e}_t\|/(T-t)$, and a Girsanov comparison of the two
diffusions, which share the coefficient $\boldsymbol{\sigma}$, bounds the
Kullback-Leibler divergence of the generated terminal law from $\boldsymbol{\Psi}$
by $\int_0^T \E\|\boldsymbol{e}_t\|^2 (T-t)^{-2}\dd t$ up to a constant. The
objective of \eqref{optimizationlossfunction} minimises
$\int_0^T\E\|\boldsymbol{e}_t\|^2\dd t$. The two differ by the weight
$(T-t)^{-2}$ alone, and both are minimised at $\boldsymbol{e}_t\equiv\boldsymbol{0}$,
so the uniform objective identifies the correct function and the distinction is
immaterial at infinite capacity. At finite capacity it determines where the
unavoidable residual error is placed: under the uniform weight the terminal
region $t>0.98\,T$ contributes two per cent of the loss, and under $(T-t)^{-2}$
it contributes almost all of it.

Appendix \ref{app:weighting} measures the consequence. Two networks trained
identically, one against $\boldsymbol{Y}^{\boldsymbol{x}}_t$ and one against
$\boldsymbol{u}$, are the same estimator to three decimal places over the first
eighty per cent of the bridge and separate by a factor of ten at $t = 0.995\,T$;
in the units of \eqref{eq:velocityrepeat} the endpoint-trained error there is
$11.15$ against $0.70$. All results below therefore regress $\boldsymbol{u}$
against its conditional target
\begin{align}
\frac{\boldsymbol{y} - \xib^{\boldsymbol{x}}_t}{T-t}
= \frac{\boldsymbol{y}-\boldsymbol{x}}{T}
- \boldsymbol{\sigma}\sqrt{\frac{t}{T(T-t)}}\;\boldsymbol{\varepsilon},
\qquad \boldsymbol{\varepsilon}\sim\mathcal{N}(\boldsymbol{0},\boldsymbol{I}),
\label{eq:velocitytarget}
\end{align}
whose conditional mean given $\xib^{\boldsymbol{x}}_t$ is
$\boldsymbol{u}(\xib^{\boldsymbol{x}}_t,t)$ by the tower property, so the
regression is unbiased for the velocity field. The second term of
\eqref{eq:velocitytarget} has conditional variance growing like
$\boldsymbol{\sigma}^2 t/(T-t)$, whose integral diverges logarithmically;
training accordingly draws $t\sim\mathcal{U}[0,T-\epsilon)$ with
$\epsilon = 10^{-3}T$, mirroring the boundary offset already present in
Algorithm \ref{alg:simulation}. At $\boldsymbol{\sigma}=\boldsymbol{0}$ the
target reduces to $(\boldsymbol{y}-\boldsymbol{x})/T$ and the objective becomes
conditional Flow Matching exactly, consistent with Remark \ref{remflowmatch} we shall provide later in the paper:
the two schemes agree to $7\times10^{-7}$ in the simulated trajectory and to
machine precision in the loss. Flow Matching is therefore not an external
baseline in what follows but the $\boldsymbol{\sigma}\to\boldsymbol{0}$ member of
the same family, evaluated through the same code.

\paragraph{The diffusion scale and the horizon are one parameter.} Writing
$t = \tau T$ with $\tau\in[0,1]$, the conditional law of Corollary
\ref{exbrowniananticipative} reads
\begin{align}
\xib^{\boldsymbol{x}}_{\tau T} \mid \{\boldsymbol{Y}=\boldsymbol{y}\}
\;\sim\; \mathcal{N}\!\left( (1-\tau)\boldsymbol{x} + \tau\boldsymbol{y},\;
\boldsymbol{\sigma}^2 T\,\tau(1-\tau)\,\right), \notag
\end{align}
so training and simulation depend on $\boldsymbol{\sigma}$ and $T$ only through
the single scale $\boldsymbol{\gamma}^2 = \boldsymbol{\sigma}^2 T$. The
configuration $\boldsymbol{\sigma}=\boldsymbol{I}$, $T=0.1$ used below
corresponds to $\boldsymbol{\gamma} = 0.3162\,\boldsymbol{I}$, and the
sensitivity study of Appendix \ref{app:gamma} is conducted over
$\boldsymbol{\gamma}$ rather than over the two parameters separately.

\paragraph{Reporting convention.} Simulation cost is reported as the number of
network evaluations per generated sample, which equals the grid size $m$ of
Algorithm \ref{alg:simulation} for every scheme compared here. This is the
quantity that determines wall-clock time: all methods share one network and
spend one evaluation per step, so measured throughput at matched evaluation
count agrees to within one image per second.

\subsubsection{Empirical Results}
\label{sec:gaussresults}

We compare the random-bridge transport with denoising diffusion and conditional
Flow Matching on CIFAR-10 and MNIST. The results include FID at different
sampling budgets, measures of fidelity and coverage, and generation cost at
matched FID.

\subsubsection*{CIFAR-10}

Table \ref{tab:cifar} reports the CIFAR-10 results. The random-bridge model uses
$\boldsymbol{\sigma}=\boldsymbol{I}$ and $T=0.1$, with Euler--Maruyama sampling
as described in Algorithm \ref{alg:simulation}. For improved denoising diffusion,
we report the ancestral sampler. It achieves lower FID than DDIM from five
network evaluations upward. At one and two evaluations, their FID scores differ
by less than $0.6$. Both samplers remain far from convergence at these budgets.

\begin{table}[H]
\begin{center}
\caption{CIFAR-10. FID against the training split, $50\,000$ generated samples
per entry, extrapolated floor $0.67$ (measured $3.15$ at $n=10\,000$).
Simulation cost is the number of network
evaluations per sample. Bold marks the best entry in each column; a dash marks a
configuration not run.}
\label{tab:cifar}
\vskip 0.1in
\resizebox{\textwidth}{!}{\begin{tabular}{lccccccccc}
\toprule
Model & \multicolumn{9}{c}{Network evaluations per sample} \\
\cmidrule(lr){2-10}
 & 1 & 2 & 5 & 10 & 20 & 50 & 100 & 200 & 1000 \\
\midrule
Random-bridge, Euler--Maruyama & 385.98 & \textbf{168.38} & 43.99 & \textbf{19.21} & 12.26 & 9.83 & 9.17 & 8.88 & 9.23 \\
Flow Matching & 372.21 & 173.63 & \textbf{41.64} & 19.41 & 13.50 & 10.43 & 9.35 & 8.82 & 8.43 \\
Improved denoising, ancestral & \textbf{364.62} & 364.42 & 46.57 & 25.60 & \textbf{11.04} & \textbf{6.04} & \textbf{5.30} & \textbf{5.28} & \textbf{5.60} \\
Improved denoising, DDIM & 364.62 & 363.91 & 70.17 & 33.69 & -- & -- & -- & -- & -- \\
\bottomrule
\end{tabular}
}
\end{center}
\end{table}

The random-bridge transport achieves lower FID than improved denoising diffusion
with two to ten network evaluations. At ten evaluations, the bridge scores
$19.21$ against $25.60$. At two evaluations, the scores are $168.38$ and
$364.42$, respectively. Ancestral sampling has too few steps at this budget to
depart from its prior. Neither method has departed from its prior at a single
evaluation, where improved denoising is marginally better, scoring $364.62$
against $385.98$. Improved denoising takes the lead between ten and twenty
evaluations and remains ahead thereafter. At fifty evaluations, it scores
$6.04$ against $9.83$ for the bridge. At two hundred, the scores are $5.28$ and
$8.88$, respectively.

Against Flow Matching the two constructs track one another closely, as the
identity established above predicts. The bridge leads at two, ten, twenty, fifty
and one hundred evaluations and trails at one, five, two hundred and one
thousand; from ten evaluations upward the largest difference in either
direction is $1.24$ FID, at twenty. At the one- and two-evaluation extremes,
where neither construct has converged, the differences are larger: $13.77$ in
favour of Flow Matching at one and $5.25$ in favour of the bridge at two.
At ten evaluations the gap of $0.20$ lies inside one standard
deviation of the simulation-seed spread and is reported as a tie, whereas at
fifty the gap of $0.60$ is roughly seven standard deviations and is resolved.

The random-bridge FID decreases monotonically to $8.88$ at two hundred
evaluations, then increases by $3.9$ per cent to $9.23$ at one thousand. Improved
denoising diffusion also shows an increase over this range, rising by $5.9$ per
cent from $5.28$ to $5.60$. Flow Matching integrates a deterministic flow, and
its FID decreases throughout, reaching $8.43$ at one thousand evaluations. The
effect is therefore attributable to stochastic simulation at fine discretisation.

\subsubsection*{MNIST}

On MNIST, Table \ref{tab:mnist} shows that the random-bridge transport attains
lower FID than improved denoising diffusion at every evaluation count: $3.72$ against $10.99$ at ten evaluations
and $1.06$ against $1.93$ at one thousand. Taken across both of its samplers the
strongest baseline entry is $1.24$, obtained by improved denoising with DDIM at
one hundred evaluations, against $1.06$ for the bridge.

\begin{table}[H]
\begin{center}
\caption{MNIST. FID against the training split, $50\,000$ generated samples per
entry; the measured floor is $0.98$ at $n=10\,000$ and does not admit a
defensible extrapolation to $50\,000$. Mean $\pm$ standard deviation over
simulation seeds where repeats were run.}
\label{tab:mnist}
\vskip 0.1in
\resizebox{\textwidth}{!}{\begin{tabular}{lccccccccc}
\toprule
Model & \multicolumn{9}{c}{Network evaluations per sample} \\
\cmidrule(lr){2-10}
 & 1 & 2 & 5 & 10 & 20 & 50 & 100 & 200 & 1000 \\
\midrule
Random-bridge, exact transition & 346.04 & 50.31 & 9.62 & 4.08 $\pm$ 0.06 & 2.14 & 1.37 $\pm$ 0.00 & 1.19 & 1.12 $\pm$ 0.00 & -- \\
Random-bridge, Euler--Maruyama & 346.30 & \textbf{49.42} & \textbf{9.46} & \textbf{3.72 $\pm$ 0.02} & \textbf{1.89} & 1.27 & 1.17 & 1.12 & 1.06 \\
Random-bridge, deterministic map & 346.05 & 51.46 & 10.51 & 5.73 $\pm$ 0.08 & 4.19 & 3.42 & 3.19 & 3.05 & -- \\
Flow Matching & 298.68 & 54.09 & 10.34 & 4.13 $\pm$ 0.06 & 2.05 & \textbf{1.23 $\pm$ 0.02} & \textbf{1.07} & \textbf{1.02 $\pm$ 0.01} & \textbf{1.02} \\
Simple denoising, cosine, DDIM & 418.50 & 380.07 & 17.76 & 6.96 & 4.07 & 2.24 & 1.40 & 2.83 & 16.30 \\
Simple denoising, cosine, ancestral & 377.15 & 380.14 & 16.27 & 9.01 & 4.83 & 3.00 & 2.90 & 2.99 & 2.71 \\
Simple denoising, linear, DDIM & 418.36 & 285.10 & 14.05 & 7.22 & 4.48 & 2.79 & 1.67 & 2.74 & 13.26 \\
Simple denoising, linear, ancestral & \textbf{282.19} & 284.25 & 18.14 & 10.85 & 6.87 & 3.94 & 3.19 & 3.07 & 3.05 \\
Improved denoising, ancestral & 389.46 & 393.23 & 20.06 & 10.99 & 7.16 & 3.62 & 2.02 & 1.74 & 1.93 \\
Improved denoising, DDIM & 389.46 & 392.03 & 23.62 & 7.82 & 4.22 & 2.38 & 1.24 & 2.72 & 21.94 \\
\bottomrule
\end{tabular}
}
\end{center}
\end{table}

Against Flow Matching,
the bridge leads from two to twenty evaluations. Its FID is $49.42$ against
$54.09$ at two, $9.46$ against $10.34$ at five, $3.72$ against $4.13$ at ten,
and $1.89$ against $2.05$ at twenty. At a single evaluation, the bridge trails
Flow Matching, scoring $346.30$ against $298.68$. Flow Matching also leads from
fifty evaluations upward. The bridge scores $1.27$ against $1.23$ at fifty
evaluations and $1.17$ against $1.07$ at one hundred. At two hundred evaluations
the scores are $1.12$ against $1.02$, and at one thousand they are $1.06$ against
$1.02$. The ordering with respect to improved denoising differs from CIFAR-10,
where the bridge leads only from two to ten evaluations.

Both DDIM baselines deteriorate past one hundred evaluations, reaching $16.30$
and $13.26$ at one thousand. Clipping the forward variance at $0.999$ leaves the
cosine schedule with a terminal signal-to-noise ratio of
$\bar\alpha_T = 2.4\times10^{-9}$, at which noise prediction is degenerate and
the reconstruction of the clean sample divides by
$\sqrt{\bar\alpha_T} = 4.9\times10^{-5}$ \cite{snr2024}. The deterioration is
confined to deterministic sampling and is reproduced with two independent
timestep grids. The ancestral samplers remain stable over the same evaluation
range. Velocity prediction is the prescribed remedy. We do not apply it in
these runs.

\subsubsection*{Fidelity and coverage}
\label{sec:gausscoverage}

FID compresses sample fidelity and distributional coverage into a single scalar,
so a model producing clean samples of part of $\boldsymbol{\Psi}$ and one
producing diffuse samples of all of it may score alike. We separate the two using precision and recall
\cite{prec2019} together with the outlier-robust density and coverage
\cite{dens2020}, computed from the same Inception activations, and report the
Kernel Inception Distance \cite{binkowski2018}, which is unbiased and therefore not
subject to the finite-sample bias that sets the FID floor.

Three findings follow. First, coverage increases with the number of network
evaluations, while precision remains approximately constant. For the
random-bridge transport, precision is $0.606$, $0.614$ and $0.607$ at ten, fifty
and two hundred evaluations, respectively; coverage increases from $0.571$ to
$0.638$ to $0.680$. Flow Matching and improved denoising diffusion exhibit the
same pattern. Second,
where the bridge leads at low evaluation count it leads on recall: at ten
evaluations it is level with improved denoising on precision ($0.606$ against
$0.617$) and on coverage ($0.571$ against $0.578$), and the FID advantage is
carried by recall ($0.530$ against $0.490$). Third, the Kernel Inception
Distance orders every row identically to FID at every evaluation count, so the
ordering is not an artefact of the FID estimator.

Because the Inception network was fitted to natural colour photographs, the
MNIST comparison is repeated in a feature space fitted to the data: a
convolutional classifier trained on the MNIST training split to $97.8$ per cent
test accuracy. The ordering is unchanged, and the classifier's predicted-label
histogram shows that no model omits a digit, every model serving its least
frequent digit at $92$ per cent or more of the real rate at one hundred
evaluations %(Appendix \ref{app:coverage}).

\subsubsection*{Simulation cost at matched quality}
\label{sec:gausscost}

\begin{table}[htbp]
\begin{center}
\caption{Network evaluations per generated sample required to reach a given FID
on CIFAR-10, by logarithmic interpolation between the measured entries of Table
\ref{tab:cifar}.}
\label{tab:cost}
\vskip 0.1in
\begin{tabular}{lccc}
\toprule
Target FID & Random-bridge & Flow Matching & Improved denoising \\
\midrule
$30$ & $7$   & $7$   & $8$  \\
$20$ & $10$  & $10$  & $12$ \\
$15$ & $15$  & $16$  & $16$ \\
$12$ & $22$  & $30$  & $19$ \\
$10$ & $47$  & $65$  & $23$ \\
$9$  & $151$ & $157$ & $27$ \\
\bottomrule
\end{tabular}
\end{center}
\end{table}

Since every scheme compared here evaluates the same network once per step,
throughput at matched evaluation count is equal to within measurement noise:
$176$, $177$ and $177$ images per second at ten evaluations for the
random-bridge, Flow Matching and improved denoising respectively. Peak allocated
memory likewise agrees, at $8645$, $8647$ and $8648$ MB for one training step at
batch $128$ and $3521$, $3510$ and $3545$ MB for one simulation batch of $500$ at
fifty evaluations. Evaluation count therefore converts proportionally into
wall-clock time, and the discriminating comparison is cost at equal quality.

We report this in Table \ref{tab:cost}. The random-bridge construct is the least
expensive route to FID $15$ and above. At FID $10$ and $9$, it requires fewer
network evaluations than Flow Matching and more than improved denoising. At
FID $20$ the advantage over improved denoising is a factor of $1.2$ in network
evaluations.

\subsubsection*{Summary of empirical findings}
\label{sec:gaussclaims}

\begin{enumerate}
\item On CIFAR-10, at matched architecture, training budget and evaluation code,
the random-bridge transport attains lower FID than improved denoising diffusion
from two to ten network evaluations ($19.21$ against $25.60$ at ten) and higher
FID from twenty upward ($9.83$ against $6.04$ at fifty). On MNIST it attains lower FID
at every evaluation count tested ($1.06$ against $1.93$ at one thousand).
\item Against Flow Matching the two are close throughout. On MNIST the bridge
leads from two to twenty evaluations and trails at one and from fifty upward. On CIFAR-10
it leads at two and from ten to one hundred, and trails at one, five, two
hundred and one thousand; from ten evaluations upward the largest difference in
either direction is $1.24$ FID.
\item Quality improves monotonically through two hundred evaluations and rises
by $3.9$ per cent at one thousand, a terminal rise improved denoising diffusion
also exhibits at $5.9$ per cent and deterministic Flow Matching does not. This can be attributed to stochastic simulation at fine discretisation.
\item Additional evaluations increase coverage while precision remains
approximately constant. Higher recall accounts for the advantage at low
evaluation counts.
\item Cost at equal quality favours the random-bridge construct by a factor of
$1.2$ in network evaluations at FID $20$ and favours improved denoising by a
factor of $2$ at FID $10$.
\end{enumerate}

These statements concern generative processes simulated directly, at one shared
training budget, and support a favourable trade-off between simulation cost and
generation quality in the few-evaluation regime. They do not support a claim of
best absolute FID, nor of an order-of-magnitude efficiency advantage.

\section{\levy Random-Bridge Transport}
Using the spatial flexibility of Definition \ref{definitionRandomBridge}, we can work with any underlying \cadlag process $\{\boldsymbol{Z}_t\}_{t\in\T}$ to construct the corresponding bridge channels $\{\boldsymbol{\xi}_t\}_{t\in\T}$. Accordingly, we shall travel beyond the Gaussian framework and choose reference dynamics from alternative families of stochastic processes capable of naturally hosting jumps, fat tails, and heavy skewness. A classic choice is the family of \levy processes, a \cadlag family with independent and stationary increments (see \cite{4a,5a}), as done in \cite{8}. The law of any \levy process $\{\boldsymbol{Z}_t\}_{t\in\T}$ initialized at the origin is uniquely characterized by its characteristic exponent $\gamma(\boldsymbol{\lambda})$ via the L\'evy–Khintchine representation:
\begin{align}
\E[e^{\text{i} \langle\boldsymbol{\lambda} , \boldsymbol{Z}_t\rangle}] = \exp\left( t \gamma(\boldsymbol{\lambda}) \right), \notag
\end{align}
for $\boldsymbol{\lambda}\in\R^n$ and $\text{i} = \sqrt{-1}$, where the characteristic exponent is defined as follows:
\begin{align}
\label{levykhinctine}
\gamma(\boldsymbol{\lambda}) &= \text{i}\langle\boldsymbol{\alpha} , \boldsymbol{\lambda}\rangle -\frac{1}{2}\langle \boldsymbol{\lambda} , \boldsymbol{\beta} \boldsymbol{\beta}^{\tp} \boldsymbol{\lambda} \rangle + \int_{\R^n}\left(\exp(\text{i} \langle \boldsymbol{\lambda} , \boldsymbol{y} \rangle)- 1 - \text{i}\langle \boldsymbol{\lambda} , \boldsymbol{y} \rangle\1(||\boldsymbol{y}|| \leq 1)\right)\boldsymbol{\eta}(\dd \boldsymbol{y}).
\end{align}
Here, $(\boldsymbol{\alpha}, \boldsymbol{\beta}, \boldsymbol{\eta})$ is called the L\'evy-Khintchine triplet with $\boldsymbol{\alpha}\in\R^n$ that tracks the deterministic drift vector, $\boldsymbol{\beta} \in \R^{n\times n}$ that represents the continuous diffusion matrix, and $\boldsymbol{\eta}$ is the \levy measure. The \levy measure is a $\sigma$-finite measure on $(\R^{n}, \mathcal{B}(\R^n))$ satisfying
\begin{align}
\boldsymbol{\eta}(\boldsymbol{0}) = 0 \hspace{0.15in} \text{and} \hspace{0.15in} \int_{\R^n} \min\left(1, ||\boldsymbol{y}||^2 \right)\boldsymbol{\eta}(\dd \boldsymbol{y} ) < \infty. \notag
\end{align}
\begin{rem}
The only continuous \levy process is Brownian motion with drift $\boldsymbol{\alpha}$ and diffusion-coefficient $\boldsymbol{\beta}$ -- it is also the only Gaussian. Any other choice introduces discontinuous dynamics, including pure-jump or jump-diffusion processes including Poisson, Gamma, Cauchy, or Stable subordinating processes.
\end{rem}
For notational convenience, we shall use $f_t(\cdot)$ generically to denote both probability density functions and probability mass functions for $t\in(0,T]$ -- i.e. we write
\begin{align}
\P(\mathbf{Z}_t \in \dd \boldsymbol{z}) = f_t(\boldsymbol{z})\dd \boldsymbol{z} \notag 
\end{align}
for any $\mathbf{Z}_t$ that admits a Lebesgue density and $\P(\mathbf{Z}_t = \boldsymbol{z}) = f_t(\boldsymbol{z})$ for purely atomic constructs. To avoid repetitive expositions, we present the equations assuming the existence of densities; for discrete settings, spatial integrals over $\R^n$ translate to summations, and differential elements are omitted. The marginal transition probabilities of $\{\mathbf{Z}_t\}_{t\in\T}$ satisfy the Chapman-Kolmogorov convolution property given by
\begin{align}
f_{t}(\boldsymbol{z})=(f_s \circledast f_{t-s})(\boldsymbol{z}) = \int_{\mathbb{R}^n}f_{s}(\boldsymbol{r})f_{t-s}(\boldsymbol{z}-\boldsymbol{r})\dd \boldsymbol{r}, \notag
\end{align}
for $0<s<t\leq T$, $\boldsymbol{z}\in\R^n$, and the finite-dimensional distribution of $\{\mathbf{Z}_t\}_{t\in\T}$ satisfies the following:
\begin{align}
\label{levyfinitedimensionaldist}
\mathbb{P}(\mathbf{Z}_{t_{1}}\in \dd \boldsymbol{z}_{1},\ldots, \mathbf{Z}_{t_{m}}\in \dd \boldsymbol{z}_{m}) =\prod^{m}_{i=1}f_{t_{i}-t_{i-1}}(\boldsymbol{z}_{i}-\boldsymbol{z}_{i-1})\dd \boldsymbol{z}_{i},
\end{align}
for $m\in\mathbb{N}_{+}$, $0<t_{1}<\ldots <t_{m} < T$ and $(\boldsymbol{z}_{1},\ldots,\boldsymbol{z}_m)\in\mathbb{R}^{n\times m}$, where $t_0 = 0$ and $\boldsymbol{z}_{0} = \boldsymbol{0}$. As it can be seen from (\ref{levykhinctine}) and (\ref{levyfinitedimensionaldist}), \levy processes are anchored at $\boldsymbol{Z}_0 = \boldsymbol{0}$. We relax this restriction by setting $\mathbf{Z}_0 = \boldsymbol{x}$ to accommodate random initialisations $\boldsymbol{x} \leftarrow \boldsymbol{\Phi}$, defining $\boldsymbol{\Phi}$-initialized L\'evy random-bridge to $\boldsymbol{\Psi}$ via (\ref{sampledinit}). 
\begin{prop}
If $\{\xib^{\boldsymbol{x}}_t\}_{t\in\T}$ is a $\boldsymbol{\Phi}$-initialized L\'evy random-bridge to $\boldsymbol{\Psi}$, then
\begin{align}
\mathbb{P}(\xib^{\boldsymbol{x}}_{t}\in \dd \boldsymbol{z})= \left(\int_{\R^n}\frac{f_{t}(\boldsymbol{\boldsymbol{z}} - \boldsymbol{\boldsymbol{x}})f_{T-t}(\boldsymbol{y} - \boldsymbol{\boldsymbol{z}})}{f_T(\boldsymbol{y} - \boldsymbol{x})}\boldsymbol{\Psi}(\dd \boldsymbol{y})\right)\dd \boldsymbol{z}, \label{levyrandombridgedistribution}
\end{align}
for any $t\in(0,T)$, given that $0 < f_T(\boldsymbol{y}-\boldsymbol{x}) < \infty$ for $(\boldsymbol{\Phi} \otimes \boldsymbol{\Psi})$-almost-every $(\boldsymbol{x},\boldsymbol{y}) \in \mathbb{X}\times \mathbb{Y}$.
\end{prop}
\begin{proof}
Extending the framework of \cite{8}, and using Definition \ref{definitionRandomBridge} and (\ref{levyfinitedimensionaldist}), we have the following identity:
\begin{align}
\mathbb{P}(\xib^{\boldsymbol{x}}_{t_{1}}\in \dd \boldsymbol{z}_{1},\ldots, \xib^{\boldsymbol{x}}_{t_{m}}\in \dd \boldsymbol{z}_{m}, \, \xib^{\boldsymbol{x}}_T \in \dd \boldsymbol{y}) = \left(\prod^{m}_{i=1}f_{t_{i}-t_{i-1}}(\boldsymbol{z}_{i}-\boldsymbol{z}_{i-1})\dd \boldsymbol{z}_{i}\right)\frac{f_{T-t_m}(\boldsymbol{y} - \boldsymbol{\boldsymbol{z}_{m}})}{f_T(\boldsymbol{y} - \boldsymbol{x})}\boldsymbol{\Psi}(\dd \boldsymbol{y}), \notag
\end{align}
with $\boldsymbol{z}_{0} = \boldsymbol{x}$, given that $0 < f_T(\boldsymbol{y}-\boldsymbol{x}) < \infty$ for $(\boldsymbol{\Phi} \otimes \boldsymbol{\Psi})$-almost-every $(\boldsymbol{x},\boldsymbol{y}) \in \mathbb{X}\times \mathbb{Y}$. The statement follows by integration.
\end{proof}

\begin{coro}
\label{levycorodensity}
If $\{\xib^{\boldsymbol{x}}_t\}_{t\in\T}$ is a $\boldsymbol{\Phi}$-initialized L\'evy random-bridge to $\boldsymbol{\Psi}$ with a specific target sample $\boldsymbol{y} \leftarrow \boldsymbol{\Psi}$, the following holds:
\begin{align}
\mathbb{P}(\xib^{\boldsymbol{x}}_t \, |_{\boldsymbol{Y} = \boldsymbol{y} \, , \, \boldsymbol{y} \leftarrow \boldsymbol{\Psi}} \in \dd \boldsymbol{z})= \frac{f_{t}(\boldsymbol{z} - \boldsymbol{x})f_{T-t}(\boldsymbol{y} - \boldsymbol{z})}{f_T(\boldsymbol{y} - \boldsymbol{x})}\dd \boldsymbol{z}, \label{levycorodensityimplicit}
\end{align}
for any $t\in(0,T)$, given that $0 < f_T(\boldsymbol{y}-\boldsymbol{x}) < \infty$ for $(\boldsymbol{\Phi} \otimes \boldsymbol{\Psi})$-almost-every $(\boldsymbol{x},\boldsymbol{y}) \in \mathbb{X}\times \mathbb{Y}$.
\end{coro}
\begin{proof}
Note that the conditional measure given by $\mathbb{P}(\xib^{\boldsymbol{x}}_t \, |_{\boldsymbol{Y} = \boldsymbol{y} \, , \, \boldsymbol{y} \leftarrow \boldsymbol{\Psi}} \in \dd \boldsymbol{z})$ implies that $\boldsymbol{\Psi}(\dd \boldsymbol{y})$ in  (\ref{levyrandombridgedistribution}) collapses to the Dirac measure at $\boldsymbol{y}\in\R^n$.
\end{proof}
In order to provide explicit formulas, we shall give some examples of Corollary \ref{levycorodensity}. In doing so, we highlight \emph{subordinators}, which are \emph{non-decreasing} \levy processes such that we have the following property:
\begin{align}
\P\left(\mathbf{Z}_t - \mathbf{Z}_s \geq 0 \,|\, \mathbf{Z}_s \right) = 1, \hspace{0.1in} \text{for all $0\leq s < t$.}\notag
\end{align}
Assuming coordinate-wise independence for spatial decoupling, the joint transition density factorizes into a product of one-dimensional marginals:
\begin{align}
&\frac{f_{t}(\boldsymbol{z} - \boldsymbol{x})f_{T-t}(\boldsymbol{y} - \boldsymbol{z})}{f_T(\boldsymbol{y} - \boldsymbol{x})} =\prod_{j=1}^n \frac{f_{t}(z^{(j)} - x^{(j)})f_{T-t}(y^{(j)} - z^{(j)})}{f_T(y^{(j)} - x^{(j)})} \triangleq \prod_{j=1}^n \nu_{0,t,T}(x^{(j)}, z^{(j)},y^{(j)}).  \label{productlevydensitybridge}
\end{align}
\begin{ex}
\label{exgamma}
A gamma process $\{Z^{(j)}_t\}_{t\in\T}$ is a subordinator with gamma distributed increments. Let $\{Z^{(j)}_t\}_{t\in\T}$ be such that $\E[Z^{(j)}_t] = t$ and $\E[Z^{(j)}_t Z^{(j)}_t] = t(t + \kappa^{-1})$ for some $\kappa^{-1} > 0$. Then,
\begin{align}
f_{t}(z^{(j)} - x^{(j)}) &= \1(z^{(j)} > x^{(j)} )\frac{\kappa^{\kappa t}}{\Gamma(\kappa t)}(z^{(j)} - x^{(j)})^{\kappa t -1} \exp(-\kappa (z^{(j)} - x^{(j)})) \notag \\
f_{T}(y^{(j)} - x^{(j)}) &= \1(y^{(j)} > x^{(j)} )\frac{\kappa^{\kappa T}}{\Gamma(\kappa T)}(y^{(j)} - x^{(j)})^{\kappa T -1}\exp(-\kappa (y^{(j)} - x^{(j)})) \notag \\
f_{T-t}(y^{(j)} - z^{(j)}) &= \1(y^{(j)} > z^{(j)} )\frac{\kappa^{\kappa (T-t)}}{\Gamma(\kappa (T-t))}(y^{(j)} - z^{(j)})^{\kappa(T-t) -1}\exp(-\kappa (y^{(j)} - z^{(j)})), \notag
\end{align}
where $\Gamma(.)$ is the gamma function:
\begin{align}
\Gamma(x) = \int_0^{\infty} s^{x-1}e^{-s}\dd s. \notag
\end{align}
If $X^{(j)} < Y^{(j)}$, then $\nu_{0,t,T}(x^{(j)}, z^{(j)},y^{(j)})$ in (\ref{productlevydensitybridge}) is well-defined. 
\end{ex}
\begin{ex}
\label{stablelevyex}
A stable-$\lambda$ process $\{Z^{(j)}_t\}_{t\in\mathbb{T}}$ is a subordinator for $\lambda \in (0,1)$ that satisfies the following scaling property:
\begin{align}
Z^{(j)}_t \law c^{-\frac{1}{\lambda}} Z^{(j)}_{c t} \hspace{0.1in} \text{$\forall t\geq 0$}, \notag
\end{align}
for every $c>0$. We shall consider the specific case where $\lambda=1/2$. This allows us to represent the process via the hitting times of a standard Brownian motion:
\begin{align}
Z^{(j)}_t \law \inf\{s : W^{(j)}_s > t\sqrt{2} \}, \notag
\end{align}
for every $t\in\mathbb{T}$, where $\{W^{(j)}_t\}_{t\in\mathbb{T}}$ is a standard Brownian motion and $\sqrt{2}$ represents the activity parameter. Then, we have the following explicit incremental transition density functions:
\begin{align}
f_{t}(z^{(j)} - x^{(j)}) &= \1(z^{(j)} > x^{(j)} ) \frac{t}{\sqrt{\pi}(z^{(j)} - x^{(j)})^{\frac{3}{2}}}\exp\left(-\frac{t^2}{z^{(j)} - x^{(j)}}\right), \notag \\
f_{T}(y^{(j)} - x^{(j)}) &= \1(y^{(j)} > x^{(j)} ) \frac{T}{\sqrt{\pi}(y^{(j)} - x^{(j)})^{\frac{3}{2}}}\exp\left(-\frac{T^2}{y^{(j)} - x^{(j)}}\right), \notag \\
f_{T-t}(y^{(j)} - z^{(j)}) &= \1(y^{(j)} > z^{(j)} ) \frac{T-t}{\sqrt{\pi}(y^{(j)} - z^{(j)})^{\frac{3}{2}}}\exp\left(-\frac{(T-t)^2}{y^{(j)} - z^{(j)}}\right). \notag
\end{align}
If $X^{(j)} < Y^{(j)}$, then $\nu_{0,t,T}(x^{(j)}, z^{(j)},y^{(j)})$ in (\ref{productlevydensitybridge}) is well-defined. 
\end{ex}

We highlight that finite sums of independent \levy processes preserve independent and stationary increments, and thereby, the combined vector process $\{\boldsymbol{Z}^{(*)}_t\}_{t\geq 0}$ defined by:
\begin{align}
\boldsymbol{Z}^{(*)}_t = \sum_{k=1}^K \boldsymbol{Z}^{(k)}_t \quad \forall t\geq 0, \notag
\end{align}
is a \levy process for any finite $K < \infty$, given that the constituent channels $\{\boldsymbol{Z}^{(k)}_t\}_{t\geq 0}$ are mutually independent. Setting $K=2$, the joint probability density function $f_t^{(*)}$ of the combined process is recovered through the spatial convolution:
\begin{align}
f_t^{(*)}(\boldsymbol{z}) &= (f_t^{(1)} \circledast f_t^{(2)})(\boldsymbol{z}) = \int_{\R^n} f_t^{(1)}(\boldsymbol{r})f_t^{(2)}(\boldsymbol{z} - \boldsymbol{r})\,\dd \boldsymbol{r}, \notag
\end{align}
where $f_t^{(1)}$ and $f_t^{(2)}$ are the respective densities of $\boldsymbol{Z}^{(1)}_t$ and $\boldsymbol{Z}^{(2)}_t$ for $t\geq 0$. Consequently, one can construct sophisticated driving processes -- and by extension, highly versatile $\boldsymbol{\Phi}$-initialized L\'evy random-bridges to $\boldsymbol{\Psi}$ -- such as jump-diffusion channels assembled from standard Brownian motions and compensated compound Poisson processes.

\subsection{Connection to Doob \textit{\textbf{h}}-Transforms}
\label{doobh}
We shall provide a brief discussion to highlight the structural connection of our framework and Doob $h$-transforms. As a concrete example, we present this formalism through \levy random-bridges. First, we study the conditional point-to-point bridge process given by
\begin{align}
\xib^{\boldsymbol{x},\boldsymbol{y}}_t \triangleq \xib^{\boldsymbol{x}}_t \, |_{\boldsymbol{Y} = \boldsymbol{y} \, , \, \boldsymbol{y} \leftarrow \boldsymbol{\Psi}},
\end{align}
for every $t\in\T$. The conditional assignment forces the deterministic atomic boundary constraint $\mathbb{P}(\xib^{\boldsymbol{x},\boldsymbol{y}}_T \in \dd \boldsymbol{y} ) = 1$. Extending Corollary \ref{levycorodensity}, we have the following transition density:
\begin{align}
\mathbb{P}( \xib^{\boldsymbol{x},\boldsymbol{y}}_t \in \dd \boldsymbol{z} \,|\, \xib^{\boldsymbol{x},\boldsymbol{y}}_s &= \boldsymbol{r}) = \frac{h_t(\boldsymbol{z};\boldsymbol{y})}{h_s(\boldsymbol{r};\boldsymbol{y})}f_{t-s}(\boldsymbol{z}-\boldsymbol{r})\dd \boldsymbol{z} \notag \\
&\triangleq \nu_{s,t,T}(\boldsymbol{r},\boldsymbol{z},\boldsymbol{y}) \label{eq:ratio}
\end{align}
for $0\leq s <t <T$, where the space-time modifier function is defined by
\begin{align}
h_t(\boldsymbol{z};\boldsymbol{y})=f_{T-t}(\boldsymbol{y}-\boldsymbol{z}) \notag 
\end{align}
for all $t\in[0,T)$. The mapping $(t, \boldsymbol{z}) \mapsto h_t(\boldsymbol{z};\boldsymbol{y})$ is space-time harmonic with respect to the underlying unconditioned driving process $\{\boldsymbol{Z}_{t}\}_{t\in\T}$, and remains strictly well-defined under the absolute continuity support condition $0 < h_t(\cdot) < \infty$ for all $t \in [0,T)$. More precisely, if $\mathcal{A}$ denotes the infinitesimal generator of $\{\boldsymbol{Z}_{t}\}_{t\in\T}$, then using the triplet $(\boldsymbol{\alpha}, \boldsymbol{\beta}, \boldsymbol{\eta})$ from (\ref{levykhinctine}), we have
\begin{align}
\mathcal{A}g(\boldsymbol{z}) = \langle \boldsymbol{\alpha}, \nabla g(\boldsymbol{z}) \rangle + \frac{1}{2} \text{Tr}\left(\boldsymbol{\beta}\boldsymbol{\beta}^{\tp} \nabla^2 g(\boldsymbol{z})\right) + \int_{\R^n} \left( g(\boldsymbol{z}+\boldsymbol{v}) - g(\boldsymbol{z}) - \langle \nabla g(\boldsymbol{z}), \boldsymbol{v} \rangle \1(\|\boldsymbol{v}\| \leq 1) \right) \boldsymbol{\eta}(\dd\boldsymbol{v}), \notag
\end{align}
for smooth $g(\boldsymbol{z})$. Hence, we have the Kolmogorov partial integro-differential equation satisfying: 
\begin{align}
\frac{\partial h_t(\boldsymbol{z};\boldsymbol{y})}{\partial t} + \mathcal{A}h_t(\boldsymbol{z};\boldsymbol{y}) = 0, \notag
\end{align}
on $\mathcal{S}_t(\boldsymbol{y}) \subseteq \R^n$, where $\mathcal{S}_t(\boldsymbol{y})$ is the support. Therefore, the conditioned density transition kernel $\nu_{s,t,T}$ in (\ref{eq:ratio}) is a Doob $h$-transform of the transition density of $\{\boldsymbol{Z}_{t}\}_{t\in\T}$ on $0\leq s < t < T$. We can establish a generalized change-of-measure formulation by defining $\{\vartheta_t(\R^n, \boldsymbol{z})\}_{t \in [0,T)}$ given by 
\begin{align}
\label{doobradonnikodym}
\vartheta_t(\R^n, \boldsymbol{z}) = \int_{\R^n} \frac{f_{T-t}(\boldsymbol{y}-\boldsymbol{z})}{f_T(\boldsymbol{y}-\boldsymbol{x})} \boldsymbol{\Psi}(\dd\boldsymbol{y}) = \int_{\mathbb{R}^n} \left( \prod_{k=1}^n \frac{f_{T-t}^{(k)}(y^{(k)} - z^{(k)})}{f_T^{(k)}(y^{(k)} - x^{(k)})} \right) \boldsymbol{\Psi}(\dd\boldsymbol{y}), 
\end{align}
where the second equality is due to the mutual independence of each coordinate of $\{\boldsymbol{Z}_{t}\}_{t\in\T}$. It follows from \cite{10b} that the process $\{\vartheta_t(\R^n, \boldsymbol{Z}_t)\}_{0 \le t < T}$ forms a non-negative $(\P,\mathcal{F}_t^{\boldsymbol{Z}})$-martingale that serves as a Radon–Nikodym derivative process over support $0 < h_t(\cdot) < \infty$ for all $t \in [0,T)$, inducing a generalized Doob $h$-transform. Note that we have the gradient
\begin{align}
\nabla_{\boldsymbol{z}} \log \vartheta_t(\mathbb{R}^n, \boldsymbol{z}) = \frac{\displaystyle\int_{\mathbb{R}^n} \nabla_{\boldsymbol{z}} \left( \prod_{k=1}^n \frac{f_{T-t}^{(k)}(y^{(k)} - z^{(k)})}{f_T^{(k)}(y^{(k)} - x^{(k)})} \right) \boldsymbol{\Psi}(\dd\boldsymbol{y})}{\displaystyle\int_{\mathbb{R}^n} \left( \prod_{k=1}^n \frac{f_{T-t}^{(k)}(y^{(k)} - z^{(k)})}{f_T^{(k)}(y^{(k)} - x^{(k)})} \right) \boldsymbol{\Psi}(\dd\boldsymbol{y})}. \label{gradientlog}
\end{align}
We shall now use (\ref{gradientlog}) to recover Corollary \ref{stochfiltcoro} purely through a Doob $h$-transform stance. Accordingly, we set $\{\boldsymbol{Z}_{t} = \boldsymbol{\sigma}\mathbf{W}_t\}_{t\in\T}$, under which (\ref{gradientlog}) takes the form
\begin{align}
\left[ \nabla_{\boldsymbol{z}} \log \vartheta_t(\mathbb{R}^n, \boldsymbol{z}) \right]_j = \frac{\displaystyle\int_{\mathbb{R}^n} \left( \frac{y^{(j)} - z^{(j)}}{\sigma_{(j)}^2 (T-t)} \right) \left( \prod_{k=1}^n \frac{f_{T-t}^{(k)}(y^{(k)} - z^{(k)})}{f_T^{(k)}(y^{(k)} - x^{(k)})} \right) \boldsymbol{\Psi}(\dd\boldsymbol{y})}{\displaystyle\int_{\mathbb{R}^n} \left( \prod_{k=1}^n \frac{f_{T-t}^{(k)}(y^{(k)} - z^{(k)})}{f_T^{(k)}(y^{(k)} - x^{(k)})} \right) \boldsymbol{\Psi}(\dd\boldsymbol{y})}, \label{coordinategradlog}
\end{align}
for each coordinate $j \in \{1, \dots, n\}$, where $f_\tau^{(k)}(u)$ is the marginal univariate Gaussian density. Writing $\boldsymbol{\sigma}^{-2} = \text{diag}\left(1/\sigma_{(1)}^2, \dots, 1/\sigma_{(n)}^2\right)$ and rearranging (\ref{coordinategradlog}), we get
\begin{align}
\nabla_{\boldsymbol{z}} \log \vartheta_t(\mathbb{R}^n, \boldsymbol{z}) &= \frac{\boldsymbol{\sigma}^{-2} }{T-t} \left[ \frac{\displaystyle\int_{\mathbb{R}^n} \boldsymbol{y} \left( \prod_{k=1}^n \frac{f_{T-t}^{(k)}(y^{(k)} - z^{(k)})}{f_T^{(k)}(y^{(k)} - x^{(k)})} \right) \boldsymbol{\Psi}(\dd\boldsymbol{y})}{\displaystyle\int_{\mathbb{R}^n} \left( \prod_{k=1}^n \frac{f_{T-t}^{(k)}(y^{(k)} - z^{(k)})}{f_T^{(k)}(y^{(k)} - x^{(k)})} \right) \boldsymbol{\Psi}(\dd\boldsymbol{y})} - \boldsymbol{z} \right] \notag \\
&= \frac{\boldsymbol{\sigma}^{-2} }{T-t} \left[ \displaystyle\int_{\mathbb{R}^n} \boldsymbol{y} \P\left( \boldsymbol{\xi}_T^{\boldsymbol{x}}\in \dd\boldsymbol{y} \, | \, \boldsymbol{\xi}_t^{\boldsymbol{x}} = \boldsymbol{z} \right) - \boldsymbol{z} \right] \notag \\
&= \frac{\boldsymbol{\sigma}^{-2}}{T-t} \left( \mathbb{E}\left[ \left. \boldsymbol{Y} \,\right|\, \boldsymbol{\xi}_t^{\boldsymbol{x}} = \boldsymbol{z} \right] - \boldsymbol{z} \right),
 \label{fullgradlog}
\end{align}
where the second equality is due to the fractional term simplifying into the conditional probability $\P\left( \boldsymbol{\xi}_T^{\boldsymbol{x}}\in \dd\boldsymbol{y} \, | \, \boldsymbol{\xi}_t^{\boldsymbol{x}} = \boldsymbol{z} \right)$ and the final equality follows given that $\boldsymbol{\xi}_T^{\boldsymbol{x}}=\boldsymbol{Y}$. Under the structural rules of the Doob $h$-transform, the drift vector field of an unconditioned SDE is additively shifted by the squared-diffusion matrix multiplied by the spatial gradient of the log-modifier. Under this transformation, Girsanov's theorem dictates that $\mathbf{W}_t$ is mapped to a new standard Brownian motion $\bar{\mathbf{W}}_t$. Hence, substituting (\ref{fullgradlog}) into the transformed drift equation yields
\begin{align}
\dd \xib^{\boldsymbol{x}}_t &= \boldsymbol{\sigma}^2 \nabla_{\boldsymbol{z}} \log \vartheta_t(\mathbb{R}^n, \xib^{\boldsymbol{x}}_t) \dd t + \boldsymbol{\sigma} \dd \bar{\mathbf{W}}_t = \boldsymbol{\sigma}^2 \left[ \frac{\boldsymbol{\sigma}^{-2}}{T-t} \left( \mathbb{E}\left[ \left. \boldsymbol{Y} \,\right|\, \xib^{\boldsymbol{x}}_t \right] - \xib^{\boldsymbol{x}}_t \right) \right] \dd t + \boldsymbol{\sigma} \dd \bar{\mathbf{W}}_t \notag \\
&\law \frac{\mathbb{E}\left[ \left. \boldsymbol{Y} \,\right|\, \xib^{\boldsymbol{x}}_t \right] - \xib^{\boldsymbol{x}}_t}{T-t} \dd t + \boldsymbol{\sigma} \dd \mathbf{W}_t, \notag
\end{align}
since $\boldsymbol{\sigma}^2 \boldsymbol{\sigma}^{-2} = \mathbf{I}$, and the unconditioned SDE $\dd \boldsymbol{Z}_{t} = \boldsymbol{\sigma}\dd \mathbf{W}_t$ has drift $\boldsymbol{0}$ in this example. The pathwise identity and its equivalence in law explicitly establish a formal link between the space-time Doob $h$-transform and the non-anticipative filtering dynamics of Corollary \ref{stochfiltcoro}, providing a unified viewpoint for generative tasks.

\subsection{Non-anticipative Dynamical Representations}
Since any \levy process $\{\mathbf{Z}_t\}_{t\in\T}$ has independent increments, $\{\mathbf{Z}_t\}_{t\in\T}$ is a Markov process -- and hence, $\{\mathbf{\xib}_t\}_{t\in\T}$ is a Markov process from Proposition \ref{markovprop} -- which brings forward the following stochastic differential equation that is analogous to the one in Proposition \ref{stochfilt}.
\begin{prop}
\label{stochfiltlevyrandom}
Let each coordinate of the driving process $\{\mathbf{Z}_t\}_{t\in\T}$ be mutually independent. If $\{\xib^{\boldsymbol{x}}_t\}_{t\in\T}$ is a $\boldsymbol{\Phi}$-initialized L\'evy random-bridge to $\boldsymbol{\Psi}$, where $\E[|\xib^{\boldsymbol{x}}_t|] < \infty$, then $\{\xib^{\boldsymbol{x}}_t\}_{t\in\T}$ admits the following representation:
\begin{align}
\label{nonanticipativereplevyrandombridge}
\xib^{\boldsymbol{x}}_t = \int_0^t \frac{  \E\left[ \boldsymbol{Y} \, | \, \xib^{\boldsymbol{x}}_s \right]  - \xib^{\boldsymbol{x}}_s }{T-s}\dd s + \mathbf{M}_t,
\end{align}
for every $t\in[0,T)$, where $\{\mathbf{M}_t\}_{t\in\T}$ is a $(\P,\{\mathcal{F}^{\xib^{\boldsymbol{x}}}_{t}\}_{t\in\T})$-martingale with $\mathbf{M}_{0} = \boldsymbol{x}$.
\end{prop}

\begin{proof}
Since each coordinate of $\{\boldsymbol{\xi}^{\boldsymbol{x}}_t\}_{t\in\T}$ is a $\boldsymbol{\Phi}$-initialized \levy random-bridge to $\boldsymbol{\Psi}$, it is Markov from Proposition \ref{markovprop}, and its conditional expectation over any sub-interval $s \in [0,t)$ satisfies the following linear interpolation representation:
\begin{align}
\label{conditionalexpsglp}
\mathbb{E}\left[ \boldsymbol{\xi}^{\boldsymbol{x}}_t \;\middle\vert{}\; \boldsymbol{\xi}^{\boldsymbol{x}}_s = \boldsymbol{z} \right] = \frac{T-t}{T-s}\boldsymbol{z} + \frac{t-s}{T-s}\mathbb{E}\left[ \boldsymbol{Y} \;\middle\vert{}\; \boldsymbol{\xi}^{\boldsymbol{x}}_s = \boldsymbol{z} \right],
\end{align}
given that $\boldsymbol{\xi}^{\boldsymbol{x}}_T = \boldsymbol{Y}$, using \cite{8,10b}. We use (\ref{conditionalexpsglp}) to verify that $\{\boldsymbol{M}_t\}_{t\in[0,T)}$ defined in \eqref{nonanticipativereplevyrandombridge} is a valid $\{\mathcal{F}_t^{\boldsymbol{\xi}^{\boldsymbol{x}}}\}$-adapted martingale. Invoking the inherited Markov property from Proposition \ref{markovprop}, the expected difference over the interior window $0 \leq s < t < T$ yields
\begin{align}
\mathbb{E}\left[ \boldsymbol{M}_t - \boldsymbol{M}_s \;\middle\vert{}\; \mathcal{F}_s^{\boldsymbol{\xi}^{\boldsymbol{x}}} \right] &= \mathbb{E}\left[ \boldsymbol{\xi}^{\boldsymbol{x}}_t - \boldsymbol{\xi}^{\boldsymbol{x}}_s \;\middle\vert{}\; \mathcal{F}_s^{\boldsymbol{\xi}^{\boldsymbol{x}}} \right] - \int_s^t \frac{\mathbb{E}[\boldsymbol{Y} \mid \boldsymbol{\xi}^{\boldsymbol{x}}_s] - \mathbb{E}[\boldsymbol{\xi}^{\boldsymbol{x}}_u \mid \boldsymbol{\xi}^{\boldsymbol{x}}_s]}{T-u} \,\dd u \nonumber \\
&= \frac{T-t}{T-s}\boldsymbol{\xi}^{\boldsymbol{x}}_{s} + \frac{t-s}{T-s}\mathbb{E}[\boldsymbol{Y} \mid \boldsymbol{\xi}^{\boldsymbol{x}}_{s}] - \boldsymbol{\xi}^{\boldsymbol{x}}_{s} - \int_s^t \frac{\mathbb{E}[\boldsymbol{Y} \mid \boldsymbol{\xi}^{\boldsymbol{x}}_s]}{T-u} \,\dd u \nonumber \\
&\qquad + \int_s^t \frac{1}{T-u}\left[\frac{T-u}{T-s}\boldsymbol{\xi}^{\boldsymbol{x}}_{s} + \frac{u-s}{T-s}\mathbb{E}[\boldsymbol{Y} \mid \boldsymbol{\xi}^{\boldsymbol{x}}_s]\right] \,\dd u. \label{martingale_integral_expansion}
\end{align}
Defining the function $g(\cdot,\cdot)$ as $g(u,s) = (u-s)/((T-u)(T-s))$ for $s \in [0,t)$, we have the following integration:
\begin{align}
-\int_s^t \left[\frac{\mathbb{E}[\boldsymbol{Y} \mid \boldsymbol{\xi}^{\boldsymbol{x}}_s]}{T-u} - g(u,s)\mathbb{E}[\boldsymbol{Y} \mid \boldsymbol{\xi}^{\boldsymbol{x}}_s]\right]\,\dd u &= \left( \log\left(\frac{T-t}{T-s}\right) - \left[\log\left(\frac{T-t}{T-s}\right) + \frac{t-s}{T-s} \right] \right) \mathbb{E}[\boldsymbol{Y} \mid \boldsymbol{\xi}^{\boldsymbol{x}}_s] \nonumber \\
&= -\frac{t-s}{T-s}\mathbb{E}[\boldsymbol{Y} \mid \boldsymbol{\xi}^{\boldsymbol{x}}_s]. \notag
\end{align}
In addition, we get the following:
\begin{align}
\int_s^t \frac{1}{T-u}\left(\frac{T-u}{T-s}\boldsymbol{\xi}^{\boldsymbol{x}}_{s}\right)\,\dd u = \int_s^t \frac{1}{T-s}\boldsymbol{\xi}^{\boldsymbol{x}}_{s}\,\dd u = \frac{t-s}{T-s}\boldsymbol{\xi}^{\boldsymbol{x}}_{s}. \notag
\end{align}
Substituting the integrated terms into (\ref{martingale_integral_expansion}) provides the cancellation
\begin{align}
\mathbb{E}\left[ \boldsymbol{M}_t - \boldsymbol{M}_s \;\middle\vert{}\; \mathcal{F}_s^{\boldsymbol{\xi}^{\boldsymbol{x}}} \right] &= \left( \frac{T-t}{T-s} - 1 + \frac{t-s}{T-s} \right) \boldsymbol{\xi}^{\boldsymbol{x}}_{s} + \left( \frac{t-s}{T-s} - \frac{t-s}{T-s} \right) \mathbb{E}[\boldsymbol{Y} \mid \boldsymbol{\xi}^{\boldsymbol{x}}_{s}] = \mathbf{0}. \notag
\end{align}
Next, given the condition $\mathbb{E}[\|\xib^{\boldsymbol{x}}_t\|] < \infty$, we formally demonstrate that the drift integral remains bounded on the open interval. Applying the triangle inequality and swapping integration bounds via Fubini's theorem yields:
\begin{align}
\mathbb{E}\left[ \int_0^t \frac{\|\mathbb{E}[\boldsymbol{Y} \mid \boldsymbol{\xi}^{\boldsymbol{x}}_s] - \boldsymbol{\xi}^{\boldsymbol{x}}_s\|}{T-s} \,\dd s \right] &\leq \mathbb{E}\left[ \int_0^t \frac{\|\mathbb{E}[\boldsymbol{Y} \mid \boldsymbol{\xi}^{\boldsymbol{x}}_s]\|}{T-s} \,\dd s \right] + \mathbb{E}\left[ \int_0^t \frac{\|\boldsymbol{\xi}^{\boldsymbol{x}}_s\|}{T-s} \,\dd s \right] \nonumber \\
&= \int_0^t \frac{\mathbb{E}\left[\|\mathbb{E}[\boldsymbol{Y} \mid \boldsymbol{\xi}^{\boldsymbol{x}}_s]\|\right]}{T-s} \,\dd s + \int_0^t \frac{\mathbb{E}[\|\boldsymbol{\xi}^{\boldsymbol{x}}_s\|]}{T-s} \,\dd s < \infty, \notag
\end{align}
which holds strictly for all $0 \leq t < T$ since the closed spatial projection process $\{\mathbb{E}[\boldsymbol{Y} \mid \mathcal{F}^{\boldsymbol{\xi}^{\boldsymbol{x}}}_t]\}_{t\in[0,T)}$ is a uniformly integrable martingale. Hence, $\mathbb{E}[\|\mathbf{M}_t\|] < \infty$. Finally, since the initial drift boundary evaluates to $\mathbf{0}$ at $t=0$, we have $\mathbf{M}_0 = \boldsymbol{\xi}^{\boldsymbol{x}}_0 = \boldsymbol{x}$ almost surely.
\end{proof}

The jump characteristics arising from the driving process $\{\boldsymbol{Z}_t\}_{t\in\T}$ are completely captured within the internal dynamics of the martingale $\{\boldsymbol{M}_t\}_{t\in\T}$. More formally, letting 
\begin{align}
\Delta\boldsymbol{\xi}^{\boldsymbol{x}}_t \triangleq \boldsymbol{\xi}^{\boldsymbol{x}}_t - \boldsymbol{\xi}^{\boldsymbol{x}}_{t-} \notag 
\end{align}
denote the jump operator at an instantaneous time $t$, any sample path discontinuity satisfying $\Delta\boldsymbol{\xi}^{\boldsymbol{x}}_t \neq \mathbf{0}$ maps bijectively to the martingale shocks:
\begin{align}
\Delta\boldsymbol{\xi}^{\boldsymbol{x}}_t \neq \mathbf{0} \iff \Delta\boldsymbol{M}_t \neq \mathbf{0}. \notag
\end{align}

\begin{rem}
\label{rem:convex_combination}
The transition identity in (\ref{conditionalexpsglp}) is a convex combination between the current path location $\xib^{\xb}_s$ and the target terminal conditional expectation $\E\left[ \boldsymbol{Y} \;\middle\vert{}\; \xib^{\xb}_s \right]$. Formally, we have
\begin{align}
w_{s,t} = \frac{T-t}{T-s} \,\,\,\, \text{and} \,\,\,\, 1 - w_{s,t} = 1 - \frac{T-t}{T-s} = \frac{T-s - (T-t)}{T-s} = \frac{t-s}{T-s}, \notag
\end{align}
for $0 \le s < t \le T$. Note that $w_{s,t} \in [0,1]$. Substituting these parameters back into \eqref{conditionalexpsglp} yields
\begin{align}
\E\left[ \xib^{\xb}_t \;\middle\vert{}\; \xib^{\xb}_s \right] = w_{s,t}\xib^{\xb}_s + (1-w_{s,t})\E\left[ \boldsymbol{Y} \;\middle\vert{}\; \xib^{\xb}_s \right]. \label{convex_comb_expanded}
\end{align}
\end{rem}
Geometrically, the convex combination in Remark \ref{rem:convex_combination} formalises a conditional mean flow velocity field, where the current (noisy) state is linearly degraded as information on $\boldsymbol{Y}$ accumulates. As $t \to T$, the anchor weight $w_{s,t} \to 0$ progressively diminishes the memory of the past state $\xib^{\xb}_s$, while the target weight $(1-w_{s,t}) \to 1$ pulls toward the predicted terminal vector $\E[\boldsymbol{Y} \mid \xib^{\xb}_s]$. Accordingly, $\E\left[ \xib^{\xb}_t \;\middle\vert{}\; \xib^{\xb}_s \right]$ is bounded within the convex hull of $\xib^{\xb}_s$ and $\E\left[ \boldsymbol{Y} \;\middle\vert{}\; \xib^{\xb}_s \right]$ for $0 \le s < t \le T$.

\begin{coro}
\label{coro:jensen_stability}
Let $\psi: \R^n \to \R$ be a convex function such that $\E[|\psi(\xib^{\xb}_t)|] < \infty$ for all $t \in \T$ and $\E[|\psi(\boldsymbol{Y})|] < \infty$. Then the following holds:
\begin{align}
\psi\left( \E\left[ \xib^{\xb}_t \;\middle\vert{}\; \xib^{\xb}_s \right] \right) \leq w_{s,t}\psi\left(\xib^{\xb}_s\right) + (1-w_{s,t})\E\left[ \psi(\boldsymbol{Y}) \;\middle\vert{}\; \xib^{\xb}_s \right], \label{jensen_convexity}
\end{align}
almost surely for all $0 \le s < t \le T$. 
\end{coro}
\begin{proof}
Since $w_{s,t} \in [0,1]$ for $0 \le s < t \le T$, applying Jensen's inequality for the convex combination (\ref{convex_comb_expanded}) gives
\begin{align}
\psi\left( \E\left[ \xib^{\xb}_t \;\middle\vert{}\; \xib^{\xb}_s \right] \right) \leq w_{s,t}\psi\left(\xib^{\xb}_s\right) + (1-w_{s,t})\psi\left(\E\left[ \boldsymbol{Y} \;\middle\vert{}\; \xib^{\xb}_s \right]\right), \notag
\end{align}
almost surely. Next, the conditional Jensen's inequality on the last term provides $\psi\left(\E\left[ \boldsymbol{Y} \;\middle\vert{}\; \xib^{\xb}_s \right]\right) \leq \E\left[ \psi(\boldsymbol{Y}) \;\middle\vert{}\; \xib^{\xb}_s \right]$, and the statement follows.
\end{proof}
Corollary \ref{coro:jensen_stability} guarantees that any convex transformation of the conditional mean flow velocity field is bounded above as in (\ref{jensen_convexity}). In generative modelling, this inequality ensures that convex regularizing potentials or error metrics along the mean path remain well-behaved and constrained by the convex mapped endpoints, providing the corresponding structural stability during transport.

\subsection{\levy Random-Bridge Transport Numerical Framework}
While Proposition \ref{stochfiltlevyrandom} proves the abstract existence of the $(\mathbb{P},\{\mathcal{F}_{t}^{\boldsymbol{\xi}^{\boldsymbol{x}}}\}_{t\in\T})$-martingale $\{\boldsymbol{M}_t\}_{t\in\T}$, isolating its explicit, closed-form pathwise structure depends heavily on the specific jump measure choices under the L\'evy–Itô decomposition. Formally studying this martingale merits dedicated theoretical work, which represents a compelling avenue for future research. Nonetheless, our canonical semimartingale decomposition in (\ref{nonanticipativereplevyrandombridge}) provides a stochastic representation that establishes a foundational connection to Flow Matching as highlighted below.

\begin{rem}
\label{remflowmatch}
By leveraging the martingale property of $\{\boldsymbol{M}_t\}_{t\in\T}$, the infinitesimal conditional drift of the random-bridge defines a deterministic velocity field $\boldsymbol{u}(\xib^{\boldsymbol{x}}_t, t)$ for almost every $t \in [0, T)$:
\begin{align}
\lim_{\Delta t \to 0} \frac{\E\left[\xib^{\boldsymbol{x}}_{t + \Delta t} - \xib^{\boldsymbol{x}}_t \, | \, \xib^{\boldsymbol{x}}_t \right]}{\Delta t} = \boldsymbol{u}(\xib^{\boldsymbol{x}}_t, t) = \frac{\E\left[ \boldsymbol{Y} \, | \, \xib^{\boldsymbol{x}}_t \right] - \xib^{\boldsymbol{x}}_t}{T-t}.
\label{continuousvelocity}
\end{align}
The vector field $\boldsymbol{u}(\xib^{\boldsymbol{x}}_t, t)$ in (\ref{continuousvelocity}) corresponds to the time-dependent velocity field derived in Flow Matching frameworks \cite{27}, but while going well beyond the Brownian setup, thereby positioning our random-bridge framework as a unified stochastic generalization of deterministic generative flows.
\end{rem}
In relation to Remark \ref{remflowmatch} and our numerical framework, evaluating a neural approximator for the $\mathcal{L}^2$ best-estimate $f^*\left( \xib^{\boldsymbol{x}}_t, t ; \boldsymbol{\Theta}^* \right) \approx \E\left[ \boldsymbol{Y} \, | \, \xib^{\boldsymbol{x}}_t \right]$ yields a trained vector field representation:
\begin{align}
\boldsymbol{u}^*(\xib^{\boldsymbol{x}}_t, t) = \frac{f^*\left( \xib^{\boldsymbol{x}}_t, t ; \boldsymbol{\Theta}^* \right) - \xib^{\boldsymbol{x}}_t}{T-t}, \quad t \in [0, T).
\label{continuousvelocitylearned}
\end{align}
By setting $t = t_r$ and defining the discrete time increment as $\delta = t_{r+1} - t_r$, this representation induces a deterministic velocity trajectory. More specifically, leveraging (\ref{conditionalexpsglp}) and (\ref{continuousvelocitylearned}), the conditional expectation of the step update can be structurally decoupled as follows:
\begin{align}
\E\left[\xib^{\boldsymbol{x}}_{t_{r+1}} \, | \, \xib^{\boldsymbol{x}}_{t_r} \right] \approx \xib^{\boldsymbol{x}}_{t_r} + \frac{ f^*\left( \xib^{\boldsymbol{x}}_{t_r}, t_r ; \boldsymbol{\Theta}^* \right)  - \xib^{\boldsymbol{x}}_{t_r} }{T-t_r}\delta = \frac{T-t_{r+1}}{T-t_r}\xib^{\boldsymbol{x}}_{t_r} + \frac{\delta}{T-t_r}f^*\left( \xib^{\boldsymbol{x}}_{t_r}, t_r ; \boldsymbol{\Theta}^* \right), 
\label{condflowmatchgenerallevy}
\end{align}
for $0 \le t_r < t_{r+1} \le T$, subject to the terminal boundary condition $\xib^{\boldsymbol{x}}_T = \boldsymbol{Y}$. Before formalising the corresponding deterministic sampling algorithm, we must ensure that this numerical trajectory remains stable near the terminal boundary and converges to the target distribution as the discretization grid becomes infinitely fine. %The following proposition demonstrates that linking the boundary buffer $\epsilon$ directly to the grid resolution dynamics enforces strict convergence onto the target vector $\boldsymbol{Y}$ as $t \to T$.

\begin{prop} 
\label{prop:terminal_convergence}
Keep the setup in Proposition \ref{stochfiltlevyrandom}. Consider a uniform temporal discrete grid 
\begin{align}
0 = t_0 < t_1 < \ldots < t_m = T -\epsilon \notag
\end{align}
with step size $\delta = (T-\epsilon)/m$. If $\epsilon = T/(m+1)$, then
\begin{align} 
\lim_{\epsilon \to 0^+} \mathbb{E}\left[\boldsymbol{\xi}^{\boldsymbol{x}}_{T} \, \middle\vert{} \, \boldsymbol{\xi}^{\boldsymbol{x}}_{T-\epsilon} \right] = \boldsymbol{Y},
\end{align}
in $\mathcal{L}^2(\mathbb{P})$ and almost surely.
\end{prop}

\begin{proof} 
If $\epsilon = T/(m+1)$, we have $\delta = (T-\epsilon)/m = \epsilon$. Hence, adding an additional step of size $\delta$ to the final grid node provides $t_m + \delta = T$. Evaluating the linear interpolation in (\ref{condflowmatchgenerallevy}) yields the following:
\begin{align} 
\mathbb{E}\left[\boldsymbol{\xi}^{\boldsymbol{x}}_{T} \, \middle\vert{} \, \boldsymbol{\xi}^{\boldsymbol{x}}_{T-\epsilon} \right] &= \frac{T-T}{T-(T-\epsilon)}\boldsymbol{\xi}^{\boldsymbol{x}}_{T-\epsilon} + \left(1 - \frac{T-T}{T-(T-\epsilon)}\right)\mathbb{E}\left[ \boldsymbol{Y} \, \middle\vert{} \, \boldsymbol{\xi}^{\boldsymbol{x}}_{T-\epsilon} \right] = \mathbb{E}\left[ \boldsymbol{Y} \, \middle\vert{} \, \boldsymbol{\xi}^{\boldsymbol{x}}_{T-\epsilon} \right]. \notag
\end{align}
Taking the limit $\epsilon \to 0^+$ is equivalent to refining the grid via $m \to \infty$. Since $\{\boldsymbol{\xi}^{\boldsymbol{x}}_t\}_{t\in\mathbb{T}}$ is a Markov process, evaluating the conditional expectation of the terminal variable $\boldsymbol{Y} = \boldsymbol{\xi}^{\boldsymbol{x}}_T$ given the state $\boldsymbol{\xi}^{\boldsymbol{x}}_{T-\epsilon}$ is equivalent to conditioning on the full historical filtration $\mathcal{F}_{T-\epsilon}^{\boldsymbol{\xi}^{\boldsymbol{x}}}$. The $\mathcal{L}^2$ best-estimate process $\{\boldsymbol{L}_t\}_{t \in [0,T)}$ defined by $\boldsymbol{L}_t = \mathbb{E}[ \boldsymbol{Y}  \mid  \mathcal{F}_t^{\boldsymbol{\xi}^{\boldsymbol{x}}} ]$ is a uniformly integrable $(\mathbb{P}, \{\mathcal{F}_t^{\boldsymbol{\xi}^{\boldsymbol{x}}}\})$-martingale. By definition, $\{\boldsymbol{\xi}^{\boldsymbol{x}}_t\}_{t \in [0,T]}$ is a \levy random-bridge terminating at $\boldsymbol{Y}$, and hence, $\lim_{t \to t_0} \mathbb{P}\left( \left\| \boldsymbol{\xi}^{\boldsymbol{x}}_t - \boldsymbol{\xi}^{\boldsymbol{x}}_{t_0} \right\| > \eta \right) = 0$ for any fixed $t_0 \in [0,T]$ and for every $\eta > 0$, rendering $\boldsymbol{Y}$ measurable with respect to the left-limit $\sigma$-algebra $\mathcal{F}_{T-}^{\boldsymbol{\xi}^{\boldsymbol{x}}}$ up to $\P$-null-sets. In addition, since the random variable $\boldsymbol{Y} \in \mathcal{L}^2(\mathbb{P}; \mathbb{Y})$ is square-integrable, $\{\boldsymbol{L}_t\}_{t \in [0,T)}$ is a closed, square-integrable martingale. Using Doob's martingale convergence theorem, $\{\boldsymbol{L}_t\}_{t \in [0,T)}$ is guaranteed to converge to a unique left-limit $\boldsymbol{L}_{T-} = \mathbb{E}[\boldsymbol{Y} \mid \mathcal{F}_{T-}^{\boldsymbol{\xi}^{\boldsymbol{x}}}] = \boldsymbol{Y}$ almost surely and in mean-square, yielding
\begin{align} 
\lim_{\epsilon \to 0^+} \mathbb{E}\left[ \boldsymbol{Y} \, \middle\vert{} \, \boldsymbol{\xi}^{\boldsymbol{x}}_{T-\epsilon} \right] = \boldsymbol{Y} \,\,\,\, \text{and} \,\,\,\, \lim_{\epsilon \to 0^+} \mathbb{E}\left[ \left\| \mathbb{E}\left[ \boldsymbol{Y} \, \middle\vert{} \, \boldsymbol{\xi}^{\boldsymbol{x}}_{T-\epsilon} \right] - \boldsymbol{Y} \right\|^2 \right] = 0.
\end{align} 
respectively, and the statement follows.
\end{proof}

The convergence in Proposition \ref{prop:terminal_convergence} guarantees that the discrete conditional expectation trajectory asymptotically isolates the target profile $\boldsymbol{Y}$ as it uncouples from the internal state memory, defining a valid generative velocity field. Hence, to ensure perfect structural consistency with the theoretical convergence conditions, $\epsilon$ in Algorithm \ref{alg:levy_ode_inference} is matched to the number of discretization intervals $m$ by setting $\epsilon = T/(m+1)$. This configuration guarantees that the forward step size $\delta = (T-\epsilon)/m$ matches the final boundary safety threshold, forcing the final iteration step onto the terminal horizon $T$. Consequently, by executing integration on a temporal grid up to $T-\epsilon$ and terminating with a final step size scaled to $\delta = \epsilon$, the framework behaves as a stable, deterministic ODE sampler that preserves non-Gaussian prior characteristics.

\bigskip

\begin{algorithm}[H]
\caption{\levy Random-Bridge Transport Training Pipeline}
\label{alg:levy_training}
\Input{Data distributions $\Lambdab, \boldsymbol{\Psi}$, horizon $T$, and initial network weights $\boldsymbol{\Theta}$}
\Output{Optimised network parameters $\boldsymbol{\Theta}^*$}
\BlankLine
\While{not converged}{
    Sample a boundary data pair $(\xb,\yb) \leftarrow \boldsymbol{\Gamma}(\Lambdab,\boldsymbol{\Psi})$\;
    Sample a random time point $t \leftarrow \mathcal{U}[0,T)$\;
    \BlankLine
    \If{$t = 0$}{
        Set the bridge state $\xib^{\xb}_0 = \xb$\;
    }
    \Else{
        Sample the path state $\xib^{\xb}_t$ via \eqref{productlevydensitybridge}\;
    }
    \BlankLine
    Compute the empirical loss $\boldsymbol{L}^{f}\left(\xb, \yb, \xib^{\xb}_t, t ; \boldsymbol{\Theta}\right)$\;
    Update network weights via gradient descent step to minimise the loss: \;
    $\boldsymbol{\Theta} \leftarrow \boldsymbol{\Theta} - \eta \nabla_{\boldsymbol{\Theta}} \boldsymbol{L}^{f}\left(\xb, \yb, \xib^{\xb}_t, t ; \boldsymbol{\Theta}\right)$\;
}
\Return $\boldsymbol{\Theta}^* = \boldsymbol{\Theta}$\;
\end{algorithm}

\bigskip

\begin{algorithm}[H]
\caption{\levy Random-Bridge Transport Deterministic ODE Inference}
\label{alg:levy_ode_inference}
\Input{Source distribution $\Lambdab$, model parameters $\boldsymbol{\Theta}^*$, grid size $m$, final time $T$}
\Output{Generated path trajectory $\{\xib^{\xb}_t\}$ and terminal samples approximating $\boldsymbol{\Psi}$}
\BlankLine
Let $\epsilon = T/(m+1)$. Initialise a temporal discretization grid: $0 = t_0 < t_1 < \ldots < t_m = T -\epsilon$\;
Compute step size: $\delta = (T - \epsilon)/m$\;
Sample initial state $\xb \leftarrow \Lambdab$ and fix starting anchor: $\xib^{\xb}_0 = \xb$\;
\BlankLine
\For{$r = 0$ \KwTo $m-1$}{
    Set current time step: $t = t_r$\;
    Evaluate neural network estimation: $\hat{\boldsymbol{Y}}_r = f^*\left( \xib^{\xb}_t, t ; \boldsymbol{\Theta}^* \right)$\;
    \BlankLine
    Compute the conditional expectation step along the velocity field:\;
    $\E\left[\xib^{\xb}_{t_{r+1}} \, \middle\vert{} \, \xib^{\xb}_t \right] = \xib^{\xb}_t + \frac{ \hat{\boldsymbol{Y}}_r  - \xib^{\xb}_t }{T-t}\delta$\;
    Assign the expected position trajectory value to the deterministic state tracking variable:\;
    $\xib^{\xb}_{t_{r+1}} = \E\left[\xib^{\xb}_{t_{r+1}} \, \middle\vert{} \, \xib^{\xb}_t \right]$\;
}
\BlankLine
Evaluate final boundary neural network prediction: $\hat{\boldsymbol{Y}}_m = f^*\left( \xib^{\xb}_{t_m}, t_m ; \boldsymbol{\Theta}^* \right)$\;
\BlankLine
Compute the terminal conditional expectation:\;
$\E\left[\xib^{\xb}_{T} \, \middle\vert{} \, \xib^{\xb}_{t_m} \right] = \frac{T - T}{T - t_m}\xib^{\xb}_{t_m} + \frac{\epsilon}{T - t_m}\hat{\boldsymbol{Y}}_m = \hat{\boldsymbol{Y}}_m$\;
Assign final generated sample mapping to enforce exact target convergence:\;
$\xib^{\xb}_{T} = \E\left[\xib^{\xb}_{T} \, \middle\vert{} \, \xib^{\xb}_{t_m} \right]$\;
\BlankLine
\Return Generated trajectories starting at $\xib^{\xb}_0$ through terminal profile $\xib^{\xb}_{T}$\;
\end{algorithm}

\bigskip

While Proposition \ref{stochfiltlevyrandom} establishes the existence of the martingale $\{\boldsymbol{M}_t\}_{t\in\T}$, its existence alone does not directly lend itself to pathwise simulations. However, such an exercise does not strictly necessitate the knowledge of $\{\boldsymbol{M}_t\}_{t\in\T}$ in closed-form by virtue of Corollary \ref{levycorodensity}. We demonstrate this for the Gamma, Stable $1/2$ and Poisson cases using coordinate sampling.

\subsubsection{Simulating Gamma Random-Bridge Transport}
\label{gammasim}

\begin{prop}
\label{propGammaBridgeBeta}
Let $\{\boldsymbol{Z}_t\}_{t\in\T}$ be a vector of mutually independent Gamma subordinators satisfying the properties in Example \ref{exgamma}. Then, the following representation holds:
\begin{align}
\label{eqExactBetaUpdate}
\boldsymbol{\xi}_{t}^{\boldsymbol{x},\boldsymbol{y}} \law \boldsymbol{\xi}_{s}^{\boldsymbol{x},\boldsymbol{y}} + \boldsymbol{B}_{s,t} \odot \Big(\boldsymbol{y} - \boldsymbol{\xi}_{s}^{\boldsymbol{x},\boldsymbol{y}}\Big),
\end{align}
for any $\boldsymbol{y} \in \mathbb{Y}$ and $0 \leq s < t < T$, where $\odot$ denotes the element-wise Hadamard product and $\boldsymbol{B}_{s,t} = [B_{s,t}^{(1)}, \ldots, B_{s,t}^{(n)}]^{\tp}$ is a random vector with mutually independent coordinates $B_{s,t}^{(j)} \sim \text{Beta}\big(\kappa(t-s),\, \kappa(T-t)\big)$ for each $j \in \{1,\ldots,n\}$, and is mutually independent of $\boldsymbol{\xi}_{s}^{\boldsymbol{x},\boldsymbol{y}}$.
\end{prop}

\begin{proof}
Using the mutual independence of the coordinates of $\{\boldsymbol{Z}_t\}_{t\in\T}$, the joint conditional bridge transition decomposes as a product of its individual coordinate transitions from (\ref{productlevydensitybridge}). Thus, given that $X^{(j)} < Y^{(j)}$ for $j \in \{1,\ldots,n\}$, we have
\begin{align}
\mathbb{P}(\xi_t^{\boldsymbol{x},(j)} \in \dd z^{(j)} \mid \xi_s^{\boldsymbol{x},(j)} = r^{(j)}, \xi_T^{\boldsymbol{x},(j)} = y^{(j)}) = \nu_{s,t,T}(r^{(j)}, z^{(j)}, y^{(j)}) = \frac{f_{T-t}(y^{(j)} - z^{(j)}) \, f_{t-s}(z^{(j)} - r^{(j)})}{f_{T-s}(y^{(j)} - r^{(j)})}, \notag
\end{align}
for which, substituting the explicit density functions provided in Example \ref{exgamma} yields
\begin{align}
\nu_{s,t,T}(r^{(j)}, z^{(j)}, y^{(j)}) = \frac{ \left[ \lambda_{T,t}  (y^{(j)} - z^{(j)})^{\kappa(T-t)-1} e^{-\kappa(y^{(j)}-z^{(j)})} \right] \left[ \lambda_{t,s}(z^{(j)} - r^{(j)})^{\kappa(t-s)-1} e^{-\kappa(z^{(j)}-r^{(j)})} \right] }{ \lambda_{T,s}  (y^{(j)} - r^{(j)})^{\kappa(T-s)-1} e^{-\kappa(y^{(j)}-r^{(j)})} }  \chi_{\mathcal{D}}, \nonumber
\end{align}
with $\chi_{\mathcal{D}} = \1(r^{(j)} < z^{(j)} < y^{(j)})$, where we defined the following:
\begin{align}
\lambda_{T,t} \triangleq \frac{\kappa^{\kappa (T-t)}}{\Gamma(\kappa(T-t))}. \notag
\end{align} 
Note that we have the following reductions:
\begin{align}
\frac{\kappa^{\kappa(T-t)} \kappa^{\kappa(t-s)}}{\kappa^{\kappa(T-s)}} = 1 \,\,\,\, \text{and} \,\,\,\, \exp\Big(-\kappa(y^{(j)} - z^{(j)}) - \kappa(z^{(j)} - r^{(j)})\Big) = \exp\Big(-\kappa(y^{(j)} - r^{(j)})\Big), \notag 
\end{align}
which provide the following simplification:
\begin{align}
\label{eqSimplifiedDensity}
\nu_{s,t,T}(r^{(j)}, z^{(j)}, y^{(j)}) &= \frac{\Gamma(\kappa(T-s))}{\Gamma(\kappa(t-s))\Gamma(\kappa(T-t))} \frac{(z^{(j)} - r^{(j)})^{\kappa(t-s)-1} (y^{(j)} - z^{(j)})^{\kappa(T-t)-1}}{(y^{(j)} - r^{(j)})^{\kappa(T-s)-1}} \chi_{\mathcal{D}} \notag \\
&= \Xi_{s,t,T} \frac{(z^{(j)} - r^{(j)})^{\kappa(t-s)-1} (y^{(j)} - z^{(j)})^{\kappa(T-t)-1}}{(y^{(j)} - r^{(j)})^{\kappa(T-s)-1}} \chi_{\mathcal{D}}
\end{align}
Next, define a normalized scaling variable 
\begin{align}
u \triangleq \frac{z^{(j)} - r^{(j)}}{y^{(j)} - r^{(j)}}, \notag 
\end{align}
which implies $\dd z^{(j)} = (y^{(j)} - r^{(j)}) \dd u$. Under this linear transformation, the support constraint $r^{(j)} < z^{(j)} < y^{(j)}$ maps to $u \in (0, 1)$, and we can write $y^{(j)} - z^{(j)} = (y^{(j)} - r^{(j)})(1 - u)$. Substituting into (\ref{eqSimplifiedDensity}) and accounting for the Jacobian determinant $(y^{(j)} - r^{(j)})$, we obtain
\begin{align}
p_U(u) &= \Xi_{s,t,T} \frac{\big(u(y^{(j)} - r^{(j)})\big)^{\kappa(t-s)-1} \big((1-u)(y^{(j)} - r^{(j)})\big)^{\kappa(T-t)-1}}{(y^{(j)} - r^{(j)})^{\kappa(T-s)-1}} (y^{(j)} - r^{(j)})  \1(0 < u < 1) \nonumber \\
&= \Xi_{s,t,T} u^{\kappa(t-s)-1} (1-u)^{\kappa(T-t)-1} \1(0 < u < 1), \nonumber
\end{align}
which is a standard $\text{Beta}\big(\kappa(t-s),\, \kappa(T-t)\big)$ random variable, mutually independent of $\xi_s^{\boldsymbol{x},(j)}$ for $j \in \{1,\ldots,n\}$. Hence, we have
\begin{align}
\label{eqBetaTransition}
\frac{\xi_t^{\boldsymbol{x},(j)} - \xi_s^{\boldsymbol{x},(j)}}{\xi_T^{\boldsymbol{x},(j)} - \xi_s^{\boldsymbol{x},(j)}} \,\,\bigg|\,\, \left\{\xi_s^{\boldsymbol{x},(j)} = r^{(j)}, \, \xi_T^{\boldsymbol{x},(j)} = y^{(j)}\right\} \,\sim\, \text{Beta}\Big(\kappa(t-s), \, \kappa(T-t)\Big).
\end{align}
Rearranging $u = (z^{(j)} - r^{(j)}) / (y^{(j)} - r^{(j)})$ yields 
\begin{align}
z^{(j)} = r^{(j)} + u(y^{(j)} - r^{(j)}) \notag
\end{align}
for $j \in \{1,\ldots,n\}$. Applying this coordinate mapping component-wise to the process components $\xi_t^{\boldsymbol{x},(j)}$ given $\xi_s^{\boldsymbol{x},(j)}$ and terminal constraints $\xi_T^{\boldsymbol{x},(j)} = y^{(j)}$ gives the relation 
\begin{align}
\xi_t^{\boldsymbol{x},\boldsymbol{y},(j)} \law \xi_s^{\boldsymbol{x},\boldsymbol{y},(j)} + B_{s,t}^{(j)}(y^{(j)} - \xi_s^{\boldsymbol{x},\boldsymbol{y},(j)}), \notag 
\end{align}
and the satement follows.
\end{proof}

For our generative purposes, since the terminal value $\boldsymbol{y}$ is unobservable during inference, we use the $\mathcal{L}^2$ best-estimate given in (\ref{continuousvelocitylearned}), which induces a dynamically targeted filtering framework. Substituting the unknown $\boldsymbol{y}$ in equation (\ref{eqExactBetaUpdate}) with its $\mathcal{L}^2$ best-estimate at time $t_r$ parameterised by the neural network $f^*(\boldsymbol{\xi}_{t_r}^{\boldsymbol{x}}, t_r; \boldsymbol{\Theta}^*) \approx \boldsymbol{Y}_{t_r}^{\boldsymbol{x}} $, we obtain an approximate simulation protocol:
\begin{align}
\boldsymbol{\xi}_{t_{r+1}}^{\boldsymbol{x}} \leftarrow\boldsymbol{\xi}_{t_r}^{\boldsymbol{x}} + \boldsymbol{B}_{t_r,t_{r+1}} \odot \Big(f^*(\boldsymbol{\xi}_{t_r}^{\boldsymbol{x}}, t_r; \boldsymbol{\Theta}^*) - \boldsymbol{\xi}_{t_r}^{\boldsymbol{x}}\Big).
\end{align}
over a discrete time grid $0 = t_0 < t_1 < \dots < t_m = T - \epsilon$. Accordingly, $f^*(\boldsymbol{\xi}_{t_r}^{\boldsymbol{x}}, t_r; \boldsymbol{\Theta}^*)$ acts as a real-time stochastic filter that constantly updates its prediction of the terminal destination based on the path trajectory observed up to time $t_r$; consequently, the approximate simulation scheme behaves as an adaptive steering mechanism. In order safeguard the strict coordinate-wise monotonicity dictated by the subordinator property against potential neural approximation artifacts at inference time, we enforce a localized projection operator via a component-wise maximum filter
\begin{align}
\hat{\boldsymbol{y}} = \max\left(\boldsymbol{\xi}^{\boldsymbol{x}}_s, \, f^*\left( \boldsymbol{\xi}^{\boldsymbol{x}}_s, s ; \boldsymbol{\Theta}^* \right)\right), \notag 
\end{align}
ensuring that simulated trajectories can never break the required mathematical monotonicity.

\bigskip

\begin{algorithm}[H]
\caption{Gamma Random-Bridge Transport Simulation Pipeline}
\label{alg:gammasimulation}
\Input{Source distribution $\boldsymbol{\Phi}$, optimised model parameters $\boldsymbol{\Theta}^*$, grid size $m$, process intensity parameter $\kappa$, final time $T$, and boundary offset $\epsilon > 0$}
\Output{Generated path trajectory $\{\boldsymbol{\xi}^{\boldsymbol{x}}_t\}$ and terminal samples approximating $\boldsymbol{\Psi}$}
\BlankLine
Initialise a temporal discretization grid: $0 = t_0 < t_1 < \ldots < t_m = T -\epsilon$\;
Sample initial state $\boldsymbol{x} \leftarrow \boldsymbol{\Phi}$ and fix starting anchor: $\boldsymbol{\xi}^{\boldsymbol{x}}_0 = \boldsymbol{x}$\;
\BlankLine
\For{$r = 0$ \KwTo $m-1$}{
    Set current time step coordinates: $s = t_r$ and $t = t_{r+1}$\;
    Evaluate raw neural network estimation: $\hat{\boldsymbol{y}}_{\text{raw}} = f^*\left( \boldsymbol{\xi}^{\boldsymbol{x}}_s, s ; \boldsymbol{\Theta}^* \right)$\;
    Enforce monotonic boundary constraints via projection: $\hat{\boldsymbol{y}} = \max\left(\boldsymbol{\xi}^{\boldsymbol{x}}_s, \, \hat{\boldsymbol{y}}_{\text{raw}}\right)$\;
    \BlankLine
    Sample coordinate-wise independent noise components: $B_{s,t}^{(j)} \sim \text{Beta}\big(\kappa(t-s),\, \kappa(T-t)\big)$ for $j \in \{1,\ldots,n\}$\;
    Assemble the stochastic transportation vector: $\boldsymbol{B}_{s,t} = [B_{s,t}^{(1)}, \ldots, B_{s,t}^{(n)}]^{\tp}$\;
    \BlankLine
    Generate next path step via boundary-preserving Beta transportation mapping:\;
    $\boldsymbol{\xi}^{\boldsymbol{x}}_{t} = \boldsymbol{\xi}^{\boldsymbol{x}}_s + \boldsymbol{B}_{s,t} \odot \Big(\hat{\boldsymbol{y}} - \boldsymbol{\xi}^{\boldsymbol{x}}_s\Big)$\;
}
\Return{Generated trajectories starting at $\boldsymbol{\xi}^{\boldsymbol{x}}_0$ through terminal profile $\boldsymbol{\xi}^{\boldsymbol{x}}_{t_m}$}\;
\end{algorithm}

\subsubsection{Simulating Stable-$1/2$ Random-Bridge Transport}
\label{stablesim}
\begin{prop}
\label{propStableBridgeBeta}
Let $\{\boldsymbol{Z}_t\}_{t\in\T}$ be a vector of mutually independent stable subordinators with $\lambda = 1/2$ satisfying the properties in Example \ref{stablelevyex}. Then, the following representation holds:
\begin{align}
\label{eqExactStableUpdate}
\boldsymbol{\xi}_{t}^{\boldsymbol{x},\boldsymbol{y}} \law \boldsymbol{\xi}_{s}^{\boldsymbol{x},\boldsymbol{y}} + \boldsymbol{U}_{s,t} \odot \Big(\boldsymbol{y} - \boldsymbol{\xi}_{s}^{\boldsymbol{x},\boldsymbol{y}}\Big),
\end{align}
for any $\boldsymbol{y} \in \mathbb{Y}$ and $0 \leq s < t < T$, where $\odot$ denotes the element-wise Hadamard product, and $\boldsymbol{U}_{s,t} = [U_{s,t}^{(1)}, \ldots, U_{s,t}^{(n)}]^{\tp}$ is a random vector with mutually independent coordinates whose probability density functions $p_{U^{(j)}}(u)$ are given by
\begin{align}
\label{eqStableBetaDensity}
p_{U^{(j)}}(u) = \frac{t-s}{\sqrt{\pi(y^{(j)} - \xi_s^{\boldsymbol{x},\boldsymbol{y},(j)})}} \frac{T-t}{T-s} u^{-\frac{3}{2}} (1-u)^{-\frac{3}{2}} \varsigma^{(j)}_{s,t,T}(y^{(j)}, \xi_s^{\boldsymbol{x},\boldsymbol{y},(j)}) \1(0 < u < 1).
\end{align}
supported on $u \in (0,1)$, where
\begin{align}
\varsigma^{(j)}_{s,t,T}(y^{(j)}, \xi_s^{\boldsymbol{x},\boldsymbol{y},(j)}) = \exp\left( - \frac{1}{y^{(j)} - \xi_s^{\boldsymbol{x},\boldsymbol{y},(j)}} \left[ \frac{(t-s)^2}{u} + \frac{(T-t)^2}{1-u} - (T-s)^2 \right] \right)
\end{align}
for $0 \leq s < t < T$ and $j \in \{1,\ldots,n\}$.
\end{prop}

\begin{proof}
Using the mutual independence of the coordinates of $\{\boldsymbol{Z}_t\}_{t\in\T}$, the joint conditional bridge transition decomposes into a product of its individual coordinate transitions from (\ref{productlevydensitybridge}). Thus, given $X^{(j)} < Y^{(j)}$ for $j \in \{1,\ldots,n\}$, we get
\begin{align}
\mathbb{P}(\xi_t^{\boldsymbol{x},(j)} \in \dd z^{(j)} \mid \xi_s^{\boldsymbol{x},(j)} = r^{(j)}, \xi_T^{\boldsymbol{x},(j)} = y^{(j)}) = \nu_{s,t,T}(r^{(j)}, z^{(j)}, y^{(j)}) = \frac{f_{T-t}(y^{(j)} - z^{(j)}) \, f_{t-s}(z^{(j)} - r^{(j)})}{f_{T-s}(y^{(j)} - r^{(j)})}, \notag
\end{align}
for spatial coordinates satisfying $r^{(j)} < z^{(j)} < y^{(j)}$, for which, substituting the stable subordinator density functions gives
\begin{align}
\nu_{s,t,T}(r^{(j)}, z^{(j)}, y^{(j)}) = \frac{ \left[ \frac{T-t}{\sqrt{\pi}(y^{(j)} - z^{(j)})^{\frac{3}{2}}} \exp\left(-\frac{(T-t)^2}{y^{(j)} - z^{(j)}}\right) \right] \left[ \frac{t-s}{\sqrt{\pi}(z^{(j)} - r^{(j)})^{\frac{3}{2}}} \exp\left(-\frac{(t-s)^2}{z^{(j)} - r^{(j)}}\right) \right] }{ \frac{T-s}{\sqrt{\pi}(y^{(j)} - r^{(j)})^{\frac{3}{2}}} \exp\left(-\frac{(T-s)^2}{y^{(j)} - r^{(j)}}\right) } \chi_{\mathcal{D}}, \nonumber
\end{align}
with $\chi_{\mathcal{D}} = \1(r^{(j)} < z^{(j)} < y^{(j)})$. Rearranging the terms in $\nu_{s,t,T}(r^{(j)}, z^{(j)}, y^{(j)})$, we obtain
\begin{align}
\frac{1}{\sqrt{\pi}} \frac{(t-s)(T-t)}{T-s} \frac{(y^{(j)} - r^{(j)})^{\frac{3}{2}}}{(z^{(j)} - r^{(j)})^{\frac{3}{2}}(y^{(j)} - z^{(j)})^{\frac{3}{2}}} \exp\left(-\frac{(T-t)^2}{y^{(j)} - z^{(j)}} - \frac{(t-s)^2}{z^{(j)} - r^{(j)}} + \frac{(T-s)^2}{y^{(j)} - r^{(j)}}\right) \chi_{\mathcal{D}}. \nonumber
\end{align}
Next, define a normalized scaling variable $u \triangleq \frac{z^{(j)} - r^{(j)}}{y^{(j)} - r^{(j)}}$, which implies $\dd z^{(j)} = (y^{(j)} - r^{(j)}) \dd u$. Under this transformation, $z^{(j)} - r^{(j)} = u(y^{(j)} - r^{(j)})$ and $y^{(j)} - z^{(j)} = (1-u)(y^{(j)} - r^{(j)})$, mapping the support to $u \in (0, 1)$. Substituting these expressions into the spatial product terms, we have the following:
\begin{align}
\frac{(y^{(j)} - r^{(j)})^{\frac{3}{2}}}{\big(u(y^{(j)} - r^{(j)})\big)^{\frac{3}{2}}\big((1-u)(y^{(j)} - r^{(j)})\big)^{\frac{3}{2}}} = (y^{(j)} - r^{(j)})^{-\frac{3}{2}} u^{-\frac{3}{2}} (1-u)^{-\frac{3}{2}}. \notag
\end{align}
Accordingly, the exponential component can be written as follows:
\begin{align}
\varsigma^{(j)}_{s,t,T}(y^{(j)}, r^{(j)}) = \exp\left( - \frac{1}{y^{(j)} - r^{(j)}} \left[ \frac{(t-s)^2}{u} + \frac{(T-t)^2}{1-u} - (T-s)^2 \right] \right). \notag
\end{align}
Compiling these expressions together to isolate the transformed density profile function $p_{U^{(j)}}(u)$ by incorporating the Jacobian coordinate determinant $(y^{(j)} - r^{(j)})$ yields:
\begin{align}
p_{U^{(j)}}(u) &= \frac{1}{\sqrt{\pi}}\frac{(t-s)(T-t)}{T-s} (y^{(j)} - r^{(j)})^{-\frac{3}{2}} u^{-\frac{3}{2}} (1-u)^{-\frac{3}{2}} \varsigma^{(j)}_{s,t,T}(y^{(j)}, r^{(j)}) (y^{(j)} - r^{(j)}) \1(0 < u < 1) \nonumber \\
&= \frac{t-s}{\sqrt{\pi(y^{(j)} - r^{(j)})}} \frac{T-t}{T-s} u^{-\frac{3}{2}} (1-u)^{-\frac{3}{2}} \varsigma^{(j)}_{s,t,T}(y^{(j)}, r^{(j)})\1(0 < u < 1) \label{eqCompiledStableDensity}
\end{align}
on the open interval $0 < u < 1$, verifying equation (\ref{eqStableBetaDensity}). Under algebraic rearrangement, we have $u = (z^{(j)} - r^{(j)}) / (y^{(j)} - r^{(j)})$ which implies $z^{(j)} = r^{(j)} + u(y^{(j)} - r^{(j)})$. Mapping back to $\xi_s^{\boldsymbol{x},(j)} = r^{(j)}$ and terminal constraints provides $\xi_t^{\boldsymbol{x},\boldsymbol{y},(j)} \law \xi_s^{\boldsymbol{x},\boldsymbol{y},(j)} + U_{s,t}^{(j)}(y^{(j)} - \xi_s^{\boldsymbol{x},\boldsymbol{y},(j)})$ for $j \in \{1,\ldots,n\}$. The statement follows under Hadamard product.
\end{proof}

As before, the unobservable $\boldsymbol{y}$ is replaced by the $\mathcal{L}^2$ best-estimate given in (\ref{continuousvelocitylearned}) during inference. More specifically, $\boldsymbol{y}$ in (\ref{eqExactStableUpdate}) is switched with $f^*(\boldsymbol{\xi}_{t_r}^{\boldsymbol{x}}, t_r; \boldsymbol{\Theta}^*)$ to obtain an approximate simulation step
\begin{align}
\boldsymbol{\xi}_{t_{r+1}}^{\boldsymbol{x}} \leftarrow \boldsymbol{\xi}_{t_r}^{\boldsymbol{x}} + \boldsymbol{U}_{t_r,t_{r+1}} \odot \Big(f^*(\boldsymbol{\xi}_{t_r}^{\boldsymbol{x}}, t_r; \boldsymbol{\Theta}^*) - \boldsymbol{\xi}_{t_r}^{\boldsymbol{x}}\Big),
\end{align}
where, the coordinates of the state-dependent $\boldsymbol{U}_{t_r,t_{r+1}} = [U_{t_r,t_{r+1}}^{(1)}, \ldots, U_{t_r,t_{r+1}}^{(n)}]^{\tp}$ are drawn from 
\begin{align}
U_{t_r,t_{r+1}}^{(j)} \sim p_{U^{(j)}}\Big(u \, \Big| \, \xi_{t_r}^{\boldsymbol{x},(j)}, \, f^{*(j)}(\boldsymbol{\xi}_{t_r}^{\boldsymbol{x}}, t_r; \boldsymbol{\Theta}^*)\Big) \notag
\end{align}
over a discrete time grid $0 = t_0 < t_1 < \dots < t_m = T - \epsilon$. In order to ensure coordinate monotonicity against neural approximation artifacts, we include $\hat{y}^{(j)} = \max(\xi_s^{\boldsymbol{x},(j)}, \, f^{*(j)}( \boldsymbol{\xi}^{\boldsymbol{x}}_s, s ; \boldsymbol{\Theta}^* ))$ for $j \in \{1,\ldots,n\}$ in the algorithm below. Furthermore, to prevent numerical singularities and potential division-by-zero exceptions (i.e., when $\hat{y}^{(j)} = \xi_s^{\boldsymbol{x},(j)}$), we introduce a conditional branch that dynamically freezes the progress multiplier for that stationary component.

\bigskip

\begin{algorithm}[H]
\caption{Stable-$1/2$ Random-Bridge Transport Simulation Pipeline}
\label{alg:stablesimulation}
\Input{Source distribution $\boldsymbol{\Phi}$, optimised model parameters $\boldsymbol{\Theta}^*$, grid size $m$, final time $T$, and boundary offset $\epsilon > 0$}
\Output{Generated path trajectory $\{\boldsymbol{\xi}^{\boldsymbol{x}}_t\}$ and terminal samples approximating $\boldsymbol{\Psi}$}
\BlankLine
Initialise a temporal discretization grid: $0 = t_0 < t_1 < \ldots < t_m = T -\epsilon$\;
Sample initial state $\boldsymbol{x} \leftarrow \boldsymbol{\Phi}$ and fix starting anchor: $\boldsymbol{\xi}^{\boldsymbol{x}}_0 = \boldsymbol{x}$\;
\BlankLine
\For{$r = 0$ \KwTo $m-1$}{
    Set current time step coordinates: $s = t_r$ and $t = t_{r+1}$\;
    Evaluate neural network estimation: $f^*\left( \boldsymbol{\xi}^{\boldsymbol{x}}_s, s ; \boldsymbol{\Theta}^* \right)$\;
    \BlankLine
    \For{$j = 1$ \KwTo $n$}{
        Enforce coordinate monotonicity via projection: $\hat{y}^{(j)} = \max\left(\xi_s^{\boldsymbol{x},(j)}, \, f^{*(j)}\left( \boldsymbol{\xi}^{\boldsymbol{x}}_s, s ; \boldsymbol{\Theta}^* \right)\right)$\;
        \BlankLine
        \If{$\hat{y}^{(j)} = \xi_s^{\boldsymbol{x},(j)}$}{
            Fix progress multiplier for stationary coordinate: $U_{s,t}^{(j)} = 0$\;
        }
        \Else{
            Sample coordinate progress step via tilted density inversion or rejection sampling:\;
            $U_{s,t}^{(j)} \sim \frac{t-s}{\sqrt{\pi(\hat{y}^{(j)} - \xi_s^{\boldsymbol{x},(j)})}} \frac{T-t}{T-s} u^{-\frac{3}{2}} (1-u)^{-\frac{3}{2}} \varsigma^{(j)}_{s,t,T}(\hat{y}^{(j)}, \xi_s^{\boldsymbol{x},(j)}) \1(0 < u < 1)$\;
        }
    }
    Assemble the state-dependent transportation vector: $\boldsymbol{U}_{s,t} = [U_{s,t}^{(1)}, \ldots, U_{s,t}^{(n)}]^{\tp}$\;
    Assemble the projected target destination proxy vector: $\hat{\boldsymbol{y}} = [\hat{y}^{(1)}, \ldots, \hat{y}^{(n)}]^{\tp}$\;
    \BlankLine
    Generate next path step via direct boundary-preserving tilted mapping:\;
    $\boldsymbol{\xi}^{\boldsymbol{x}}_{t} = \boldsymbol{\xi}^{\boldsymbol{x}}_s + \boldsymbol{U}_{s,t} \odot \Big(\hat{\boldsymbol{y}} - \boldsymbol{\xi}^{\boldsymbol{x}}_s\Big)$\;
}
\Return{Generated trajectories starting at $\boldsymbol{\Phi}$ through terminal profile $\boldsymbol{\xi}^{\boldsymbol{x}}_{t_m}$}\;
\end{algorithm}

\subsubsection{Simulating Poisson Random-Bridge Transport}
\label{poissonsim}
%In the same spirit of Section \ref{gammasim}, we shall also provide a simulation setup for the Poisson case using coordinate sampling from a Binomial distribution.

\begin{prop}
\label{propPoissonBridgeBinomial}
Let $\{\boldsymbol{Z}_t\}_{t\in\T}$ be a vector of mutually independent Poisson processes with intensity $\lambda > 0$. Then, the following representation holds:
\begin{align}
\label{eqExactBinomialUpdate}
\boldsymbol{\xi}_{t}^{\boldsymbol{x},\boldsymbol{y}} \law \boldsymbol{\xi}_{s}^{\boldsymbol{x},\boldsymbol{y}} + \boldsymbol{K}_{s,t},
\end{align}
for any $\boldsymbol{y} \in \mathbb{Y}$ and $0 \leq s < t < T$, where $\boldsymbol{K}_{s,t} = [K_{s,t}^{(1)}, \ldots, K_{s,t}^{(n)}]^{\tp}$ is a random vector with mutually independent coordinates $K_{s,t}^{(j)} \sim \text{Binomial}\left(y^{(j)} - \xi_s^{\boldsymbol{x},\boldsymbol{y},(j)}, \, \frac{t-s}{T-s}\right)$ for each $j \in \{1,\ldots,n\}$.
\end{prop}

\begin{proof}
Using the mutual independence of the coordinates of $\{\boldsymbol{Z}_t\}_{t\in\T}$, the joint conditional bridge transition decomposes as in (\ref{productlevydensitybridge}). Thus, given that $X^{(j)} \leq Y^{(j)}$ for $j \in \{1,\ldots,n\}$, we have
\begin{align}
\mathbb{P}(\xi_t^{\boldsymbol{x},(j)} = z^{(j)} \mid \xi_s^{\boldsymbol{x},(j)} = r^{(j)}, \xi_T^{\boldsymbol{x},(j)} = y^{(j)}) = \nu_{s,t,T}(r^{(j)}, z^{(j)}, y^{(j)}) = \frac{f_{T-t}(y^{(j)} - z^{(j)}) \, f_{t-s}(z^{(j)} - r^{(j)})}{f_{T-s}(y^{(j)} - r^{(j)})}, \notag
\end{align}
for integers $r^{(j)} \leq z^{(j)} \leq y^{(j)}$. Substituting the explicit Poisson probability mass functions yields 
\begin{align}
\nu_{s,t,T}(r^{(j)}, z^{(j)}, y^{(j)}) = \frac{ \left[ \frac{(\lambda(T-t))^{y^{(j)} - z^{(j)}} e^{-\lambda(T-t)}}{(y^{(j)} - z^{(j)})!} \right] \left[ \frac{\lambda(t-s))^{z^{(j)} - r^{(j)}} e^{-\lambda(t-s)}}{(z^{(j)} - r^{(j)})!} \right] }{ \frac{(\lambda(T-s))^{y^{(j)} - r^{(j)}} e^{-\lambda(T-s)}}{(y^{(j)} - r^{(j)})!} }  \chi_{\mathcal{D}}, \nonumber
\end{align}
with $\chi_{\mathcal{D}} = \1(r^{(j)} \leq z^{(j)} \leq y^{(j)})$. Note that we have the following reductions:
\begin{align}
e^{-\lambda(T-t)} \cdot e^{-\lambda(t-s)} = e^{-\lambda(T-s)} \quad \text{and} \quad \frac{\lambda^{y^{(j)} - z^{(j)}} \cdot \lambda^{z^{(j)} - r^{(j)}}}{\lambda^{y^{(j)} - r^{(j)}}} = 1. \notag
\end{align}
Therefore, the remaining factorial terms and temporal parameters yield the following:
\begin{align}
\label{eqPoissonSimplified}
\nu_{s,t,T}(r^{(j)}, z^{(j)}, y^{(j)}) = \frac{(y^{(j)} - r^{(j)})!}{(z^{(j)} - r^{(j)})!(y^{(j)} - z^{(j)})!} \cdot \frac{(t-s)^{z^{(j)} - r^{(j)}} (T-t)^{y^{(j)} - z^{(j)}}}{(T-s)^{y^{(j)} - r^{(j)}}} \chi_{\mathcal{D}}.
\end{align}
Note that $(T-s)^{y^{(j)} - r^{(j)}} = (T-s)^{z^{(j)} - r^{(j)}} \cdot (T-s)^{y^{(j)} - z^{(j)}}$, and hence, we achieve
\begin{align}
\nu_{s,t,T}(r^{(j)}, z^{(j)}, y^{(j)}) &= \binom{y^{(j)} - r^{(j)}}{z^{(j)} - r^{(j)}} \left( \frac{t-s}{T-s} \right)^{z^{(j)} - r^{(j)}} \left( \frac{T-t}{T-s} \right)^{y^{(j)} - z^{(j)}} \chi_{\mathcal{D}} \notag \\
&= \binom{y^{(j)} - r^{(j)}}{z^{(j)} - r^{(j)}} p^{z^{(j)} - r^{(j)}} \left( 1 - p \right)^{y^{(j)} - z^{(j)}} \chi_{\mathcal{D}}. \notag
\end{align}
It follows that the random integer increment $K_{s,t}^{(j)} \triangleq z^{(j)} - r^{(j)}$ has a standard Binomial distribution with success probability $p = (t-s)/(T-s)$. Thus, we can write
\begin{align}
\Big( \xi_t^{\boldsymbol{x},(j)} - \xi_s^{\boldsymbol{x},(j)} \Big) \,\,\bigg|\,\, \left\{\xi_s^{\boldsymbol{x},(j)} = r^{(j)}, \, \xi_T^{\boldsymbol{x},(j)} = y^{(j)}\right\} \,\sim\, \text{Binomial}\left(y^{(j)} - r^{(j)}, \, \frac{t-s}{T-s}\right). \notag
\end{align}
Setting $K_{s,t}^{(j)} \law \xi_t^{\boldsymbol{x},\boldsymbol{y},(j)} - \xi_s^{\boldsymbol{x},\boldsymbol{y},(j)}$ and mapping it coordinate-wise provides $\xi_t^{\boldsymbol{x},\boldsymbol{y},(j)} \law \xi_s^{\boldsymbol{x},\boldsymbol{y},(j)} + K_{s,t}^{(j)}$ for $j \in \{1,\ldots,n\}$, and the statement follows.
\end{proof}

\begin{algorithm}[H]
\caption{Poisson Random-Bridge Transport Simulation Pipeline}
\label{alg:poissonsimulation}
\Input{Source distribution $\boldsymbol{\Phi}$, optimised model parameters $\boldsymbol{\Theta}^*$, grid size $m$, final time $T$, and boundary offset $\epsilon > 0$}
\Output{Generated path trajectory $\{\boldsymbol{\xi}^{\boldsymbol{x}}_t\}$ and terminal samples approximating $\boldsymbol{\Psi}$}
\BlankLine
Initialise a temporal discretization grid: $0 = t_0 < t_1 < \ldots < t_m = T -\epsilon$\;
Sample initial state $\boldsymbol{x} \leftarrow \boldsymbol{\Phi}$ and fix starting anchor: $\boldsymbol{\xi}^{\boldsymbol{x}}_0 = \boldsymbol{x}$\;
\BlankLine
\For{$r = 0$ \KwTo $m-1$}{
    Set current time step coordinates: $s = t_r$ and $t = t_{r+1}$\;
    Evaluate neural network estimation: $f^*\left( \boldsymbol{\xi}^{\boldsymbol{x}}_s, s ; \boldsymbol{\Theta}^* \right)$\;
    \BlankLine
    \For{$j = 1$ \KwTo $n$}{
        Compute current dynamic trial upper bound: $N_j = \max\left(0, \big\lfloor f^{*(j)}(\boldsymbol{\xi}^{\boldsymbol{x}}_s, s ; \boldsymbol{\Theta}^*) - \xi_{s}^{\boldsymbol{x},(j)} \big\rfloor\right)$\;
        Sample independent component-wise count: $K_{s,t}^{(j)} \sim \text{Binomial}\left(N_j, \, \frac{t-s}{T-s}\right)$\;
    }
    Assemble the stochastic update vector: $\boldsymbol{K}_{s,t} = [K_{s,t}^{(1)}, \ldots, K_{s,t}^{(n)}]^{\tp}$\;
    Generate next path step via direct Binomial count matching:\;
    $\boldsymbol{\xi}^{\boldsymbol{x}}_{t} = \boldsymbol{\xi}^{\boldsymbol{x}}_s + \boldsymbol{K}_{s,t}$\;
}
\Return{Generated trajectories starting at $\boldsymbol{\xi}^{\boldsymbol{x}}_0$ through terminal profile $\boldsymbol{\xi}^{\boldsymbol{x}}_{t_m}$}\;
\end{algorithm}

\bigskip

As done in Section \ref{gammasim}, for our generative purposes, we substitute the unknown $\boldsymbol{y}$ in equation (\ref{eqExactBinomialUpdate}) with its $\mathcal{L}^2$ best-estimate $f^*(\boldsymbol{\xi}_{t_r}^{\boldsymbol{x}}, t_r; \boldsymbol{\Theta}^*)$ at time $t_r$, and obtain
\begin{align}
\boldsymbol{\xi}_{t_{r+1}}^{\boldsymbol{x}} \leftarrow \boldsymbol{\xi}_{t_r}^{\boldsymbol{x}} + \boldsymbol{K}_{t_r,t_{r+1}}, \quad \text{where} \quad K_{t_r,t_{r+1}}^{(j)} \sim \text{Binomial}\left(\Big\lfloor f^{*(j)}(\boldsymbol{\xi}_{t_r}^{\boldsymbol{x}}, t_r; \boldsymbol{\Theta}^*) - \xi_{t_r}^{\boldsymbol{x},(j)}\Big\rfloor, \, \frac{t_{r+1}-t_r}{T-t_r}\right). \notag
\end{align}
over a discrete time grid $0 = t_0 < t_1 < \dots < t_m = T - \epsilon$.

\subsubsection{Simulating Brownian Random-Bridge Transport: Revisited}
Following the paradigms presented in Sections \ref{gammasim}, \ref{stablesim} and \ref{poissonsim}, we can similarly formulate an approximate simulation algorithm for the Brownian case. More specifically, using the analytical transition densities via the Doob $h$-transform instead of an infinitesimal SDE discretization approximation, we can achieve an alternative approximate simulation protocol.

\begin{prop}
\label{propBrownianBridgeGaussian}
Keep the setup in Corollary \ref{exbrowniananticipative} so that $\boldsymbol{\sigma}\in\mathbb{R}^{n\times n}$ is a diagonal matrix with elements $\sigma^{(j)} > 0$ and the driving process coordinates $Z_t^{(j)} = \sigma^{(j)} W_t^{(j)}$ with $W^{(j)}_0 = x^{(j)}/\sigma^{(j)}$ for $j \in \{1,\ldots,n\}$. Then, the following representation holds:
\begin{align}
\label{eqExactGaussianUpdate}
\boldsymbol{\xi}_{t}^{\boldsymbol{x},\boldsymbol{y}} \law \boldsymbol{\xi}_{s}^{\boldsymbol{x},\boldsymbol{y}} + \frac{t-s}{T-s}\Big(\boldsymbol{y} - \boldsymbol{\xi}_{s}^{\boldsymbol{x},\boldsymbol{y}}\Big) + \boldsymbol{G}_{s,t},
\end{align}
for any $\boldsymbol{y} \in \mathbb{Y}$ and $0 \leq s < t < T$, where $\boldsymbol{G}_{s,t} = [G_{s,t}^{(1)}, \ldots, G_{s,t}^{(n)}]^{\tp}$ is a random vector with mutually independent coordinates $G_{s,t}^{(j)} \sim \mathcal{N}\left(0, \, (\sigma^{(j)})^2 \frac{(t-s)(T-t)}{T-s}\right)$ for $j \in \{1,\ldots,n\}$, and is mutually independent of $\boldsymbol{\xi}_{s}^{\boldsymbol{x},\boldsymbol{y}}$.
\end{prop}

\begin{proof}
Using the mutual independence of the coordinates of $\{\boldsymbol{Z}_t\}_{t\in\T}$, the joint conditional bridge transition decomposes as a product as in (\ref{productlevydensitybridge}). Thus, we have
\begin{align}
\mathbb{P}(\xi_t^{\boldsymbol{x},(j)} \in \dd z^{(j)} \mid \xi_s^{\boldsymbol{x},(j)} = r^{(j)}, \xi_T^{\boldsymbol{x},(j)} = y^{(j)}) = \nu_{s,t,T}(r^{(j)}, z^{(j)}, y^{(j)}) = \frac{f_{T-t}(y^{(j)} - z^{(j)}) \, f_{t-s}(z^{(j)} - r^{(j)})}{f_{T-s}(y^{(j)} - r^{(j)})}, \notag
\end{align}
for $j \in \{1,\ldots,n\}$. Since $Z_t^{(j)} - Z_s^{(j)} \sim \mathcal{N}(0, (\sigma^{(j)})^2(t-s))$, we get
\begin{align}
\nu_{s,t,T}(r^{(j)}, z^{(j)}, y^{(j)}) = \frac{ \left[ \frac{1}{\sqrt{2\pi(\sigma^{(j)})^2(T-t)}} \exp\left(-\frac{(y^{(j)} - z^{(j)})^2}{2(\sigma^{(j)})^2(T-t)}\right) \right] \left[ \frac{1}{\sqrt{2\pi(\sigma^{(j)})^2(t-s)}} \exp\left(-\frac{(z^{(j)} - r^{(j)})^2}{2(\sigma^{(j)})^2(t-s)}\right) \right] }{ \frac{1}{\sqrt{2\pi(\sigma^{(j)})^2(T-s)}} \exp\left(-\frac{(y^{(j)} - r^{(j)})^2}{2(\sigma^{(j)})^2(T-s)}\right) }. \nonumber
\end{align}
Note that we have the following reduction:
\begin{align}
\frac{\sqrt{2\pi(\sigma^{(j)})^2(T-s)}}{\sqrt{2\pi(\sigma^{(j)})^2(T-t)}\sqrt{2\pi(\sigma^{(j)})^2(t-s)}} = \frac{1}{\sqrt{2\pi (\sigma^{(j)})^2 \frac{(t-s)(T-t)}{T-s}}} \triangleq \Xi_{s,t,T}^{(j)}. \notag
\end{align}
Factoring out the parameter multiplier $-1/(2(\sigma^{(j)})^2)$ from the exponents, the sum within the exponential expression reduces to
\begin{align}
\lambda^{(j)}_{s,t,T} = \frac{(y^{(j)} - z^{(j)})^2}{T-t} + \frac{(z^{(j)} - r^{(j)})^2}{t-s} - \frac{(y^{(j)} - r^{(j)})^2}{T-s}. \notag
\end{align}
Next, define $\mu_j \triangleq r^{(j)} + \frac{t-s}{T-s}(y^{(j)} - r^{(j)})$ and the bridge variance $\upsilon^2_{s,t,T} \triangleq \frac{(t-s)(T-t)}{T-s}$. It follows from an algebraic expansion and grouping the quadratic terms in $z^{(j)}$ that the exponent matches 
\begin{align}
\lambda^{(j)}_{s,t,T} = \frac{(z^{(j)} - \mu_j)^2}{\upsilon^2_{s,t,T}}. \notag
\end{align}
Substituting these reductions into the conditional transition density provides
\begin{align}
\label{eqBrownianSimplified}
\nu_{s,t,T}(r^{(j)}, z^{(j)}, y^{(j)}) = \frac{1}{\sqrt{2\pi(\sigma^{(j)})^2\upsilon^2_{s,t,T}}} \exp\left(-\frac{\left(z^{(j)} - \left[ r^{(j)} + \frac{t-s}{T-s}(y^{(j)} - r^{(j)}) \right]\right)^2}{2(\sigma^{(j)})^2\upsilon^2_{s,t,T}}\right),
\end{align}
which is the standard Gaussian distribution $\mathcal{N}(\mu_j, (\sigma^{(j)})^2\upsilon^2_{s,t,T})$. Since $\upsilon^2_{s,t,T}$ depends strictly on the time parameters, the centered residual components are mutually independent of $r^{(j)}$ and $y^{(j)}$. Hence, isolating the spatial increment relative to the mean path, the update simplifies component-wise expressions to the following representation:
\begin{align}
\xi_t^{\boldsymbol{x},\boldsymbol{y},(j)} \law \xi_s^{\boldsymbol{x},\boldsymbol{y},(j)} + \frac{t-s}{T-s}(y^{(j)} - \xi_s^{\boldsymbol{x},\boldsymbol{y},(j)}) + G_{s,t}^{(j)},  \notag
\end{align}
where $G_{s,t}^{(j)} \sim \mathcal{N}(0, (\sigma^{(j)})^2\upsilon^2_{s,t,T})$ for $j \in \{1,\ldots,n\}$ and $0 \leq s < t < T$. The statement follows.
\end{proof}

As in Sections \ref{gammasim} and \ref{poissonsim}, by substituting $\boldsymbol{y}$ in (\ref{eqExactGaussianUpdate}) with its $\mathcal{L}^2$ best-estimate $f^*(\boldsymbol{\xi}_{t_r}^{\boldsymbol{x}}, t_r; \boldsymbol{\Theta}^*)$ at time $t_r$, we reach
\begin{align}
\boldsymbol{\xi}_{t_{r+1}}^{\boldsymbol{x}} \leftarrow \boldsymbol{\xi}_{t_r}^{\boldsymbol{x}} + \frac{t_{r+1}-t_r}{T-t_r}\Big(f^*(\boldsymbol{\xi}_{t_r}^{\boldsymbol{x}}, t_r; \boldsymbol{\Theta}^*) - \boldsymbol{\xi}_{t_r}^{\boldsymbol{x}}\Big) + \boldsymbol{G}_{t_r,t_{r+1}}, \notag
\end{align}
over a discrete time grid $0 = t_0 < t_1 < \dots < t_m = T - \epsilon$.

\bigskip

\begin{algorithm}[H]
\caption{Brownian Random-Bridge Transport Alternative Simulation Pipeline}
\label{alg:browniansimulationrevisited}
\Input{Source distribution $\boldsymbol{\Phi}$, optimised model parameters $\boldsymbol{\Theta}^*$, grid size $m$, diagonal volatility matrix $\boldsymbol{\sigma}$ with entries $\sigma^{(j)}$, final time $T$, and boundary offset $\epsilon > 0$}
\Output{Generated path trajectory $\{\boldsymbol{\xi}^{\boldsymbol{x}}_t\}$ and terminal samples approximating $\boldsymbol{\Psi}$}
\BlankLine
Initialise a temporal discretization grid: $0 = t_0 < t_1 < \ldots < t_m = T -\epsilon$\;
Sample initial state $\boldsymbol{x} \leftarrow \boldsymbol{\Phi}$ and fix starting anchor: $\boldsymbol{\xi}^{\boldsymbol{x}}_0 = \boldsymbol{x}$\;
\BlankLine
\For{$r = 0$ \KwTo $m-1$}{
    Set current time step coordinates: $s = t_r$ and $t = t_{r+1}$\;
    Evaluate neural network estimation: $f^*\left( \boldsymbol{\xi}^{\boldsymbol{x}}_s, s ; \boldsymbol{\Theta}^* \right)$\;
    Compute the exact localized conditional variance parameter: $\upsilon^2_{s,t,T} = \frac{(t-s)(T-t)}{T-s}$\;
    \BlankLine
    \For{$j = 1$ \KwTo $n$}{
        Sample independent coordinate noise: $G_{s,t}^{(j)} \sim \mathcal{N}\left(0, \, (\sigma^{(j)})^2 \upsilon^2_{s,t,T}\right)$\;
    }
    Assemble the stochastic perturbation vector: $\boldsymbol{G}_{s,t} = [G_{s,t}^{(1)}, \ldots, G_{s,t}^{(n)}]^{\tp}$\;
    \BlankLine
    Generate next path step via direct density-driven macro-interpolation mapping:\;
    $\boldsymbol{\xi}^{\boldsymbol{x}}_{t} = \boldsymbol{\xi}^{\boldsymbol{x}}_s + \frac{t-s}{T-s} \Big(f^*\left( \boldsymbol{\xi}^{\boldsymbol{x}}_s, s ; \boldsymbol{\Theta}^* \right) - \boldsymbol{\xi}^{\boldsymbol{x}}_s\Big) + \boldsymbol{G}_{s,t}$\;
}
\Return{Generated trajectories starting at $\boldsymbol{\xi}^{\boldsymbol{x}}_0$ through terminal profile $\boldsymbol{\xi}^{\boldsymbol{x}}_{t_m}$}\;
\end{algorithm}

\subsection{Outlook on \levy Transport Simulation}
The explicit cases of Brownian, Gamma, Stable-$1/2$ and Poisson pipelines highlight a foundational unifying feature of our approach: the capacity to translate abstract measure-theoretic bridges into approximating simulation protocols. While Proposition \ref{stochfiltlevyrandom} establishes the structural existence of the $(\mathbb{P},\{\mathcal{F}_{t}^{\boldsymbol{\xi}^{\boldsymbol{x}}}\}_{t\in\mathbb{T}})$-martingale $\{\boldsymbol{M}_t\}_{t\in\T}$ for any integrable \levy driving process, achieving its closed-form representation remains open, and merits a separate dedicated theoretical study. As demonstrated by the cases of Gamma, Stable-$1/2$ and Poisson, the explicit form of $\{\boldsymbol{M}_t\}_{t\in\T}$ is not a strict prerequisite for simulation efforts. Nevertheless, the mathematical value of $\{\boldsymbol{M}_t\}_{t\in\T}$ extends beyond numerical sampling; it provides a foundational theoretical link to dynamic SDE representations and Flow Matching frameworks. In addition, if an explicit closed-form or a neural approximation for $\{\boldsymbol{M}_t\}_{t\in\mathbb{T}}$ is achieved, one can further consider SDE-based simulations. 

\section{Conclusion}
We present a mathematical framework that serves as an alternative approach to address generative modelling via random-bridges, which act as stochastic transports between arbitrary pairs of probability distributions. In principle, our mathematical framework expands the perspective of diffusion-based models in generative AI by venturing beyond Brownian motion-driven diffusion constructs through a broader avenue of stochastic behavior, while establishing a connection to information processing and stochastic filtering. Our core philosophy is to start with a general probabilistic construct, and thereafter branch out into various representations that can exhibit continuous, discontinuous, or hybrid path trajectories. We provide explicit training and simulation algorithms over an important subclass of random-bridges, namely Gaussian random-bridges, alongside several canonical \levy random-bridges including Gamma, Poisson, and Stable-$1/2$ subordinators. Utilizing the Gaussian random-bridge framework, we validate our algorithms through empirical experiments on the MNIST and CIFAR-10 datasets. Our empirical results demonstrate that the proposed framework delivers highly competitive generation quality, matching widely used denoising diffusions and Flow Matching baselines.

The promising performance of this paradigm motivates several directions for future research: 
(i) Validation of our approach on more complex, high-resolution datasets such as ImageNet to better analyze the framework's characteristics and scalability in high-dimensional spaces; 
(ii) Development of adaptive optimization schedules for learning the hyperparameters within our framework to further enhance generation efficiency; 
(iii) Extension to image-to-image translation tasks, leveraging the framework's innate capacity to handle arbitrary reference distributions; 
(iv) Expansion of our comparative analysis to benchmark our framework directly against other recent advancements in generative modelling, such as distillation-based approaches and Schr{\"o}dinger bridges;
(v) Empirical implementation of our ready-to-deploy \levy random-bridge transport algorithms to investigate their practical structural advantages in scenarios where the underlying transport exhibits specific patterns, such as transports for monotonic transformations.
%to circumvent unnecessary spatial exploration in order-preserving generation; 

\appendix
% =====================================================================
%  Appendix: supporting measurements for the Gaussian experiments
%  Place AFTER \section{Conclusion} and BEFORE \begin{thebibliography},
%  preceded by \appendix
% =====================================================================

\section{Supporting Measurements for the Gaussian Experiments}
\label{app:experiments}

\subsection{Velocity Field, Fidelity and Coverage }
\label{app:weighting}

The velocity-regression objective in Section \ref{sec:gausssetup} targets the
velocity field $\boldsymbol{u}(\xib^{\boldsymbol{x}}_t,t)$ of
\eqref{continuousvelocity} in place of the $\mathcal{L}^2$ best-estimate
$\boldsymbol{Y}^{\boldsymbol{x}}_t$. The two objectives differ by the weight
$(T-t)^{-2}$, and both are minimised at the same function. We measure how this
weighting affects performance at finite network capacity.
To that end, we trained two networks identically, one against
$\boldsymbol{Y}^{\boldsymbol{x}}_t$ and one against $\boldsymbol{u}$. Both networks
were then evaluated as estimators of $\boldsymbol{Y}^{\boldsymbol{x}}_t$ along the
same $\boldsymbol{\sigma}=\boldsymbol{0}$ interpolant. Using a common trajectory
avoids favouring either network through its own training noise. The
velocity-trained output is converted by
$\xib^{\boldsymbol{x}}_t + (T-t)\,\hat{\boldsymbol{u}}$. Measurement
uses $2048$ triples $(\boldsymbol{x},\boldsymbol{y},t)$ drawn from held-out data.

\begin{table}[htbp]
\begin{center}
\caption{CIFAR-10. Squared error of the endpoint estimate (left) and the same
quantity expressed in the units of the velocity field \eqref{eq:velocityrepeat}
(right). The three networks agree to three decimal places over the first eighty
per cent of the bridge.}
\label{tab:weighting}
\vskip 0.1in
\begin{tabular}{lccccccc}
\toprule
& \multicolumn{3}{c}{endpoint squared error} & & \multicolumn{3}{c}{velocity squared error}\\
\cmidrule{2-4}\cmidrule{6-8}
$t/T$ & $\gamma{=}0.05$ & $\gamma{=}0.3162$ & velocity & & $\gamma{=}0.05$ & $\gamma{=}0.3162$ & velocity\\
\midrule
$0.050$ & $0.20779$ & $0.20783$ & $0.20801$ & & $0.230$ & $0.230$ & $0.230$\\
$0.400$ & $0.04515$ & $0.04534$ & $0.04407$ & & $0.125$ & $0.126$ & $0.122$\\
$0.600$ & $0.02017$ & $0.02051$ & $0.01932$ & & $0.126$ & $0.128$ & $0.121$\\
$0.800$ & $0.00659$ & $0.00698$ & $0.00616$ & & $0.165$ & $0.175$ & $0.154$\\
$0.900$ & $0.00242$ & $0.00275$ & $0.00216$ & & $0.242$ & $0.275$ & $0.216$\\
$0.950$ & $0.00094$ & $0.00120$ & $0.00076$ & & $0.377$ & $0.480$ & $0.305$\\
$0.980$ & $0.00033$ & $0.00052$ & $0.00018$ & & $0.827$ & $1.288$ & $0.455$\\
$0.995$ & $0.00018$ & $0.00028$ & $0.00002$ & & $7.171$ & $11.148$ & $0.698$\\
\bottomrule
\end{tabular}
\end{center}
\end{table}

\begin{figure}[htbp]
\begin{center}
\includegraphics[width=\textwidth]{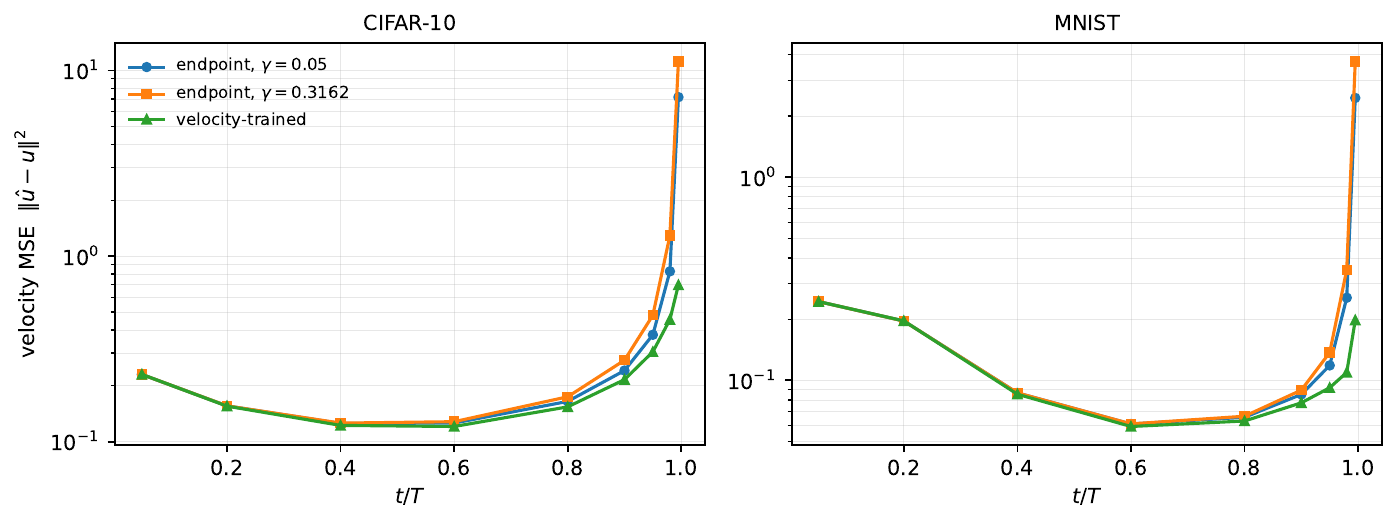}
\caption{Squared error of the velocity field along the bridge, logarithmic
vertical scale, $2048$ samples per point on a common
$\boldsymbol{\sigma}=\boldsymbol{0}$ trajectory. Left: CIFAR-10. Right: MNIST.
The curves are indistinguishable until $t\approx0.8\,T$ and separate thereafter.}
\label{fig:velocityerror}
\end{center}
\end{figure}

Velocity regression reduces CIFAR-10 FID from $38.93$ to $23.23$ at fifty
network evaluations with $\boldsymbol{\gamma}=0.3162\,\boldsymbol{I}$. Both
measurements use Euler--Maruyama sampling at $n=2000$, with architecture, budget,
sampler and evaluation code held fixed. MNIST shows a slightly larger
endpoint-error ratio at $t=0.995\,T$, measuring $12.4$ against $10.3$ for
CIFAR-10. Its terminal error in velocity units is lower, measuring $3.70$ against
$11.15$ for CIFAR-10. The coarse structure of near-binary MNIST digits is
resolved well before $t=0.8\,T$. CIFAR-10 texture is resolved in the terminal
region.
The second term of \eqref{eq:velocitytarget} has conditional variance
$\boldsymbol{\sigma}^2 t/(T-t)$, so larger $\boldsymbol{\gamma}$ increases the
gradient variance near the terminal time. 
The training objective settles at
$0.330$ for $\boldsymbol{\gamma}=0.3162$, compared with $0.190$ for
$\boldsymbol{\gamma}=0.05$. FID at fifty evaluations is %nevertheless similar:
$23.23$ and $23.34$, respectively.

%\subsection{Fidelity, coverage and an alternative feature space}
%\label{app:coverage}

The deterministic DDIM sampler has recall of $0.570$ and coverage of $0.396$
at ten network evaluations (Table \ref{tab:coverage}). These values indicate that
the generated manifold is dispersed widely enough to contain real points, while
few real points have a generated neighbour nearby. The FID of $35.69$ does not
identify this specific sampling problem.

\begin{table}[H]
\begin{center}
\caption{CIFAR-10 at $n=10\,000$. Precision is the fraction of generated samples
inside the real manifold and recall the fraction of real samples inside the
generated one; coverage is the outlier-robust form of recall. The Kernel
Inception Distance (KID) orders the rows identically to FID at every evaluation
count.}
\label{tab:coverage}
\vskip 0.1in
\resizebox{\textwidth}{!}{%
\begin{tabular}{llcccccc}
\toprule
Model & Simulation & Evals & FID & precision & recall & coverage & KID\\
\midrule
Random-bridge, velocity   & Euler--Maruyama   & $10$ & $\mathbf{20.87}$ & $0.606$ & $\mathbf{0.530}$ & $0.571$ & $\mathbf{0.01391}$\\
Flow Matching             & Euler             & $10$ & $21.21$ & $0.576$ & $0.514$ & $0.565$ & $0.01873$\\
Improved denoising        & ancestral         & $10$ & $27.69$ & $\mathbf{0.617}$ & $0.490$ & $\mathbf{0.578}$ & $0.01827$\\
Improved denoising        & DDIM              & $10$ & $35.69$ & $0.514$ & $0.570$ & $0.396$ & $0.01945$\\
Simple denoising          & ancestral         & $10$ & $52.97$ & $0.472$ & $0.286$ & $0.340$ & $0.05150$\\
Random-bridge, endpoint   & exact posterior   & $10$ & $44.44$ & $0.494$ & $0.365$ & $0.340$ & $0.04809$\\
\midrule
Random-bridge, velocity   & Euler--Maruyama   & $50$ & $11.86$ & $0.614$ & $0.579$ & $0.638$ & $0.00647$\\
Flow Matching             & Euler             & $50$ & $12.45$ & $0.618$ & $0.577$ & $0.670$ & $0.00949$\\
Improved denoising        & ancestral         & $50$ & $\mathbf{8.02}$ & $\mathbf{0.655}$ & $\mathbf{0.583}$ & $\mathbf{0.753}$ & $\mathbf{0.00283}$\\
Improved denoising        & DDIM              & $50$ & $14.82$ & $0.599$ & $0.608$ & $0.636$ & $0.00954$\\
Simple denoising          & ancestral         & $50$ & $20.84$ & $0.581$ & $0.496$ & $0.570$ & $0.01797$\\
Random-bridge, endpoint   & exact posterior   & $50$ & $31.33$ & $0.563$ & $0.458$ & $0.423$ & $0.02989$\\
\midrule
Random-bridge, velocity   & Euler--Maruyama   & $200$ & $10.84$ & $0.607$ & $0.584$ & $0.680$ & $0.00596$\\
Flow Matching             & Euler             & $200$ & $10.88$ & $0.624$ & $0.594$ & $0.680$ & $0.00662$\\
Improved denoising        & ancestral         & $200$ & $\mathbf{7.35}$ & $\mathbf{0.662}$ & $0.602$ & $\mathbf{0.757}$ & $\mathbf{0.00299}$\\
Improved denoising        & DDIM              & $200$ & $10.85$ & $0.619$ & $\mathbf{0.615}$ & $0.682$ & $0.00642$\\
Simple denoising          & ancestral         & $200$ & $13.73$ & $0.622$ & $0.557$ & $0.656$ & $0.01111$\\
Random-bridge, endpoint   & exact posterior   & $200$ & $29.50$ & $0.563$ & $0.448$ & $0.428$ & $0.02650$\\
\bottomrule
\end{tabular}}
\end{center}
\end{table}

\begin{table}[htbp]
\begin{center}
\caption{MNIST, $50\,000$ generated samples, scored in the feature space of a
convolutional classifier trained on the MNIST training split. Mode TV is the
total-variation distance between the classifier's predicted-label histogram over
generated samples and the histogram of the data. The final row compares the real
test split with the training split and measures their train/test difference.
It does not estimate the finite-sample FID floor.}
\label{tab:cfid}
\vskip 0.1in
\begin{tabular}{llcccc}
\toprule
Model & Simulation & Evals & classifier FID & mode TV & least-served digit\\
\midrule
Random-bridge      & Euler--Maruyama & $10$ & $\mathbf{27.05}$ & $0.0464$ & $1$ at $0.819$\\
Flow Matching      & Euler           & $10$ & $31.72$ & $\mathbf{0.0438}$ & $0$ at $0.857$\\
Improved denoising & ancestral       & $10$ & $39.64$ & $0.0885$ & $1$ at $0.629$\\
\midrule
Random-bridge      & Euler--Maruyama & $100$ & $\mathbf{0.85}$ & $0.0128$ & $7$ at $0.922$\\
Flow Matching      & Euler           & $100$ & $1.19$ & $0.0147$ & $7$ at $0.923$\\
Improved denoising & ancestral       & $100$ & $1.11$ & $\mathbf{0.0110}$ & $2$ at $0.950$\\
\midrule
Real test split    & ---             & ---   & $1.78$ & $0.0080$ & $6$ at $0.960$\\
\bottomrule
\end{tabular}
\end{center}
\end{table}

Improved denoising diffusion generates the digit $1$ at $62.9$ per cent of its
frequency in the data at ten network evaluations (Table \ref{tab:cfid}). Its label
distortion is twice that of the random-bridge and Flow Matching models. This
ordering matches the CIFAR-10 recall results in Table \ref{tab:coverage}.

\subsection{Sensitivity to Diffusion Scale and Simulation Scheme}
\label{app:gamma}

The deterministic map omits the driving increment $\dd\boldsymbol{Z}_t$ from
Algorithm \ref{alg:simulation} and does not preserve the same marginals. A
deterministic scheme that preserves them must follow the probability-flow ODE.
Its drift includes the additional term
$-\tfrac{1}{2}\boldsymbol{\sigma}^2\nabla\log p_t$. The resulting bias increases
with $\boldsymbol{\gamma}$, as shown in Figure \ref{fig:gamma}. At two hundred
network evaluations, the deterministic map has FID scores of $1.47$, $1.91$,
$3.39$, $9.51$ and $17.58$ for
$\boldsymbol{\gamma} = 0.1, 0.2, 0.3162, 0.6, 1.0$, respectively.
Euler--Maruyama FID remains between $1.50$ and $2.90$ across the same range.

\begin{figure}[H]
\begin{center}
\includegraphics[width=\textwidth]{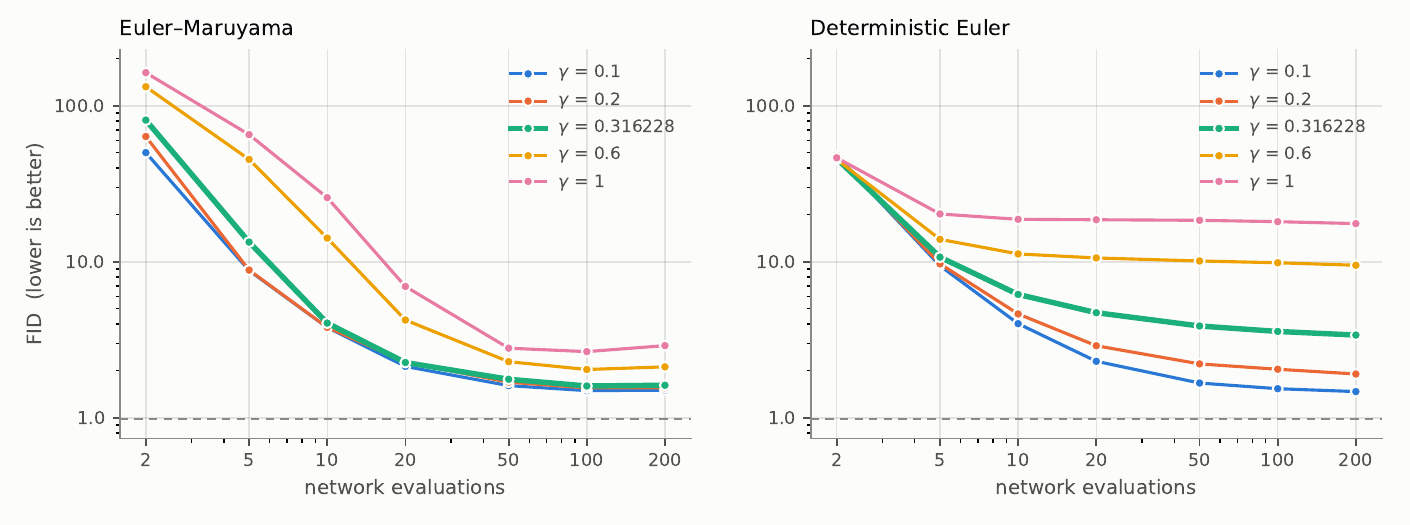}
\caption{MNIST FID against network evaluations for five values of
$\boldsymbol{\gamma}=\boldsymbol{\sigma}\sqrt{T}$. Each value uses a separately
trained model with the same architecture, training budget and evaluation code.
Results use $n=10\,000$, and the dashed line marks the measured floor of $0.98$.
The heavier line corresponds to $\boldsymbol{\gamma}=0.3162$, the value used
throughout the main results. The left panel uses the Euler--Maruyama scheme of
Algorithm \ref{alg:simulation}. The right panel uses the deterministic map with
the driving increment removed. This map does not preserve the marginals, and
its bias grows with $\boldsymbol{\gamma}$.}
\label{fig:gamma}
\end{center}
\end{figure}

We also evaluate the simulation schemes on a one-dimensional problem where
$\boldsymbol{Y}^{\boldsymbol{x}}_t$ is available in closed form. All residual
error is therefore attributable to the simulation scheme. The Euler--Maruyama
and exact-transition schemes both reach the Monte-Carlo floor in
$1$-Wasserstein distance, with values of $0.0095$ and $0.0096$ against a floor of
$0.0091$. The deterministic map stalls at $0.0625$. Its terminal standard
deviation is $1.726$ against $1.758$. The residual scales with
$\boldsymbol{\gamma}$ as predicted by the omitted drift term. It is $1.2$ times
the floor at $\boldsymbol{\gamma}=0.0316$ and $1.5$ times at
$\boldsymbol{\gamma}=0.1$. It rises to $6.9$ times the floor at
$\boldsymbol{\gamma}=0.3162$ and $18.6$ times at $\boldsymbol{\gamma}=0.6$.
Smaller $\boldsymbol{\gamma}$ improves both schemes at every evaluation count
above five. On CIFAR-10 under the velocity objective, $\boldsymbol{\gamma}=0.05$
and $\boldsymbol{\gamma}=0.3162$ attain $23.34$ and $23.23$ at fifty
evaluations, a difference within the resolution of that measurement.

\subsection{Generated Samples}
\label{app:samples}

The figures below show generated CIFAR-10 and MNIST images for the
sampling-budget and model comparisons in Section \ref{sec:gaussresults}.
Figure \ref{fig:samplesobjective} illustrates the regression-target comparison
in Appendix \ref{app:weighting}. Each panel shows the first $64$ samples of the run named. Initial states are
drawn from one seeded generator in fixed batch order, so the first $64$ are the
same draws regardless of the sample count used for the corresponding FID entry.

%\suppressfloats[t]

\begin{figure}[htbp]
\begin{center}
\includegraphics[width=\textwidth]{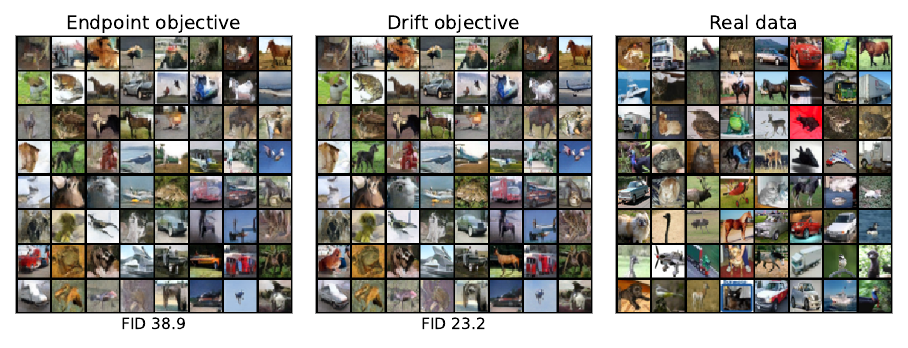}
\caption{CIFAR-10 at fifty network evaluations with
$\boldsymbol{\sigma}=\boldsymbol{I}$, $T=0.1$, under the two regression targets
at matched architecture, optimiser, sampler and budget. Both panels are
Euler--Maruyama and both figures are quoted at $n=2000$, the sample count at
which both were measured. The velocity-trained
panel resolves object boundaries and surface texture left diffuse by the
endpoint-trained panel, which is the terminal region of the bridge that the
uniform weight discounts.}
\label{fig:samplesobjective}
\end{center}
\end{figure}

\begin{figure}[htbp]
\begin{center}
\includegraphics[width=\textwidth]{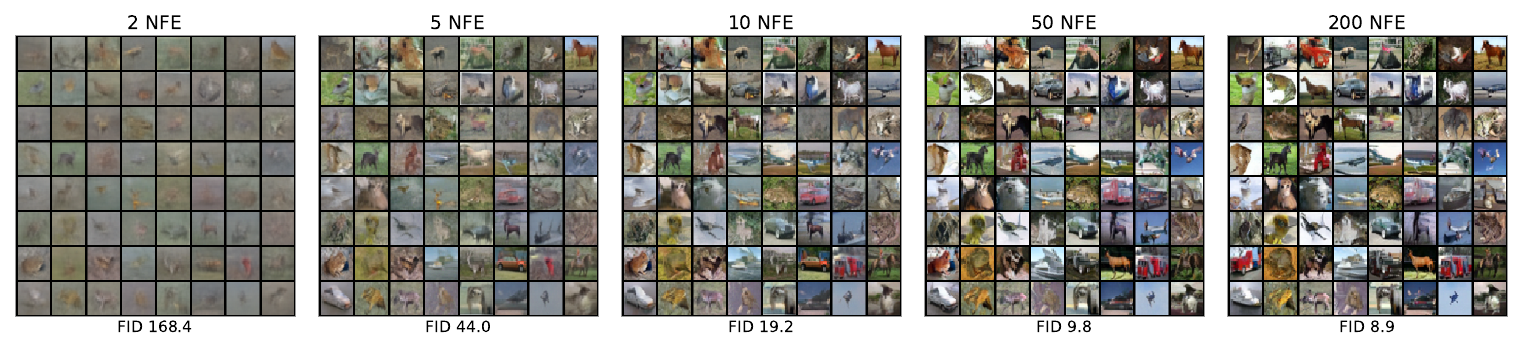}
\caption{CIFAR-10 samples against simulation budget,
$\boldsymbol{\sigma}=\boldsymbol{I}$, $T=0.1$, Euler--Maruyama, FID at
$n=50\,000$. Two evaluations return the conditional mean of the data; five
resolve object layout; ten give recognisable objects.}
\label{fig:samplessteps}
\end{center}
\end{figure}

\begin{figure}[htbp]
\begin{center}
\includegraphics[width=\textwidth]{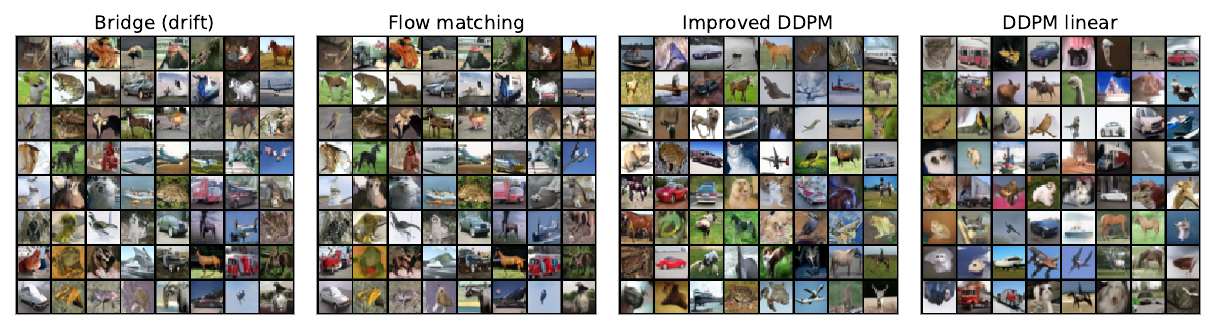}
\caption{CIFAR-10 at fifty network evaluations under one architecture and one
training budget. Optimisation is AdamW throughout; the random-bridge panel uses
zero weight decay and the three baselines AdamW's default $10^{-2}$.}
\label{fig:samplesmodels}
\end{center}
\end{figure}

\begin{figure}[htbp]
\begin{center}
\includegraphics[width=\textwidth]{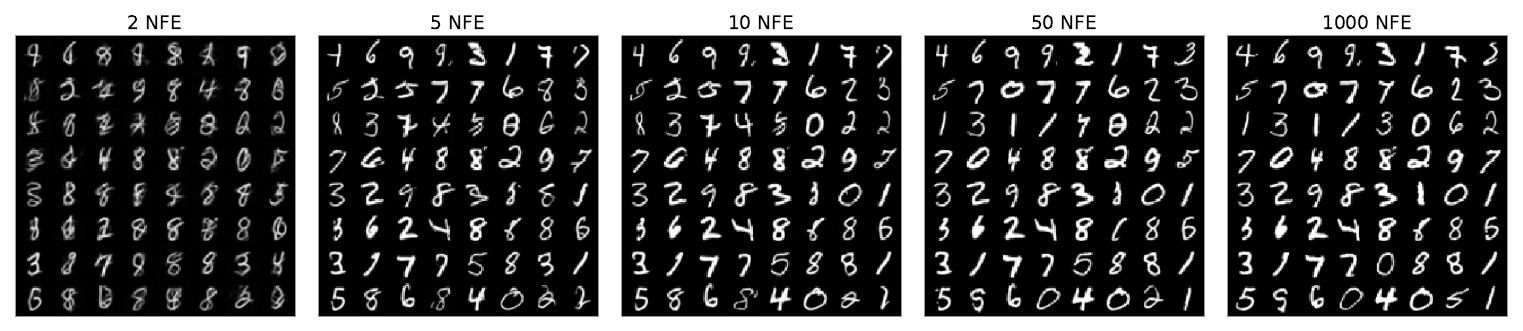}
\caption{MNIST samples against simulation budget,
$\boldsymbol{\sigma}=\boldsymbol{I}$, $T=0.1$, Euler--Maruyama.}
\label{fig:samplesmnist}
\end{center}
\end{figure}

\subsection{Conditional Variance Process in Proposition \ref{propconditionalvarianceprocess}}
\label{app:variance}

Proposition \ref{propconditionalvarianceprocess} states that
$\{\boldsymbol{V}^{\boldsymbol{x}}_t\}$ is a supermartingale under the Loewner
order and Remark \ref{remarkconditionalvarianceprocess} that it vanishes at the
terminal time. The training objective estimates the quantity without additional
computation, since
\begin{align}
\E\|\boldsymbol{y}-f(\xib^{\boldsymbol{x}}_t,t;\boldsymbol{\Theta})\|^2
\to \mathrm{Tr}\,\boldsymbol{V}^{\boldsymbol{x}}_t / n \notag 
\end{align}
once $f$ attains the
$\mathcal{L}^2$ projection, so a trained network measures the process. The theory predicts two values exactly. The trace must vanish as $t\to T$,
when the state determines $\boldsymbol{Y}$. The measured value is
$9\times10^{-5}$. At $t=0$, the state $\xib^{\boldsymbol{x}}_0=\boldsymbol{x}$ is
independent of $\boldsymbol{Y}$ under the product coupling. The trace must
therefore equal the mean per-pixel variance of $\boldsymbol{\Psi}$. The
theoretical value is $0.21201$ and the measured value is $0.21346$, giving a
ratio of $1.007$. The profile decreases across all $21$ sampled times
(Figure \ref{fig:variance}).

\begin{figure}[htbp]
\begin{center}
\includegraphics[width=\textwidth]{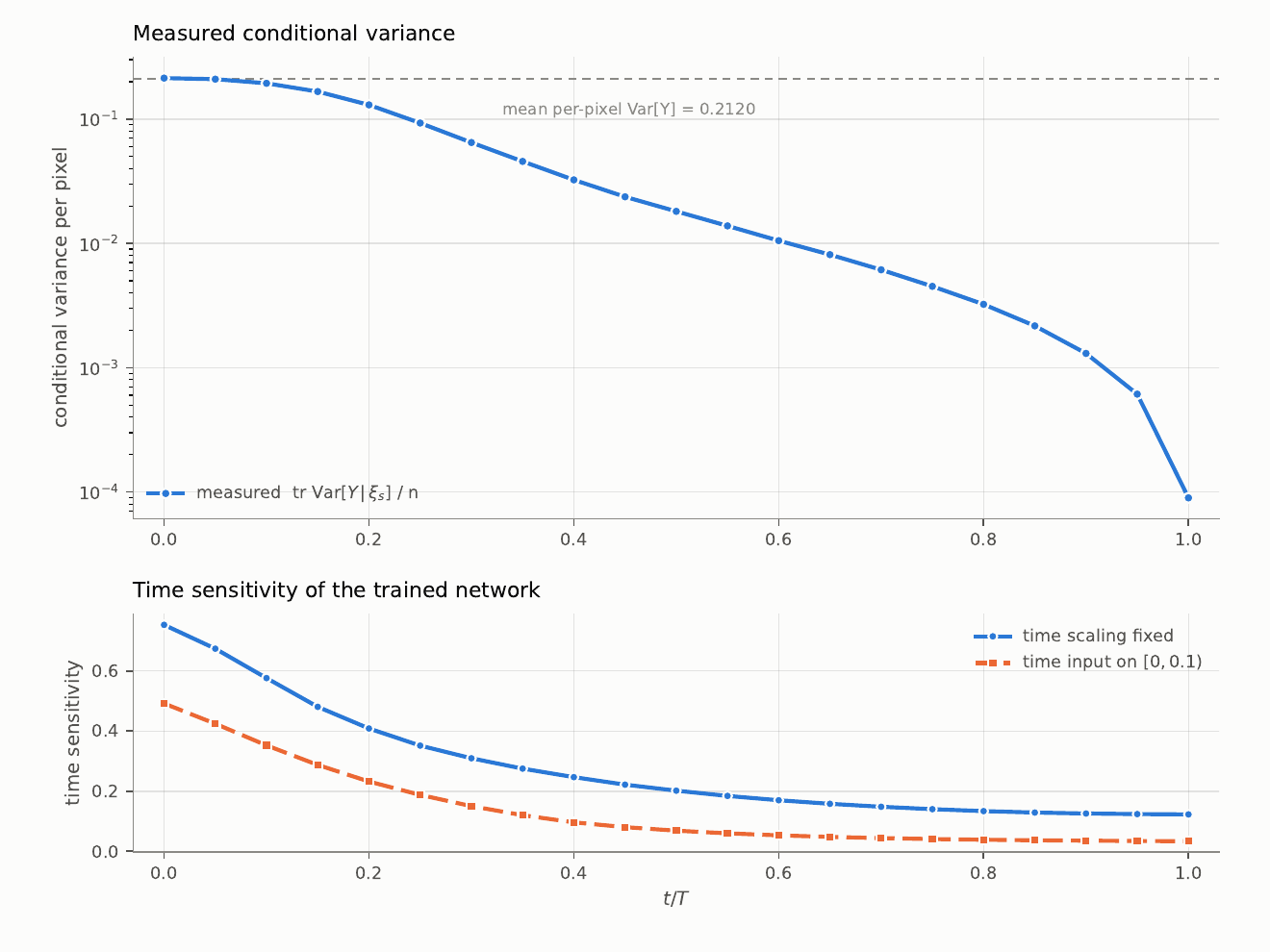}
\caption{Upper panel: $\mathrm{Tr}\,\boldsymbol{V}^{\boldsymbol{x}}_t/n$
estimated by the regression loss of a trained network on held-out data. The
dashed rule marks the mean per-pixel variance of $\boldsymbol{\Psi}$, the value
Proposition \ref{propconditionalvarianceprocess} requires at $t=0$. Lower panel:
the dependence of the network output on its time argument, obtained by holding
the state fixed and sweeping $t$, normalised by the output's own pixel standard
deviation. The two curves are the same architecture and budget trained with the
time argument rescaled to $[0,1)$ and left on $[0,T)$ respectively; they reach
the same training loss at every $t$ and differ in this sensitivity alone.}
\label{fig:variance}
\end{center}
\end{figure}

The conditional expectation map \eqref{condflowmatch} returns
$\E[\boldsymbol{Y}]$ for every initial state when evaluated in a single step.
All generated samples therefore coincide. The measured FID is $385.98$ in the
one-evaluation column of Table \ref{tab:cifar}. In the one-dimensional study,
the terminal standard deviation is $1.6\times10^{-16}$.
Algorithm \ref{alg:simulation} integrates \eqref{nonanticipativerepGMRB} and is
unaffected by this degeneracy at finer grids. The exact-transition scheme of
Proposition \ref{propBrownianBridgeGaussian} replaces the posterior
$\pi_t(\dd\boldsymbol{y})$ with a point mass at its mean. The law of total
variance shows that this removes
\begin{align}
\left(\frac{t'-t}{T-t}\right)^{2}\boldsymbol{V}^{\boldsymbol{x}}_t \notag 
\end{align}
from the
transition variance over $[t,t']$. We estimate this missing term isotropically
using the profile in Figure \ref{fig:variance}. Adding this estimate raises the
exact-transition scheme's coverage from $0.602$ to $0.639$ at fifty network
evaluations. Euler--Maruyama coverage is $0.638$. The exact-transition scheme's
FID falls from $17.70$ to $14.28$. The correction leaves Euler--Maruyama
unchanged to four decimal places, as the identity requires.

% Flush every outstanding float before control returns to the
% main file, so that none can migrate into the bibliography.

\end{document}